\documentclass[english]{article}
\usepackage[T1]{fontenc}
\usepackage[latin9]{inputenc}
\usepackage{wrapfig}
\usepackage{url}
\usepackage{amsmath}
\usepackage{amsthm}
\usepackage{amssymb}
\usepackage{graphicx}
\usepackage[numbers]{natbib}

\makeatletter

\providecommand{\tabularnewline}{\\}

\numberwithin{equation}{section}
\numberwithin{figure}{section}
\theoremstyle{plain}
\newtheorem{thm}{\protect\theoremname}
  \theoremstyle{definition}
  \newtheorem{defn}[thm]{\protect\definitionname}
  \theoremstyle{plain}
  \newtheorem{lem}[thm]{\protect\lemmaname}
  \theoremstyle{plain}
  \newtheorem{cor}[thm]{\protect\corollaryname}
  \theoremstyle{plain}
  \newtheorem{fact}[thm]{\protect\factname}

\usepackage{times}
\usepackage{thm-restate}

\usepackage{iclr2018_conference}
\iclrfinalcopy

\usepackage[T1]{fontenc}    
\usepackage{hyperref}       
\usepackage{url}            
\usepackage{booktabs}       
\usepackage{amsfonts}       
\usepackage{nicefrac}       
\usepackage{microtype}      

\usepackage{amsmath}
\usepackage{amssymb}

\usepackage{algorithmic}
\usepackage[ruled]{algorithm2e}

\author{Daniel Soudry, Elad Hoffer \\
   Department of Electrical Engineering \\ 
   Technion \\ 
   Haifa, 320003, Israel \\
   \texttt{{daniel.soudry,elad.hoffer}@gmail.com}
}

\newcommand*{\QEDA}{\hfill\ensuremath{\blacksquare}}%

\newtheorem{assm}{Assumption}

\makeatother

\usepackage{babel}
  \providecommand{\corollaryname}{Corollary}
  \providecommand{\definitionname}{Definition}
  \providecommand{\factname}{Fact}
  \providecommand{\lemmaname}{Lemma}
\providecommand{\theoremname}{Theorem}

\begin{document}
\global\long\def\E{\mathbb{E}}
\global\long\def\wb{{\mathbf{w}}}
\global\long\def\eps{{\varepsilon}}
\global\long\def\dS{{d_{S}}}

\title{Exponentially vanishing sub-optimal local minima in multilayer neural
networks }
\maketitle
\begin{abstract}
\textbf{Background: }Statistical mechanics results (Dauphin et al.
(2014); Choromanska et al. (2015)) suggest that local minima with
high error are exponentially rare in high dimensions. However, to
prove low error guarantees for Multilayer Neural Networks (MNNs),
previous works so far required either a heavily modified MNN model
or training method, strong assumptions on the labels (\emph{e.g.},
``near'' linear separability), or an unrealistically wide hidden
layer with $\Omega\left(N\right)$ units. 

\textbf{Results: }We examine a MNN with one hidden layer of piecewise
linear units, a single output, and a quadratic loss. We prove that,
with high probability in the limit of $N\rightarrow\infty$ datapoints,
the volume of differentiable regions of the empiric loss containing
sub-optimal differentiable local minima is exponentially vanishing
in comparison with the same volume of global minima, given standard
normal input of dimension $d_{0}=\tilde{\Omega}\left(\sqrt{N}\right)$,
and a more realistic number of $d_{1}=\tilde{\Omega}\left(N/d_{0}\right)$
hidden units. We demonstrate our results numerically: for example,
$0\%$ binary classification training error on CIFAR with only $N/d_{0}\approx16$
hidden neurons.
\end{abstract}

\section{Introduction\label{sec:Introduction}}

\paragraph*{Motivation. }

Multilayer Neural Networks (MNNs), trained with simple variants of
stochastic gradient descent (SGD), have achieved state-of-the-art
performances in many areas of machine learning \citep{LeCun2015}.
However, theoretical explanations seem to lag far behind this empirical
success (though many hardness results exist\emph{, e.g.}, \citep{Sima2002,Shamir2016}).
For example, as a common rule-of-the-thumb, a MNN should have at least
as many parameters as training samples. However, it is unclear why
such over-parameterized MNNs often exhibit remarkably small generalization
error (\emph{i.e.}, difference between ``training error'' and ``test
error''), even without explicit regularization \citep{Zhang2016}. 

Moreover, it has long been a mystery why MNNs often achieve low training
error \citep{Dauphin2014}. SGD is only guaranteed to converge to
critical points in which the gradient of the expected loss is zero
\citep{Bottou1998}, and, specifically, to local minima \citep{Pemantle1990}
(this is true also for regular gradient descent \citep{Lee2016}).
Since loss functions parameterized by MNN weights are non-convex,
it is unclear why does SGD often work well \textendash{} rather than
converging to sub-optimal local minima with high training error, which
are known to exist \citep{Fukumizu2000,Swirszcz2016}. Understanding
this behavior is especially relevant in important cases where SGD
does get stuck \citep{He2015} \textendash{} where training error
may be a bottleneck in further improving performance.

Ideally, we would like to quantify the probability to converge to
a local minimum as a function of the error at this minimum, where
the probability is taken with the respect to the randomness of the
initialization of the weights, the data and SGD. Specifically, we
would like to know, under which conditions this probability is very
small if the error is high, as was observed empirically (\emph{e.g.},
\citep{Dauphin2014,Goodfellow2014}). However, this seems to be a
daunting task for realistic MNNs, since it requires a characterization
of the sizes and distributions of the basins of attraction for all
local minima. 

Previous works \citep{Dauphin2014,Choromanska2014}, based on statistical
physics analogies, suggested a simpler property of MNNs: that with
high probability, local minima with high error diminish exponentially
with the number of parameters. Though proving such a geometric property
with realistic assumptions would not guarantee convergence to global
minima, it appears to be a necessary first step in this direction
(see discussion on section \ref{sec:Discussion}). It was therefore
pointed out as an open problem at the Conference of Learning Theory
(COLT) 2015. However, one has to be careful and use realistic MMN
architectures, or this problem becomes ``too easy''.

For example, one can easily achieve zero training error \citep{Nilsson1965,Baum1988}
\textendash{} if the MNN's last hidden layer has more neurons than
training samples. Such extremely wide MNNs are easy to optimize \citep{Yu1992,Huang2006,Livni2014,Shen2016,Nguyen2017}.
In this case, the hidden layer becomes linearly separable in classification
tasks, with high probability over the random initialization of the
weights. Thus, by training the last layer we get to a global minimum
(zero training error). However, such extremely wide layers are not
very useful, since they result in a huge number of weights, and serious
overfitting issues. Also, training only the last layer seems to take
little advantage of the inherently non-linear nature of MNNs.

Therefore, in this paper we are interested to understand the properties
of local and global minima, but at a more practical number of parameters
\textendash{} and when at least two weight layers are trained. For
example, Alexnet \citep{Krizhevsky2014} is trained using about 1.2
million ImageNet examples, and has about $60$ million parameters
\textendash{} $16$ million of these in the two last weight layers.
Suppose we now train the last two weight layers in such an over-parameterized
MNN. When do the sub-optimal local minima become exponentially rare
in comparison to the global minima?

\paragraph*{Main contributions.}

We focus on MNNs with a single hidden layer and piecewise linear units,
optimized using the Mean Square Error (MSE) in a supervised binary
classification task (Section \ref{sec:Preliminaries}). We define
$N$ as the number of training samples, $d_{l}$ as the width of the
$l$-th activation layer, and $g\left(x\right)\dot{<}h\left(x\right)$
as an asymptotic inequality in the leading order (formally: $\lim_{x\rightarrow\infty}\frac{\log g\left(x\right)}{\log h\left(x\right)}<1$).
We examine Differentiable Local Minima (DLMs) of the MSE: sub-optimal
DLMs where at least a fraction of $\epsilon>0$ of the training samples
are classified incorrectly, and global minima where all samples are
classified correctly. 

Our main result, Theorem \ref{thm: volume ratio}, states that, with
high probability, the total volume of the differentiable regions of
the MSE containing sub-optimal DLMs is exponentially vanishing in
comparison to the same volume of global minima, given that: 

{\assm{\label{asmp:Gaussian Input} The datapoints (MNN inputs) are
sampled from a standard normal distribution.}}

{\assm{\label{asmp: asymptotic} $N\rightarrow\infty$, $d_{0}\left(N\right)$
and $d_{1}\left(N\right)$ increase with $N$, while $\epsilon\in\left(0,1\right)$
is a constant\footnote{For brevity we will usually keep implicit the $N$ dependencies of
$d_{0}$ and $d_{1}$.}. }}

{\assm{\label{asmp: input dimension} The input dimension scales
as $\sqrt{N}\dot{<}d_{0}\dot{\leq}N$.}}

{\assm{\label{asmp: over-parameterization}The hidden layer width
scales as }}
\begin{equation}
\frac{N\log^{4}N}{d_{0}}\dot{<}d_{1}\dot{<}N\,.\label{eq: overparameterization}
\end{equation}
Importantly, we use a standard, unmodified, MNN model, and make no
assumptions on the target function. Moreover, as the number of parameters
in the MNN is approximately $d_{0}d_{1}$, we require only ``asymptotically
mild'' over-parameterization: $d_{0}d_{1}\dot{>}N\log^{4}N$ from
eq. (\ref{eq: overparameterization}). For example, if $d_{0}\propto N$,
we only require $d_{1}\dot{>}\log^{4}N$ neurons. This improves over
previously known results \citep{Yu1992,Huang2006,Livni2014,Shen2016,Nguyen2017}
\textendash{} which require an extremely wide hidden layer with $d_{1}\geq N$
neurons (and thus $Nd_{0}$ parameters) to remove sub-optimal local
minima with high probability. 

In section \ref{sec:Numerical-Experiments} we validate our results
numerically. We show that indeed the training error becomes low when
the number of parameters is close to $N$. For example, with binary
classification on CIFAR and ImageNet, with only 16 and 105 hidden
neurons (about $N/d_{0}$), respectively, we obtain less then $0.1\%$
training error. Additionally, we find that convergence to non-differentiable
critical points does not appear to be very common. 

Lastly, in section \ref{sec:Discussion} we discuss our results might
be extended, such as how to apply them to ``mildly'' non-differentiable
critical points.

\paragraph*{Plausibility of assumptions. }

Assumption \ref{asmp:Gaussian Input} is common in this type of analysis
\citep{Andoni2014,Choromanska2014,Xie2016,Tian2017,Brutzkus2017}.
At first it may appear rather unrealistic, especially since the inputs
are correlated in typical datasets. However, this no-correlation part
of the assumption may seem more justified if we recall that datasets
are many times whitened before being used as inputs. Alternatively,
if, as in our motivating question, we consider the input to the our
simple MNN to be the output of the previous layers of a deep MNN with
fixed random weights, this also tends to de-correlate inputs \citep[Figure 3]{Poole2016}.
The remaining part of assumption \ref{asmp:Gaussian Input}, that
the distribution is normal, is indeed strong, but might be relaxed
in the future, \emph{e.g.} using central limit theorem type arguments. 

In assumption \ref{asmp: asymptotic} we use this asymptotic limit
to simplify our proofs and final results. Multiplicative constants
and finite (yet large) $N$ results can be found by inspection of
the proofs. We assume a constant error $\epsilon$ since typically
the limit $\epsilon\rightarrow0$ is avoided to prevent overfitting.

In assumption \ref{asmp: input dimension}, for simplicity we have
$d_{0}\dot{\leq}N$, since in the case $d_{0}\geq N$ the input is
generically linearly separable, and sub-optimal local minima are not
a problem \citep{Gori1992,Safran2015}. Additionally, we have $\sqrt{N}\dot{<}d_{0},$
which seems very reasonable, since for example, $d_{0}/N\approx0.016$,
$0.061$ and $0.055$ MNIST, CIFAR and ImageNet, respectively. 

In assumption \ref{asmp: over-parameterization}, for simplicity we
have $d_{1}\dot{<}N$, since, as mentioned earlier, if $d_{1}\geq N$
the hidden layer is linearly separable with high probability, which
removes sub-optimal local minima. The other bound $N\log^{4}N\dot{<}d_{0}d_{1}$
is our main innovation \textendash{} a large over-parameterization
which is nevertheless asymptotically mild and improves previous results.

\paragraph*{Previous work.}

So far, general low (training or test) error guarantees for MNNs could
not be found \textendash{} unless the underlying model (MNN) or learning
method (SGD or its variants) have been significantly modified. For
example, \citep{Dauphin2014} made an analogy with high-dimensional
random Gaussian functions, local minima with high error are exponentially
rare in high dimensions; \citep{Choromanska2014,Kawaguchi2016} replaced
the units (activation functions) with independent random variables;
\citep{Pennington2017} replaces the weights and error residuals with
independent random variables; \citep{Baldi1989,Saxe2013,Hardt2016,Lu2017,Zhou2017}
used linear units; \citep{Zhang2017} used unconventional units (\emph{e.g.},
polynomials) and very large hidden layers ($d_{1}=\mathrm{poly}\left(d_{0}\right)$,
typically $\gg N$); \citep{Brutzkus2017,Du2017a,Shalev-Shwartz2017}
used a modified convnet model with less then $d_{0}$ parameters (therefore,
not a universal approximator \citep{Cybenko1989,Hornik1991}); \citep{Tian2017,Soltanolkotabi2017,Li2017}
assume the weights are initialized very close to those of the teacher
generating the labels; and \citep{Janzamin2015,Zhong2017} use a non-standard
tensor method during training. Such approaches fall short of explaining
the widespread success of standard MNN models and training practices.

Other works placed strong assumptions on the target functions. For
example, to prove convergence of the training error near the global
minimum, \citep{Gori1992} assumed linearly separable datasets, while
\citep{Safran2015} assumed strong clustering of the targets (``near''
linear-separability). Also, \citep{Andoni2014} showed a $p$-degree
polynomial is learnable by a MNN, if the hidden layer is very large
($d_{1}=\Omega\left(d_{0}^{6p}\right)$, typically $\gg N$) so learning
the last weight layer is sufficient. However, these are not the typical
regimes in which MNNs are required or used. In contrast, we make no
assumption on the target function. Other closely related results \citep{Soudry2016,Xie2016}
also used unrealistic assumptions, are discussed in section \ref{sec:Discussion},
in regards to the details of our main results. 

Therefore, in contrast to previous works, the assumptions in this
paper are applicable in \emph{some} situations (\emph{e.g.}, Gaussian
input) where a MNN trained using SGD might be used and be useful (\emph{e.g.},
have a lower test error then a linear classier). 

\section{Preliminaries and notation \label{sec:Preliminaries}}

\paragraph{Model.}

We examine a Multilayer Neural Network (MNN) with a single hidden
layer and a scalar output. The MNN is trained on a finite training
set of $N$ datapoints (features) $\mathbf{X}\triangleq\left[\mathbf{x}^{\left(1\right)},\dots,\mathbf{x}^{\left(N\right)}\right]\in\mathbb{R}^{d_{0}\times N}$
with their target labels $\mathbf{y}\triangleq\left[y^{\left(1\right)},\dots,y^{\left(N\right)}\right]^{\top}\in\left\{ 0,1\right\} {}^{N}$
\textendash{} each datapoint-label pair $\left(\mathbf{x}^{\left(n\right)},y^{\left(n\right)}\right)$
is independently sampled from some joint distribution $\mathbb{P}_{X,Y}$.
We define $\mathbf{W}=\left[\mathbf{w}_{1},\dots,\mathbf{w}_{d_{1}}\right]^{\top}\in\mathbb{R}^{d_{1}\times d_{0}}$
and $\mathbf{z}\in\mathbb{R}^{d_{1}}$ as the first and second weight
layers (bias terms are ignored for simplicity), respectively, and
$f\left(\cdot\right)$ as the common leaky rectifier linear unit (LReLU
\citep{Maas2013}) 
\begin{equation}
f\left(u\right)\triangleq ua\left(u\right)\,\,\mathrm{with}\,\,a\left(u\right)\triangleq\begin{cases}
1 & ,\,\mathrm{if}\,,u>0\\
\rho & ,\,\mathrm{if}\,u<0
\end{cases}\,,\label{eq: LReLU}
\end{equation}
for some $\rho\neq1$ (so the MNN is non-linear) , where both functions
$f$ and $a$ operate component-wise \emph{(e.g.}, for any matrix
$\mathbf{M}$: $\left(f\left(\mathbf{M}\right)\right)_{ij}=f\left(M_{ij}\right)$).
Thus, the output of the MNN  on the entire dataset can be written
as 
\begin{equation}
f\left(\mathbf{W}\mathbf{X}\right)^{\top}\mathbf{z}\in\mathbb{R}^{N}.\label{eq: MNN}
\end{equation}
We use the mean square error (MSE) loss for optimization
\begin{equation}
\mathrm{MSE}\triangleq\frac{1}{N}\left\Vert \mathbf{e}\right\Vert {}^{2}\,\,\mathrm{with}\,\,\mathbf{e}\triangleq\mathbf{y}-f\left(\mathbf{W}\mathbf{X}\right)^{\top}\mathbf{z}\,,\label{eq:MSE 2-layer}
\end{equation}
 where $\left\Vert \cdot\right\Vert $ is the standard euclidean norm.
Also, we measure the empiric performance as the fraction of samples
that are classified correctly using a decision threshold at $y=0.5$,
and denote this as the mean classification error, or MCE\footnote{Formally (this expression is not needed later): $\mathrm{MCE}\triangleq\frac{1}{2N}\sum_{n=1}^{N}\left[1+\left(1-2y^{\left(n\right)}\right)\mathrm{sign}\left(e^{\left(n\right)}-\frac{1}{2}\right)\right].$}.
Note that the variables $\mathbf{e}$, $\mathrm{MSE}$, $\mathrm{MCE}$
and other related variables (\emph{e.g.}, their derivatives) all depend
on $\mathbf{W},\mathbf{z},\mathbf{X},\mathbf{y}$ and $\rho$, but
we keep this dependency implicit, to avoid cumbersome notation.

\paragraph{Additional Notation.}

We define $g\left(x\right)\dot{<}h\left(x\right)$ if and only if
$\lim_{x\rightarrow\infty}\frac{\log g\left(x\right)}{\log h\left(x\right)}<1$
(and similarly $\dot{\leq}$ and $\dot{=}$). We denote ``$\mathbf{M}\sim\mathcal{N}$''
when $\mathbf{M}$ is a matrix with entries drawn independently from
a standard normal distribution (\emph{i.e.}, $\forall i,j$: $M_{ij}\sim\mathcal{N}\left(0,1\right)$).
The Khatari-rao product (\emph{cf. }\citep{Allman2009}) of two matrices,
$\mathbf{A}=\left[\boldsymbol{a}^{\left(1\right)},\dots,\boldsymbol{a}^{\left(N\right)}\right]\in\mathbb{R}^{d_{1}\times N}$
and $\mathbf{X}=\left[\mathbf{x}^{\left(1\right)},\dots,\mathbf{x}^{\left(N\right)}\right]\in\mathbb{R}^{d_{0}\times N}$
is defined as
\begin{equation}
\mathbf{A}\circ\mathbf{X}\triangleq\left[\boldsymbol{a}^{\left(1\right)}\otimes\mathbf{x}^{\left(1\right)},\dots,\boldsymbol{a}^{\left(N\right)}\otimes\mathbf{x}^{\left(N\right)}\right]\in\mathbb{R}^{d_{0}d_{1}\times N}\,,\label{eq:G}
\end{equation}
where $\boldsymbol{a}\otimes\mathbf{x}=\left[a_{1}\mathbf{x}^{\top},\dots,a_{d_{1}}\mathbf{x}^{\top}\right]^{\top}$is
the Kronecker product.

\section{Basic Properties of Differentiable Local minima }

MNNs are typically trained by minimizing the loss over the training
set, using Stochastic Gradient Descent (SGD), or one of its variants
(\emph{e.g.}, Adam \citep{Kingma2015}). Under rather mild conditions
\citep{Pemantle1990,Bottou1998}, SGD asymptotically converges to
local minima of the loss. For simplicity, we focus on differentiable
local minima (DLMs) of the MSE (eq. (\ref{eq:MSE 2-layer})). In section
\ref{sec:Main-Results} we will show that sub-optimal DLMs are exponentially
rare in comparison to global minima. Non-differentiable critical points,
in which some neural input (pre-activation) is exactly zero, are shown
to be numerically rare in section \ref{sec:Numerical-Experiments},
and are left for future work, as discussed in section \ref{sec:Discussion}. 

Before we can provide our results, in this section we formalize a
few necessary notions. For example, one has to define how to measure
the amount of DLMs in the over-parameterized regime: there is an infinite
number of such points, but they typically occupy only a measure zero
volume in the weight space. Fortunately, using the differentiable
regions of the MSE (definition \ref{def: differentiable region}),
the DLMs can partitioned to a finite number of equivalence groups,
so all DLMs in each region have the same error (Lemma \ref{lem: Ge=00003D0}).
Therefore, we use the volume of these regions (definition \ref{def: angular volume})
as the relevant measure in our theorems.

\paragraph{Differentiable regions of the MSE.}

The MSE is a piecewise differentiable function of $\mathbf{W}$, with
at most $2^{d_{1}N}$ differentiable regions, defined as follows.
\begin{defn}
\label{def: differentiable region}For any $\mathbf{A}\in\left\{ \rho,1\right\} ^{d_{1}\times N}$
we define the corresponding differentiable region
\begin{equation}
\mathcal{D}_{\mathbf{A}}\left(\mathbf{X}\right)\triangleq\left\{ \mathbf{W}|a\left(\mathbf{W}\mathbf{X}\right)=\mathbf{A}\right\} \subset\mathbb{R}^{d_{1}\times d_{0}}\,.\label{eq: differentiable region}
\end{equation}
Also, any DLM $\left(\mathbf{W},\mathbf{z}\right)$, for which $\mathbf{W}\in\mathcal{D}_{\mathbf{A}}\left(\mathbf{X}\right)$
is denoted as ``in $\mathcal{D}_{\mathbf{A}}\left(\mathbf{X}\right)$''. 
\end{defn}
Note that $\mathcal{D}_{\mathbf{A}}\left(\mathbf{X}\right)$ is an
open set, since $a\left(0\right)$ is undefined (from eq. \ref{eq: LReLU}).
Clearly, for all $\mathbf{W}\in\mathcal{D}_{\mathbf{A}}\left(\mathbf{X}\right)$
the MSE is differentiable, so any local minimum can be non-differentiable
only if it is not in any differentiable region. Also, all DLMs in
a differentiable region are equivalent, as we prove on appendix section
\ref{sec:First-order-condition proof}:
\begin{lem}
\label{lem: Ge=00003D0}At all DLMs in $\mathcal{D}_{\mathbf{A}}\left(\mathbf{X}\right)$
the residual error $\mathbf{e}$ is identical, and furthermore
\begin{equation}
\left(\mathbf{A}\circ\mathbf{X}\right)\mathbf{e}=0\,.\label{eq:Ge=00003D0}
\end{equation}
\end{lem}
The proof is directly derived from the first order necessary condition
of DLMs ($\nabla\mathrm{MSE}=0$) and their stability. Note that Lemma
\ref{lem: Ge=00003D0} constrains the residual error $\mathbf{e}$
in the over-parameterized regime: $d_{0}d_{1}\geq N$. In this case
eq. (\ref{eq:Ge=00003D0}) implies $\mathbf{e}=0$, if $\mathrm{rank}\left(\mathbf{A}\circ\mathbf{X}\right)=N$.
Therefore, we must have $\mathrm{rank}\left(\mathbf{\mathbf{A}\circ\mathbf{X}}\right)<N$
for sub-optimal DLMs to exist.  Later, we use similar rank-based
constraints to bound the volume of differentiable regions which contain
DLMs with high error. Next, we define this volume formally.

\paragraph{Angular Volume.}

From its definition (eq. (\ref{eq: differentiable region})) each
region $\mathcal{D}_{\mathbf{A}}\left(\mathbf{X}\right)$ has an infinite
volume in $\mathbb{R}^{d_{1}\times d_{0}}$: if we multiply a row
of $\mathbf{W}$ by a positive scalar, we remain in the same region.
Only by rotating the rows of $\mathbf{W}$ can we move between regions.
We measure this ``angular volume''  of a region in a probabilistic
way: we randomly sample the rows of $\mathbf{W}$ from an isotropic
distribution,\emph{ e.g.}, standard Gaussian: $\mathbf{W}\sim\mathcal{N}$,
and measure the probability to fall in $\mathcal{D}_{\mathbf{A}}\left(\mathbf{X}\right)$,
arriving to the following
\begin{defn}
\label{def: angular volume}For any region $\mathcal{R}\subset\mathbb{R}^{d_{1}\times d_{0}}$.
The \emph{angular volume} of $\mathcal{R}$ is 
\begin{equation}
\mathcal{V}\left(\mathcal{R}\right)\triangleq\mathbb{P}_{\mathbf{W\sim\mathcal{N}}}\left(\mathbf{W}\in\mathcal{R}\right)\,.\label{eq: measure of volume}
\end{equation}
\end{defn}

\section{Main Results\label{sec:Main-Results}}

Some of the DLMs are global minima, in which $\mathbf{e}=0$ and so,
$\mathrm{MCE}=\mathrm{MSE}=0$, while other DLMs are sub-optimal local
minima in which $\mathrm{MCE>}\epsilon>0$. We would like to compare
the angular volume (definition \ref{def: angular volume}) corresponding
to both types of DLMs. Thus, we make the following definitions. 
\begin{defn}
\label{def: local minima regions}We define\footnote{More formally: if $\mathcal{A}\left(\mathbf{X},\mathbf{y},\epsilon\right)$
is the set of $\mathbf{A}\in\left\{ \rho,1\right\} ^{d_{1}\times N}$
for which $\mathcal{D}_{\mathbf{A}}\!\left(\mathbf{X}\right)$ contains
a DLM with $\mathrm{MCE}=\epsilon$, then $\forall\epsilon>0$, $\mathcal{\mathcal{L}}_{\epsilon}\left(\mathbf{X},\mathbf{y}\right)\triangleq\bigcup_{\epsilon^{\prime}\geq\epsilon}\left[\bigcup_{\mathbf{A}\in\mathcal{A}\left(\mathbf{X},\mathbf{y},\epsilon^{\prime}\right)}\mathcal{D}_{\mathbf{A}}\!\left(\mathbf{X}\right)\right]$
and $\mathcal{G}\left(\mathbf{X},\mathbf{y}\right)\triangleq\bigcup_{\mathbf{A}\in\mathcal{A}\left(\mathbf{X},\mathbf{y},0\right)}\mathcal{D}_{\mathbf{A}}\!\left(\mathbf{X}\right)$.} $\text{\ensuremath{\mathcal{\mathcal{L}}}}_{\epsilon}\subset\mathbb{R}^{d_{1}\times d_{0}}$
as the union of differentiable regions containing sub-optimal DLMs
with $\mathrm{MCE}\!>\!\epsilon$ , and $\mathcal{G}\subset\mathbb{R}^{d_{1}\times d_{0}}$
as the union of differentiable regions containing global minima with
$\mathrm{MCE}=0$. 
\end{defn}
\begin{defn}
\label{def: gamma epsilon}We define the constant $\gamma_{\epsilon}$
as $\gamma_{\epsilon}\triangleq0.23\max\left[\lim_{N\rightarrow\infty}\left(d_{0}\left(N\right)/N\right),\epsilon\right]^{3/4}$
if $\rho\neq\left\{ 0,1\right\} $, and $\gamma_{\epsilon}\triangleq0.23\epsilon^{3/4}$
if $\rho=0$.
\end{defn}
In this section, we use assumptions \ref{asmp:Gaussian Input}-\ref{asmp: over-parameterization}
(stated in section \ref{sec:Introduction}) to bound the angular volume
of the region $\mathcal{\mathcal{L}}_{\epsilon}$ encapsulating all
sub-optimal DLMs, the region $\mathcal{G}$, encapsulating all global
minima, and the ratio between the two.

\paragraph{Angular volume of sub-optimal DLMs.}

First, in appendix section \ref{sec: bad local minima proof} we
prove the following upper bound in expectation
\begin{thm}
\label{thm: Main theorem}Given assumptions \ref{asmp:Gaussian Input}-\ref{asmp: over-parameterization},
the expected angular volume of sub-optimal DLMs, with $\mathrm{MCE}>\epsilon>0$,
is exponentially vanishing in $N$ as 
\[
\E_{\mathbf{X}\sim\mathcal{N}}\mathcal{V}\left(\mathcal{L}_{\epsilon}\left(\mathbf{X},\mathbf{y}\right)\right)\dot{\leq}\exp\left(-\gamma_{\epsilon}N^{3/4}\left[d_{1}d_{0}\right]^{1/4}\right)\,.
\]
\end{thm}
and, using Markov inequality, its immediate probabilistic corollary
\begin{cor}
\label{cor: main theorem corollary} Given assumptions \ref{asmp:Gaussian Input}-\ref{asmp: over-parameterization},
for any $\delta>0$ (possibly a vanishing function of $N$), we have,
with probability $1-\delta$, that the angular volume of sub-optimal
DLMs, with $\mathrm{MCE}>\epsilon>0$, is exponentially vanishing
in $N$ as 
\[
\mathbb{\mathcal{V}}\left(\mathcal{\mathcal{L}}_{\epsilon}\left(\mathbf{X},\mathbf{y}\right)\right)\dot{\leq}\frac{1}{\delta}\exp\left(-\gamma_{\epsilon}N^{3/4}\left[d_{1}d_{0}\right]^{1/4}\right)
\]
\end{cor}
\emph{Proof idea} \emph{of Theorem} \ref{thm: Main theorem}: we first
show that in differentiable regions with $\mathrm{MCE}>\epsilon>0$,
the condition in Lemma \ref{lem: Ge=00003D0}, $\left(\mathbf{A}\circ\mathbf{X}\right)\mathbf{e}=0$,
implies that $\mathbf{A}=a\left(\mathbf{W}\mathbf{X}\right)$ must
have a low rank. Then, we show that, when $\mathbf{X}\sim\mathcal{N}$
and $\mathbf{W}\sim\mathcal{N}$ , the matrix $\mathbf{A}=a\left(\mathbf{W}\mathbf{X}\right)$
has a low rank with exponentially low probability. Combining both
facts, we obtain the bound.

\paragraph{Existence of global minima.}

Next, to compare the volume of sub-optimal DLMs with that of global
minima, in appendix section \ref{subsec: Overfitting solution} we
show first that, generically, global minima do exist (using a variant
of the proof of \citep[Theorem 1]{Baum1988}):
\begin{thm}
\label{thm: overfitting solution}For any \textbf{$\mathbf{y}\in\left\{ 0,1\right\} {}^{N}$}
and $\mathbf{X}\in\mathbb{R}^{d_{0}\times N}$ almost everywhere\footnote{\emph{i.e.}, the set of entries of $\mathbf{X}$, for which the following
statement does not hold, has zero measure (Lebesgue).} we find matrices $\mathbf{W}^{*}\in\mathbb{R}^{d_{1}^{*}\times d_{0}}$
and $\mathbf{z}^{*}\in\mathbb{R}^{d_{1}^{*}}$, such that $\mathbf{y}=f\left(\mathbf{W}^{*}\mathbf{X}\right)^{\top}\mathbf{z}^{*}$
, where $d_{1}^{*}\triangleq4\left\lceil N/\left(2d_{0}-2\right)\right\rceil $
and $\forall i,n:\,\mathbf{w}_{i}^{\top}\mathbf{x}^{\left(n\right)}\neq0.$
Therefore, every MNN with $d_{1}\geq d_{1}^{*}$ has a DLM which achieves
zero error $\mathbf{e}=0$.
\end{thm}
 Recently \citep[Theorem 1]{Zhang2016} similarly proved that a 2-layer
MNN with approximately $2N$ parameters can achieve zero error. However,
that proof required $N$ neurons (similarly to \citep{Nilsson1965,Baum1988,Yu1992,Huang2006,Livni2014,Shen2016}),
while Theorem \ref{thm: overfitting solution} here requires much
less: approximately $d_{1}^{*}\approx2N/d_{0}$. Also, \citep[Theorem 3.2]{Hardt2016}
showed a deep residual network with $N\log N$ parameters can achieve
zero error. In contrast, here we require just one hidden layer with
$2N$ parameters. 

Note the construction in Theorem \ref{thm: overfitting solution}
here achieves zero training error by overfitting to the data realization,
so it is not expected to be a ``good'' solution in terms of generalization.
To get good generalization, one needs to add additional assumptions
on the data ($\mathbf{X}$ and $\mathbf{y}$). Such a possible (common
yet insufficient for MNNs) assumption is that the problem is ``realizable'',
\emph{i.e.}, there exist a small ``solution MNN'', which achieves
low error. For example, in the zero error case: 

{\assm{ \label{asmp: Teacher} \textbf{\emph{(Optional)}} The labels
are generated by some teacher $\mathbf{y}=f\left(\mathbf{W}^{*}\mathbf{X}\right)^{\top}\mathbf{z}^{*}$
with weight matrices $\mathbf{W}^{*}\in\mathbb{R}^{d_{1}^{*}\times d_{0}}$
and $\mathbf{z}^{*}\in\mathbb{R}^{d_{1}^{*}}$ independent of $\mathbf{X}$,
for some $d_{1}^{*}\dot{<}N/d_{0}$.}}

This assumption is not required for our main result (Theorem \ref{thm: volume ratio})
\textendash{} it is merely helpful in improving the following lower
bound on $\mathcal{V}\left(\mathcal{G}\right)$.

\paragraph{Angular volume of global minima.}

We prove in appendix section \ref{sec: global minima proof}:
\begin{thm}
\label{thm: Main theorem 2}Given assumptions \ref{asmp:Gaussian Input}-\ref{asmp: input dimension},
we set $\delta\dot{=}\sqrt{\frac{8}{\pi}}d_{0}^{-1/2}+2d_{0}^{1/2}\sqrt{\log d_{0}}/N$
and $d_{1}^{*}=2N/d_{0}$ , or if assumption \ref{asmp: Teacher}
holds, we set $d_{1}^{*}$ as in this assumption. Then, with probability
$1-\delta$, the angular volume of global minima is lower bounded
as,
\[
\mathcal{V}\left(\mathcal{G}\left(\mathbf{X},\mathbf{y}\right)\right)\dot{>}\exp\left(-d_{1}^{*}d_{0}\log N\right)\dot{\geq}\exp\left(-2N\log N\right)\,.
\]
\end{thm}
\emph{Proof idea:} First, we lower bound $\mathcal{V}\left(\mathcal{G}\right)$
with the angular volume of a single differentiable region of one global
minimum $\left(\mathbf{W^{*}},\mathbf{z}^{*}\right)$ \textendash{}
either from Theorem \ref{thm: overfitting solution}, or from assumption
\ref{asmp: Teacher}. Then we show that this angular volume is lower
bounded when $\mathbf{W}\sim\mathcal{N}$, given a certain angular
margin between the datapoints in $\mathbf{X}$ and the rows of $\mathbf{W}^{*}$.
We then calculate the probability of obtaining this margin when $\mathbf{X}\sim\mathcal{N}$.
Combining both results, we obtain the final bound.

\paragraph{Main result: angular volume ratio.}

Finally, combining Theorems \ref{thm: Main theorem} and \ref{thm: Main theorem 2}
it is straightforward to prove our main result in this paper, as we
do in appendix section \ref{sec:proof of Volume-ratio}:
\begin{thm}
\label{thm: volume ratio} Given assumptions \ref{asmp:Gaussian Input}-\ref{asmp: input dimension},
we set $\delta\doteq\sqrt{\frac{8}{\pi}}d_{0}^{-1/2}+2d_{0}^{1/2}\sqrt{\log d_{0}}/N$.
Then, with probability $1-\delta$, the angular volume of sub-optimal
DLMs, with $\mathrm{MCE}>\epsilon>0$, is exponentially vanishing
in N, in comparison to the angular volume of global minima with $\mathrm{MCE=0}$
\[
\frac{\mathcal{V}\left(\mathcal{L}_{\epsilon}\left(\mathbf{X},\mathbf{y}\right)\right)}{\mathcal{V}\left(\mathcal{G}\left(\mathbf{X},\mathbf{y}\right)\right)}\dot{\leq}\exp\left(-\gamma_{\epsilon}N^{3/4}\left[d_{1}d_{0}\right]^{1/4}\right)\dot{\leq}\exp\left(-\gamma_{\epsilon}N\log N\right)\,.
\]
\end{thm}

\section{Numerical Experiments\label{sec:Numerical-Experiments}}

\begin{figure*}[t]
\begin{centering}
\includegraphics[width=0.7\columnwidth]{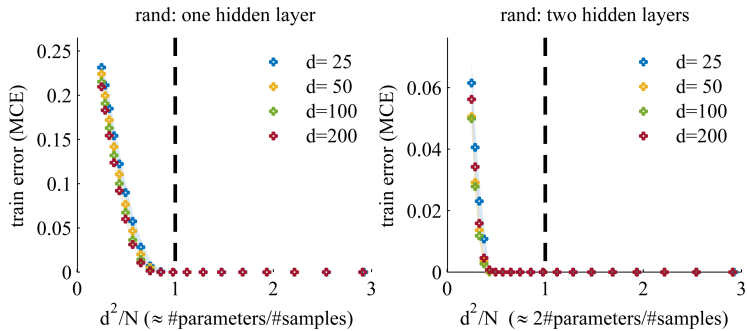}
\par\end{centering}
\caption{\textbf{Gaussian data: final training error (mean$\pm$std, 30 repetitions)
in the over-parameterized regime is low (right of the dashed black
line). }We trained MNNs with one and two hiddens layer (with widths
equal to $d=d_{0}$) on a synthetic random dataset in which $\forall n=1,\dots,N$,
\textbf{$\mathbf{x}^{\left(n\right)}$ }was drawn from a normal distribution
$\mathcal{N}\left(0,1\right)$,\textbf{ }and $y^{\left(n\right)}=\pm1$
with probability $0.5$. \label{fig:training error}}
\end{figure*}

Theorem \ref{thm: volume ratio} implies that, with ``asymptotically
mild'' over-parameterization (\emph{i.e. }in which \#parameters =$\tilde{\Omega}\left(N\right)$),
differentiable regions in weight space containing sub-optimal DLMs
(with high MCE) are exponentially small in comparison with the same
regions for global minima. Since these results are asymptotic in $N\rightarrow\infty$,
in this section we examine it numerically for a finite number of samples
and parameters. We perform experiments on random data, MNIST, CIFAR10
and ImageNet-ILSVRC2012. In each experiment, we used ReLU activations
($\rho=0$), a binary classification target (we divided the original
classes to two groups), MSE loss for optimization (eq. (\ref{eq:MSE 2-layer})),
and MCE to determine classification error. Additional implementation
details are given in appendix part \ref{sec:Implementation-details}. 

First, on the small synthetic Gaussian random data (matching our assumptions)
we perform a scan on various networks and dataset sizes. With either
one or two hidden layers (Figure \ref{fig:training error}) , the
error goes to zero when the number of non-redundant parameters (approximately
$d_{0}d_{1}$) is greater than the number of samples, as suggested
by our asymptotic results. Second, on the non-syntehtic datasets,
MNIST, CIFAR and ImageNet (In ImageNet we downsampled the images to
size $64\times64$, to allow input whitening) we only perform a simulation
with a single 1-hidden layer MNN for which $\#\mathrm{parameters}\approx N$
, and again find (Table \ref{tab:CIFAR-and-ImageNet}) that the final
error is zero (for MNIST and CIFAR) or very low (ImageNet).

Lastly, in Figure \ref{fig:Differntiability-of-local-minima} we find
that, on the Gaussian dataset, the inputs to the hidden neurons converge
to a distinctly non-zero value. This indicates we converged to DLMs
\textendash{} since non-differentiable critical points must have zero
neural inputs. Note that occasionally, during optimization, we could
find some neural inputs with very low values near numerical precision
level, so convergence to non-differentiable minima may be possible.
However, as explained in the next section, as long as the number of
neural inputs equal to zero are not too large, our bounds also hold
for these minima.

\begin{table}
\begin{centering}
\begin{tabular}{|c|c|c|c|c|c|}
\hline 
 & $\mathrm{MCE}$ & $d_{0}$ & $d_{1}$ & $N$ & \textbf{$\#\mathrm{parameters}/N$}\tabularnewline
\hline 
\hline 
MNIST & $\mathbf{0\%}$ & 784 & 89 & $7\cdot10^{4}$ & \textbf{$\mathbf{0.999}$}\tabularnewline
\hline 
CIFAR & $\mathbf{0\%}$ & 3072 & 16 & $5\cdot10^{4}$ & \textbf{$\mathbf{0.983}$}\tabularnewline
\hline 
ImageNet (downsampled to $64\times64$) & $\mathbf{0.1\%}$ & 12288 & 105 & $128\cdot10^{4}$ & \textbf{$\mathbf{1.008}$}\tabularnewline
\hline 
\end{tabular}
\par\end{centering}
\caption{\textbf{Binary classification of MNIST, CIFAR and ImageNet: 1-hidden
layer achieves very low training error (MCE)} \textbf{with a few hidden
neurons, so that $\mathrm{\#parameters}\approx d_{0}d_{1}\approx N$}.
In ImageNet we downsampled the images to allow input whitening. \label{tab:CIFAR-and-ImageNet}}
\end{table}

\begin{wrapfigure}{I}{0.5\columnwidth}%
\begin{centering}
\includegraphics[height=2.6cm]{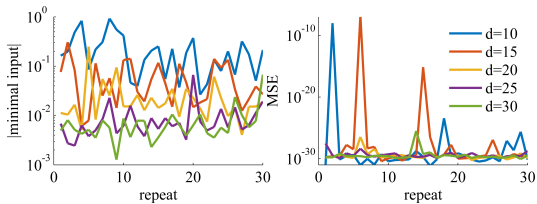}
\par\end{centering}
\caption{\textbf{Gaussian data: convergence of the MSE to }\textbf{\emph{differentiable}}\textbf{
local minima, as indicated by the convergence of the neural inputs
to distinctly non-zero values. }We trained MNNs with one hidden layer
on the Gaussian dataset from Figure \ref{fig:training error}, with
various widths $d=d_{0}=d_{1}$ and $N=\left\lfloor d^{2}/5\right\rfloor $
for $1000$ epochs, then decreased the learning rate exponentially
for another $1000$ epochs. This was repeated $30$ times. For all
$d$ and repeats, we see that \emph{(left}) the final absolute value
of the minimal neural input (\emph{i.e.}, $\min_{i,n}\left|\mathbf{w}_{i}^{\top}\mathbf{x}^{\left(n\right)}\right|$)
in the range of $10^{-3}-10^{0}$, which is much larger then (r\emph{ight})
the final MSE error for all $d$ and all repeats \textendash{} in
the range $10^{-31}-10^{-7}$.\label{fig:Differntiability-of-local-minima}}
\end{wrapfigure}%

\section{Discussion\label{sec:Discussion}}

In this paper we examine Differentiable Local Minima (DLMs) of the
empiric loss of Multilayer Neural Networks (MNNs) with one hidden
layer, scalar output, and LReLU nonlinearities (section \ref{sec:Preliminaries}).
We prove (Theorem \ref{thm: volume ratio}) that with high probability
the angular volume (definition \ref{def: angular volume}) of sub-optimal
DLMs is exponentially vanishing in comparison to the angular volume
of global minima (definition \ref{def: local minima regions}), under
assumptions \ref{asmp:Gaussian Input}-\ref{asmp: over-parameterization}.
This results from an upper bound on sub-optimal DLMs (Theorem \ref{thm: Main theorem})
and a lower bound on global minima (Theorem \ref{thm: Main theorem 2}).

\paragraph{Convergence of SGD to DLMs.}

These results suggest a mechanism through which low training error
is obtained in such MNNs. However, they do not guarantee it. One issue
is that sub-optimal DLMs may have exponentially large basins of attraction.
We see two possible paths that might address this issue in future
work, using additional assumptions on $\mathbf{y}$. One approach
is to show that, with high probability, \emph{no sub optimal DLM}
falls within the vanishingly small differentiable regions we bounded
in Theorem \ref{thm: Main theorem}. Another approach would be to
bound the size of these basins of attraction, by showing that sufficiently
large of number of differentiable regions near the DLM are also vanishingly
small (other methods might also help here \citep{Freeman2016a}).
Another issue is that SGD might get stuck near differentiable saddle
points, if their Hessian does not have strictly negative eigenvalues
(\emph{i.e.}, the strict saddle property \citep{Sun2015}). It should
be straightforward to show that such points also have exponentially
vanishing angular volume, similar to sub-optimal DLMs. Lastly, SGD
might also converge to non-differentiable critical points, which we
discuss next.

\paragraph{Non-differentiable critical points.}

The proof of Theorem \ref{thm: Main theorem} stems from a first order
necessary condition (Lemma \ref{lem: Ge=00003D0}): $\left(\mathbf{A}\circ\text{\textbf{X}}\right)\mathbf{e}=0$,
which is true for any DLM. However, non-differentiable critical points,
in which some neural inputs are exactly zero, may also exist (though,
numerically, they don't seem very common \textendash{} see Figure
\ref{fig:Differntiability-of-local-minima}). In this case, to derive
a similar bound, we can replace the condition with $\mathbf{P}\left(\mathbf{A}\circ\mathbf{X}\right)\mathbf{e}=0$,
where $\mathbf{P}$ is a projection matrix to the subspace orthogonal
to the non-differentiable directions. As long as there are not too
many zero neural inputs, we should be able to obtain similar results.
For example, if only a constant ratio $r$ of the neural inputs are
zero, we can simply choose $\mathbf{P}$ to remove all rows of $\left(\mathbf{A}\circ\mathbf{X}\right)$
corresponding to those neurons, and proceed with exactly the same
proof as before, with $d_{1}$ replaced with $\left(1-r\right)d_{1}$.
It remains a theoretical challenge to find reasonable assumptions
under which the number of non-differentiable directions (\emph{i.e.},
zero neural inputs) does not become too large.

\paragraph{Related results.}

Two works have also derived related results using the $\left(\mathbf{A}\circ\text{\textbf{X}}\right)\mathbf{e}=0$
condition from Lemma \ref{lem: Ge=00003D0}. In \citep{Soudry2016},
it was noticed that an infinitesimal perturbation of \textbf{$\mathbf{A}$}
makes the matrix $\mathbf{A}\circ\text{\textbf{X}}$ full rank with
probability 1 \citep[Lemma 13]{Allman2009} \textendash{} which entails
that $\mathbf{e}=0$ at all DLMs. Though a simple and intuitive approach,
such an infinitesimal perturbation is problematic: from continuity,
it cannot change the original MSE at sub-optimal DLMs \textendash{}
unless the weights go to infinity, or the DLM becomes non-differentiable
\textendash{} which are both undesirable results. An extension of
this analysis was also done to constrain $\mathbf{e}$ using the singular
values of $\mathbf{A}\circ\text{\textbf{X}}$ \citep{Xie2016}, deriving
bounds that are easier to combine with generalization bounds. Though
a promising approach, the size of the sub-optimal regions (where the
error is high) does not vanish exponentially in the derived bounds.
More importantly, these bounds require assumptions on the activation
kernel spectrum $\gamma_{m}$, which do not appear to hold in practice
(\emph{e.g.}, \citep[Theorems 1,3]{Xie2016} require $m\gamma_{m}\gg1$
to hold with high probability, while $m\gamma_{m}<10^{-2}$ in \citep[Figure 1]{Xie2016}).

\paragraph{Modifications and extensions.}

There are many relatively simple extensions of these results: the
Gaussian assumption could be relaxed to other near-isotropic distributions
(\emph{e.g.}, sparse-land model, \citep[Section 9.2]{Elad2010}) and
other convex loss functions are possible instead of the quadratic
loss. More challenging directions are extending our results to MNNs
with multi-output and multiple hidden layers, or combining our training
error results with novel generalization bounds which might be better
suited for MNNs (\emph{e.g.}, \citep{Feng2016,Sokolic2016,Dziugaite2017})
than previous approaches \citep{Zhang2016}.

\subsubsection*{Acknowledgments}

The authors are grateful to A. Z. Abassi, D. Carmon, R. Giryes, and
especially to Y. Carmon for all the insightful advice we have received
during this work, and to I. Hubara, I. Safran, and R. Meir for helpful
comments on the manuscript. The research was supported by the Gruss
Lipper Charitable Foundation, and by the Intelligence Advanced Research
Projects Activity (IARPA) via Department of Interior/ Interior Business
Center (DoI/IBC) contract number D16PC00003. The U.S. Government is
authorized to reproduce and distribute reprints for Governmental purposes
notwithstanding any copyright annotation thereon. Disclaimer: The
views and conclusions contained herein are those of the authors and
should not be interpreted as necessarily representing the official
policies or endorsements, either expressed or implied, of IARPA, DoI/IBC,
or the U.S. Government.

{\small{}\bibliographystyle{iclr2018_conference_no_url}

}{\small \par}

\newpage{}

\part*{Supplementary information - Appendix}

The appendix is divided into three parts. In part \ref{part:Proofs-of-Main-Results}
we prove all the main theorems mentioned in the paper. Some of these
rely on other technical results, which we prove later in part \ref{part:Technical-Results-and-Proofs}.
Lastly, in part \ref{sec:Implementation-details} we give additional
numerical details and results. First, however, we define additional
notation (some already defined in the main paper) and mention some
known results, which we will use in our proofs.

\section*{Extended Preliminaries }
\begin{itemize}
\item The indicator function $\mathcal{I}\left(\mathcal{A}\right)\triangleq\begin{cases}
1 & ,\mathrm{if}\,\mathcal{A}\\
0 & ,\mathrm{else}
\end{cases}$, for any event $\mathcal{A}$.
\item Kronecker's delta $\delta_{ij}\triangleq\mathcal{I}\left(i=j\right)$.
\item The Matrix $\mathbf{I}_{d}$ as the identity matrix in $\mathbb{R}^{d\times d}$,
and $\mathbf{I}_{d\times k}$ is the relevant $\mathbb{R}^{d\times k}$
upper left sub-matrix of the identity matrix.
\item $\left[L\right]\triangleq\left\{ 1,2,\dots,L\right\} $
\item The vector $\mathbf{m}_{n}$ as the $n$'th column of a matrix $\mathbf{M}$,
unless defined otherwise (then $\mathbf{m}_{n}$ will be a row of
$\mathbf{M}$).
\item $\mathbf{M}>0$ implies that $\forall i,j:\,M_{ij}>0$.
\item $\mathbf{M}_{S}$ is the matrix composed of the columns of $\mathbf{M}$
that are in the index set $S$.
\item A property holds ``$\mathbf{M}$-almost everywhere'' (a.e. for short),
if the set of entries of $\mathbf{M}$ for which the property does
not hold has zero measure (Lebesgue).
\item $\left\Vert \mathbf{v}\right\Vert _{0}=\sum_{i=1}^{d}\mathcal{I}\left(v_{i}>0\right)$
is the $L_{0}$ ``norm'' that counts the number of non-zero values
in $\mathbf{v}\in\mathbb{R}^{d}$. 
\item If $\mathbf{x}\sim\mathcal{N}\left(\boldsymbol{\mu},\boldsymbol{\Sigma}\right)$
the $\mathbf{x}$ is random Gaussian vector. 
\item $\phi\left(x\right)\triangleq\frac{1}{\sqrt{2\pi}}\exp\left(-\frac{1}{2}x^{2}\right)$
as the univariate Gaussian probability density function.
\item $\Phi\left(x\right)\triangleq\int_{-\infty}^{x}\phi\left(u\right)du$
as the Gaussian cumulative distribution function. 
\item $B\left(x,y\right)$ as the beta function.
\end{itemize}
Lastly, we recall the well known Markov Inequality:
\begin{fact}
\textbf{\emph{(Markov Inequality)\label{fact: Markov-Inequality}}}
For any random variable $X\geq0$, we have $\forall\eta>0$
\[
\mathbb{P}\left(X\geq\eta\right)\leq\frac{\mathbb{E}X}{\eta}\,.
\]
\end{fact}

\part{Proofs of the main results\label{part:Proofs-of-Main-Results}}

\section{First order condition: Proof of Lemma \emph{\ref{lem: Ge=00003D0}}\label{sec:First-order-condition proof}}
\begin{lem}
\textbf{\emph{(Lemma \ref{lem: Ge=00003D0} restated)}} At all DLMs
in $\mathcal{D}_{\mathbf{A}}\left(\mathbf{X}\right)$ the residual
error $\mathbf{e}$ is identical, and furthermore
\begin{equation}
\left(\mathbf{A}\circ\mathbf{X}\right)\mathbf{e}=0\,.\label{eq:Ge=00003D0-1}
\end{equation}
\end{lem}
\begin{proof}
Let $\mathbf{W}=\left[\mathbf{w}_{1},\dots,\mathbf{w}_{d_{1}}\right]^{\top}\in\mathcal{D}_{\mathbf{A}}\left(\mathbf{X}\right)$,
$\mathbf{G\triangleq}\mathbf{A}\circ\mathbf{X}\in\mathbb{R}^{d_{0}d_{1}\times N}$,
$\tilde{\mathbf{W}}=\mathrm{diag}\left(\mathbf{z}\right)\mathbf{W}=\left[\tilde{\mathbf{w}}_{1},\dots,\tilde{\mathbf{w}}_{d_{1}}\right]^{\top}$
and $\tilde{\mathbf{w}}\triangleq\mathrm{vec}\left(\tilde{\mathbf{W}}^{\top}\right)\in\mathbb{R}^{d_{0}d_{1}}$,
where $\mathrm{diag}\left(\mathbf{v}\right)$ is the diagonal matrix
with $\mathbf{v}$ in its diagonal, and $\mathrm{vec}\left(\mathbf{M}\right)$
is vector obtained by stacking the columns of the matrix $\mathbf{M}$
on top of one another. Then, we can re-write the MSE (eq. (\ref{eq:MSE 2-layer}))
as 
\begin{equation}
\mathrm{MSE}=\frac{1}{N}\left\Vert \mathbf{y}-\mathbf{G}^{\top}\tilde{\mathbf{w}}\right\Vert {}^{2}=\frac{1}{N}\left\Vert \mathbf{e}\right\Vert {}^{2},\label{eq: MSE tilde-1}
\end{equation}
where $\mathbf{G}^{\top}\tilde{\mathbf{w}}$ is the output of the
MNN. Now, if $\left(\mathbf{W},\mathbf{z}\right)$ is a DLM of the
MSE in eq. (\ref{eq:MSE 2-layer}), then there is no infinitesimal
perturbation of $\left(\mathbf{W},\mathbf{z}\right)$ which reduces
this MSE. 

Next, for each row $i$, we will show that $\partial\mathrm{MSE}/\partial\tilde{\mathbf{w}}_{i}=0$,
since otherwise we can find an infinitesimal perturbation of $\left(\mathbf{W},\mathbf{z}\right)$
which decreases the MSE, contradicting the assumption that $(\mathbf{W},\mathbf{z})$
is a local minimum. For each row $i$, we divide into two cases:

First, we consider the case $z_{i}\neq0$. In this case, any infinitesimal
perturbation $\mathbf{q}_{i}$ in $\mathbf{\tilde{\mathbf{w}}_{i}}$
can be produced by an infinitesimal perturbation in $\mathbf{w}_{i}$:
$\tilde{\mathbf{w}}_{i}+\mathbf{q}_{i}=(\mathbf{w}_{i}+\mathbf{q}_{i}/z_{i})z_{i}.$
Therefore, unless the gradient $\partial\mathrm{MSE}/\partial\tilde{\mathbf{w}}_{i}$
is equal to zero, we can choose an infinitesimal perturbation $\mathbf{q}_{i}$
in the opposite direction to this gradient, which will decrease the
MSE. 

Second, we consider the case $z_{i}=0$. In this case, the MSE is
not affected by changes made exclusively to $\mathbf{w}_{i}$. Therefore,
all $\mathbf{w}_{i}$ derivatives of the MSE are equal to zero ($\partial^{k}\mathrm{MSE}/\partial^{k}\mathbf{w}_{i}$,
to any order $k$) . Also, since we are at a differentiable local
minimum, $\partial\mathrm{MSE}/\partial z_{i}=0.$ Thus, using a Taylor
expansion, if we perturb $(\mathbf{w}_{i},z_{i})$ by ($\mathbf{\hat{w}}_{i}$,$\hat{z}_{i}$
) then the MSE is perturbed by
\[
\hat{z}_{i}\mathbf{\hat{w}}_{i}^{\top}\frac{\partial}{\partial\tilde{\mathbf{w}}_{i}}\frac{\partial}{\partial z_{i}}\mathrm{MSE}+O(\hat{z}_{i}^{2})
\]
Therefore, unless $\partial^{2}\mathrm{MSE}/\left(\partial\mathbf{w}_{i}\partial z_{i}\right)=0$
we can choose \textbf{$\mathbf{\hat{w}}_{i}$} and a sufficiently
small $\hat{z}_{i}$ such that the MSE is decreased. Lastly, using
the chain rule 
\begin{align*}
\frac{\partial}{\partial z_{i}}\frac{\partial}{\partial\mathbf{w}_{i}}\mathrm{MSE} & =\frac{\partial}{\partial z_{i}}\left[z_{i}\frac{\partial}{\partial\tilde{\mathbf{w}}_{i}}\mathrm{MSE}\right]=\frac{\partial}{\partial\tilde{\mathbf{w}}_{i}}\mathrm{MSE}\,.
\end{align*}

Thus, $\partial\mathrm{MSE}/\partial\tilde{\mathbf{w}}_{i}=0$. This
implies that $\tilde{\mathbf{w}}$ is also a DLM\footnote{Note that the converse argument is not true \textendash{} a DLM in
$\tilde{\mathbf{w}}$ might not be a DLM in $\left(\mathbf{W},\mathbf{z}\right)$.} of eq. (\ref{eq: MSE tilde-1}), which entails
\begin{equation}
0=-\frac{N}{2}\frac{\partial}{\partial\tilde{\mathbf{w}}_{i}}\mathrm{MSE}=\mathbf{G}\left(\mathbf{y}-\mathbf{G}^{\top}\tilde{\mathbf{w}}\right)\,.\label{eq: nabla MSE =00003D0-1}
\end{equation}
Since $\mathbf{G}=\mathbf{A}\circ\mathbf{X}$ and $\mathbf{e}=\mathbf{y}-\mathbf{G}^{\top}\tilde{\mathbf{w}}$
this proves eq. (\ref{eq:Ge=00003D0-1}). Now, for any two solutions
$\tilde{\mathbf{w}}_{1}$ and $\tilde{\mathbf{w}}_{2}$ of eq. (\ref{eq: nabla MSE =00003D0-1}),
we have
\[
0=\mathbf{G}\left(\mathbf{y}-\mathbf{G}^{\top}\tilde{\mathbf{w}}_{1}\right)-\mathbf{G}\left(\mathbf{y}-\mathbf{G}^{\top}\tilde{\mathbf{w}}_{1}\right)=\mathbf{G}\mathbf{G}^{\top}\left(\tilde{\mathbf{w}}_{2}-\tilde{\mathbf{w}}_{1}\right)\,.
\]
Multiplying by $\left(\tilde{\mathbf{w}}_{2}-\tilde{\mathbf{w}}_{1}\right)^{\top}$from
the left we obtain
\[
\left\Vert \mathbf{G}^{\top}\left(\tilde{\mathbf{w}}_{2}-\tilde{\mathbf{w}}_{1}\right)\right\Vert ^{2}=0\Rightarrow\mathbf{G}^{\top}\left(\tilde{\mathbf{w}}_{2}-\tilde{\mathbf{w}}_{1}\right)=0\,.
\]
Therefore, the MNN output and the residual error $\mathbf{e}$ are
equal for all DLMs in $\mathcal{D}_{\mathbf{A}}\left(\mathbf{X}\right)$. 
\end{proof}

\section{\label{sec: bad local minima proof}Sub-optimal differentiable local
minima: Proof of Theorem \ref{thm: Main theorem} and its corollary}
\begin{thm}
\emph{(}\textbf{\emph{Theorem \ref{thm: Main theorem} restated}}\emph{)}
Given assumptions \ref{asmp:Gaussian Input}-\ref{asmp: over-parameterization},
the expected angular volume of sub-optimal DLMs, with $\mathrm{MCE}>\epsilon>0$,
is exponentially vanishing in $N$ as 
\[
\E_{\mathbf{X}\sim\mathcal{N}}\mathcal{V}\left(\mathcal{L}_{\epsilon}\left(\mathbf{X},\mathbf{y}\right)\right)\dot{\leq}\exp\left(-\gamma_{\epsilon}N^{3/4}\left[d_{1}d_{0}\right]^{1/4}\right)\,,
\]
where $\gamma_{\epsilon}\triangleq0.23\max\left[\lim_{N\rightarrow\infty}\left(d_{0}\left(N\right)/N\right),\epsilon\right]^{3/4}$
if $\rho\neq\left\{ 0,1\right\} $, and $\gamma_{\epsilon}\triangleq0.23\epsilon^{3/4}$
if $\rho=0$.
\end{thm}
To prove this theorem we upper bound the angular volume of $\mathcal{\mathcal{L}_{\epsilon}}$
(definition \ref{def: local minima regions}), \emph{i.e.}, differentiable
regions in which there exist DLMs with $\mathrm{MCE}>\epsilon>0$.
Our proof uses the first order necessary condition for DLMs from Lemma
\ref{lem: Ge=00003D0}, $\left(\mathbf{A}\circ\mathbf{X}\right)\mathbf{e}=0,$
to find which configurations of $\mathbf{A}$ allow for a high residual
error $\mathbf{e}$ with $\mathrm{MCE}>\epsilon>0$. In these configurations
$\mathbf{A}\circ\mathbf{X}$ cannot have full rank, and therefore,
as we show (Lemma \ref{lem: error lower bounds S-main paper} below),
$\mathbf{A}=a\left(\mathbf{W}\mathbf{X}\right)$ must have a low rank.
However, $\mathbf{A}=a\left(\mathbf{W}\mathbf{X}\right)$ has a low
rank with exponentially low probability when $\mathbf{X}\sim\mathcal{N}$
and $\mathbf{W}\sim\mathcal{N}$ (Lemmas \ref{lem: Rank configurations bound-main paper}
and \ref{lem: Product of Gaussian Matrices-main paper} below). Thus,
we derive an upper bound on $\E_{\mathbf{X}\sim\mathcal{N}}\mathcal{V}\left(\mathcal{L}_{\epsilon}\left(\mathbf{X},\mathbf{y}\right)\right)$.

Before we begin, let us recall some notation: $\left[L\right]\triangleq\left\{ 1,2,\dots,L\right\} $,$\mathbf{M}>0$
implies that $\forall i,j:\,M_{ij}>0$, $\mathbf{M}_{S}$ is the matrix
composed of the columns of $\mathbf{M}$ that are in the index set
$S$, $\left\Vert \mathbf{v}\right\Vert _{0}$ as the $L_{0}$ ``norm''
that counts the number of non-zero values in $\mathbf{v}$. First
we consider the case $\rho\neq0$. Also, we denote $K_{r}\triangleq\max\left[N\epsilon,rd_{0}\right]$
.

First we consider the case $\rho\neq0$. 

From definition \ref{def: angular volume} of the angular volume
\begin{align}
\mathbb{E}_{\mathbf{X}\sim\mathcal{N}}\mathbb{\mathcal{V}}\left(\mathcal{\mathcal{L}}_{\epsilon}\left(\mathbf{X},\mathbf{y}\right)\right)= & \mathbb{P}_{\left(\mathbf{X,}\mathbf{y}\right)\mathbf{\sim\mathbb{P}}_{X,Y},\mathbf{W\sim\mathcal{N}}}\left(\mathbf{W}\in\mathcal{\mathcal{L}}_{\epsilon}\left(\mathbf{X},\mathbf{y}\right)\right)\nonumber \\
\overset{\left(1\right)}{\leq} & \mathbb{P}_{\left(\mathbf{X,}\mathbf{y}\right)\mathbf{\sim\mathbb{P}}_{X,Y},\mathbf{W\sim\mathcal{N}}}\left(\exists\mathbf{A}\in\left\{ \rho,1\right\} ^{d_{1}\times N},\,\mathbf{W}\in\mathcal{D}_{\mathbf{A}}\left(\mathbf{X}\right),\mathbf{v}\in\mathbb{R}^{N}:\,\left(\mathbf{A}\circ\mathbf{X}\right)\mathbf{v}=0\,,N\epsilon\leq\left\Vert \mathbf{\mathbf{v}}\right\Vert _{0}\right)\nonumber \\
\overset{\left(2\right)}{=} & \mathbb{P}_{\mathbf{X}\sim\mathcal{N},\mathbf{W\sim\mathcal{N}}}\left(\exists\mathbf{A}\in\left\{ \rho,1\right\} ^{d_{1}\times N},\,\mathbf{W}\in\mathcal{D}_{\mathbf{A}}\left(\mathbf{X}\right),\mathbf{v}\in\mathbb{R}^{N}:\,\left(\mathbf{A}\circ\mathbf{X}\right)\mathbf{v}=0\,,N\epsilon\leq\left\Vert \mathbf{\mathbf{v}}\right\Vert _{0}\right)\nonumber \\
\overset{\left(3\right)}{\leq} & \mathbb{P}_{\mathbf{X}\sim\mathcal{N},\mathbf{W\sim\mathcal{N}}}\left(\exists S\subset\left[N\right]:\,\left|S\right|\geq\max\left[N\epsilon,\mathrm{rank}\left(a\left(\mathbf{W}\mathbf{X}_{S}\right)\right)d_{0}+1\right]\right)\nonumber \\
= & \mathbb{E}_{\mathbf{X}\sim\mathcal{N}}\left[\mathbb{P}_{\mathbf{W\sim\mathcal{N}}}\left(\exists S\subset\left[N\right]:\,\left|S\right|\geq\max\left[N\epsilon,\mathrm{rank}\left(a\left(\mathbf{W}\mathbf{X}_{S}\right)\right)d_{0}+1\right]|\mathbf{X}\right)\right]\nonumber \\
\overset{\left(4\right)}{\leq} & \mathbb{E}_{\mathbf{X}\sim\mathcal{N}}\left[\sum_{r=1}^{N/d_{0}}\mathbb{P}_{\mathbf{W}\sim\mathcal{N}}\left(\exists S\subset\left[N\right]:\,\left|S\right|=K_{r},\mathrm{rank}\left(a\left(\mathbf{W}\mathbf{X}_{S}\right)\right)=r|\mathbf{X}\right)\right]\nonumber \\
\overset{\left(5\right)}{\leq} & \mathbb{E}_{\mathbf{X}\sim\mathcal{N}}\left[\sum_{r=1}^{N/d_{0}}\sum_{S:\left|S\right|=K_{r}}\mathbb{P}_{\mathbf{W}\sim\mathcal{N}}\left(\mathrm{rank}\left(a\left(\mathbf{W}\mathbf{X}_{S}\right)\right)=r|\mathbf{X}\right)\right],\label{eq: bad minima bound 1}
\end{align}
where
\begin{enumerate}
\item If we are at DLM a in $\mathcal{D}_{\mathbf{A}}\left(\mathbf{X}\right)$,
then Lemma \ref{lem: Ge=00003D0} implies $\left(\mathbf{A}\circ\mathbf{X}\right)\mathbf{\mathbf{e}}=0$.
Also, if $e^{\left(n\right)}=0$ on some sample, we necessarily classify
it correctly, and therefore $\mathrm{MCE}\leq\left\Vert \mathbf{e}\right\Vert _{0}/N$.
Since $\mathrm{MCE}>\epsilon$ in $\mathcal{\mathcal{L}}_{\epsilon}$
this implies that $N\epsilon<\left\Vert \mathbf{e}\right\Vert _{0}\,.$
Thus, this inequality holds for $\mathbf{v}=\mathbf{e}$.
\item We apply assumption \ref{asmp:Gaussian Input}, that $\mathbf{X}\sim\mathcal{N}$.
\item Assumption \ref{asmp: over-parameterization} implies $d_{0}d_{1}\dot{>}N\log^{4}N\geq N$.
Thus, we can apply the following Lemma, proven in appendix section
\ref{sec:Proof-of-rank-amplification}:
\begin{lem}
\label{lem: error lower bounds S-main paper}Let $\mathbf{X}\in\mathbb{R}^{d_{0}\times N}$,
$\mathbf{\mathbf{A}}\in\left\{ \rho,1\right\} ^{d_{1}\times N}$,
$S\subset\left[N\right]$ and $d_{0}d_{1}\geq N$. Then, simultaneously
for every possible $\mathbf{A}$ and $S$ such that 
\[
\left|S\right|\leq\mathrm{rank}\left(\mathbf{A}_{S}\right)d_{0}\,,
\]
we have that, $\mathbf{X}$-a.e., $\nexists\mathbf{v}\in\mathbb{R}^{N}$
such that $v_{n}\neq0$ $\forall n\in S$ and $\left(\mathbf{A}\circ\mathbf{X}\right)\mathbf{v}=0$
.
\end{lem}
\item Recall that $K_{r}\triangleq\max\left[N\epsilon,rd_{0}\right]$. We
use the union bound over all possible ranks $r\geq1$: we ignore the
$r=0$ case since for $\rho\neq0$ (see eq. (\ref{eq: LReLU})) there
is zero probability that $\mathrm{rank}\left(a\left(\mathbf{W}\mathbf{X}_{S}\right)\right)=0$
for some non-empty $S$. For each rank $r\geq1$, it is required that
$\left|S\right|>K_{r}=\max\left[N\epsilon,rd_{0}\right]$, so $\left|S\right|=K_{r}$
is a relaxation of the original condition, and thus its probability
is not lower. 
\item We again use the union bound over all possible subsets $S$ of size
$K_{r}$.
\end{enumerate}
Thus, from eq. (\ref{eq: bad minima bound 1}), we have
\begin{align}
 & \mathbb{E}_{\mathbf{X}\sim\mathcal{N}}\mathbb{\mathcal{V}}\left(\mathcal{\mathcal{L}}_{\epsilon}\left(\mathbf{X},\mathbf{y}\right)\right)\nonumber \\
\leq & \sum_{r=1}^{N/d_{0}}\sum_{S:\left|S\right|=K_{r}}\mathbb{E}_{\mathbf{X}\sim\mathcal{N}}\left[\mathbb{P}_{\mathbf{W}\sim\mathcal{N}}\left(\mathrm{rank}\left(a\left(\mathbf{W}\mathbf{X}_{S}\right)\right)=r|\mathbf{X}\right)\right]\nonumber \\
\overset{\left(1\right)}{=} & \sum_{r=1}^{N/d_{0}}\left(\begin{array}{c}
N\\
K_{r}
\end{array}\right)\mathbb{P}_{\mathbf{X}\sim\mathcal{N},\mathbf{W}\sim\mathcal{N}}\left(\mathrm{rank}\left(a\left(\mathbf{W}\mathbf{X}_{[K_{r}]}\right)\right)=r\right)\nonumber \\
\overset{\left(2\right)}{\dot{\leq}} & \sum_{r=1}^{N/d_{0}}\left(\begin{array}{c}
N\\
K_{r}
\end{array}\right)2^{K_{r}+rd_{0}\left(\log d_{1}+\log K_{r}\right)+r^{2}}\mathbb{P}_{\mathbf{X}\sim\mathcal{N},\mathbf{W}\sim\mathcal{N}}\left(\mathbf{W}\mathbf{X}_{[K_{r}/2]}>0\right)\nonumber \\
\overset{\left(3\right)}{\dot{\leq}} & \sum_{r=1}^{N/d_{0}}\left(\begin{array}{c}
N\\
K_{r}
\end{array}\right)2^{K_{r}+rd_{0}\left(\log d_{1}+\log K_{r}\right)+r^{2}}\exp\left(-0.2K_{r}\left(2\frac{d_{0}d_{1}}{K_{r}}\right)^{1/4}\right)\nonumber \\
\overset{\left(4\right)}{\dot{\leq}} & \sum_{r=1}^{N/d_{0}}2^{N\log N}\exp\left(-0.23N^{3/4}\left[d_{1}d_{0}\right]^{1/4}\max\left[\epsilon,rd_{0}/N\right]^{3/4}\right)\nonumber \\
\overset{\left(5\right)}{\dot{\leq}} & \exp\left(-\gamma_{\epsilon}N^{3/4}\left[d_{1}d_{0}\right]^{1/4}\right)\,.\label{eq: bad minima bound 2}
\end{align}

\begin{enumerate}
\item Since we take the expectation over $\mathbf{X}$, the location of
$S$ does not affect the probability. Therefore, we can set without
loss of generality $S=\left[K_{r}\right]$.
\item Note that $r\leq N/d_{0}\dot{<}\min\left[d_{0},d_{1}\right]$ from
assumptions \ref{asmp: input dimension} and \ref{asmp: over-parameterization}.
Thus, with $k=K_{r}\geq d_{0}$, we apply the following Lemma, proven
in appendix section \ref{sec:Proof-of-Lemma-rank-configurations}:
\begin{lem}
\label{lem: Rank configurations bound-main paper}Let $\mathbf{X}\in\mathbb{R}^{d_{0}\times k}$
be a random matrix with independent and identically distributed columns,
and $\mathbf{W}\in\mathbb{R}^{d_{1}\times d_{0}}$ an independent
standard random Gaussian matrix. Then, in the limit $\min\left[k,d_{0},d_{1}\right]\dot{>}r$,
\[
\mathbb{P}\left(\mathrm{rank}\left(a\left(\mathbf{W}\mathbf{X}\right)\right)=r\right)\dot{\leq}2^{k+rd_{0}\left(\log d_{1}+\log k\right)+r^{2}}\mathbb{P}\left(\mathbf{W}\mathbf{X}_{\left[\left\lfloor k/2\right\rfloor \right]}>0\right)\,.
\]
\end{lem}
\item Note that $K_{r}\geq N\epsilon\dot{=}N>2d_{1}$, and $\mathrm{min}\left[K_{r},d_{0},d_{1}\right]\dot{>}d_{0}d_{1}/K_{r}\dot{>}1$
from assumptions \ref{asmp: asymptotic} and \ref{asmp: over-parameterization}.
Thus, we apply the following Lemma (with $\mathbf{C}=\mathbf{X^{\top}},$$\mathbf{B}=\mathbf{W}^{\top}$,
$M=d_{0}$, $L=d_{1}$ and $N=K_{r}/2$), proven in appendix section
\ref{sec:Proof-of-product of Gaussian}:
\begin{lem}
\label{lem: Product of Gaussian Matrices-main paper}Let $\mathbf{C}\in\mathbb{R}^{N\times M}$
and $\mathbf{B}\in\mathbb{R}^{M\times L}$ be two independent standard
random Gaussian matrices. Without loss of generality, assume $N\geq L$,
and denote $\alpha\triangleq ML/N$. Then, in the regime $M\leq N$
and in the limit $\mathrm{min}\left[N,M,L\right]\dot{>}\alpha\dot{>}1$,
we have
\[
\mathbb{P}\left(\mathbf{C}\mathbf{B}>0\right)\dot{\leq}\exp\left(-0.4N\alpha^{1/4}\right)\,.
\]
\end{lem}
\item We use $rd_{0}\leq N$, $\left(\begin{array}{c}
N\\
K_{r}
\end{array}\right)\leq2^{N}$,$K_{r}\leq N$, and $d_{1}\dot{<}N$ (from assumption \ref{asmp: over-parameterization})
and $r^{2}\leq N^{2}/d_{0}^{2}\dot{<}N$ (from assumption (\ref{asmp: input dimension}))
to simplify the combintaorial expressions.
\item First, note that $r=1$ is the maximal term in the sum, so we can
neglect the other, exponentially smaller, terms. Second, from assumption
\ref{asmp: input dimension} we have $d_{0}\dot{\leq}N$, so 
\[
\lim_{N\rightarrow\infty}0.23\max\left[\epsilon,d_{0}\left(N\right)/N\right]^{3/4}=0.23\max\left[\epsilon,\lim_{N\rightarrow\infty}d_{0}\left(N\right)/N\right]^{3/4}=\gamma_{\epsilon}\,.
\]
Third, from assumption \ref{asmp: over-parameterization} we have
$N\log^{4}N\dot{<}d_{0}d_{1}$, so the $2^{N\log N}$ term is negligible.
\end{enumerate}
Thus, 
\begin{equation}
\mathbb{E}_{\mathbf{X}\sim\mathcal{N}}\mathbb{\mathcal{V}}\left(\mathcal{\mathcal{L}}_{\epsilon}\left(\mathbf{X},\mathbf{y}\right)\right)\dot{\leq}\exp\left(-\gamma_{\epsilon}N^{3/4}\left[d_{1}d_{0}\right]^{1/4}\right)\,.\label{eq: bad areas upper bound}
\end{equation}
which proves the Theorem for the case $\rho\neq0$. 

Next, we consider the case $\rho=0$. In this case, we need to change
transition $\left(4\right)$ in eq. (\ref{eq: bad minima bound 1}),
so the sum starts from $r=0$, since now we can have $\mathrm{rank}\left(a\left(\mathbf{W}\mathbf{X}_{S}\right)\right)=0$.
Following exactly the same logic (except the modification to the sum),
we only need to modify transition $\left(5\right)$in eq. (\ref{eq: bad minima bound 2})
\textendash{} since now the maximal term in the sum is at $r=0$.
This entails $\gamma_{\epsilon}=0.23\epsilon^{3/4}$.

\QEDA
\begin{cor}
\textbf{\emph{(Corollary \ref{cor: main theorem corollary} restated)}}
Given assumptions \ref{asmp:Gaussian Input}-\ref{asmp: over-parameterization},
for any $\delta>0$ (possibly a vanishing function of $N$), we have,
with probability $1-\delta$, that the angular volume of sub-optimal
DLMs, with $\mathrm{MCE}>\epsilon>0$, is exponentially vanishing
in $N$ as 
\[
\mathbb{\mathcal{V}}\left(\mathcal{\mathcal{L}}_{\epsilon}\left(\mathbf{X},\mathbf{y}\right)\right)\dot{\leq}\frac{1}{\delta}\exp\left(-\gamma_{\epsilon}N^{3/4}\left[d_{1}d_{0}\right]^{1/4}\right)
\]
\end{cor}
\begin{proof}
Since $\mathbb{\mathcal{V}}\left(\mathcal{\mathcal{L}}_{\epsilon}\left(\mathbf{X},\mathbf{y}\right)\right)\geq0$
we can use Markov's Theorem (Fact \ref{fact: Markov-Inequality})
$\forall\eta>0$: 
\[
\mathbb{P}_{\mathbf{X}\sim\mathcal{N}}\left(\mathbb{\mathcal{V}}\left(\mathcal{\mathcal{L}}_{\epsilon}\left(\mathbf{X},\mathbf{y}\right)\right)<\eta\right)>1-\frac{\mathbb{E}_{\mathbf{X}\sim\mathcal{N}}\mathbb{\mathcal{V}}\left(\mathcal{\mathcal{L}}_{\epsilon}\left(\mathbf{X},\mathbf{y}\right)\right)}{\eta}
\]
denoting $\eta=\frac{1}{\delta}\mathbb{E}_{\mathbf{X}\sim\mathcal{N}}\mathbb{\mathcal{V}}\left(\mathcal{\mathcal{L}}_{\epsilon}\left(\mathbf{X},\mathbf{y}\right)\right)$,
and using Theorem (\ref{thm: Main theorem}) we prove the corollary.
\begin{align*}
1-\delta & <\mathbb{P}_{\mathbf{X}\sim\mathcal{N}}\left(\mathbb{\mathcal{V}}\left(\mathcal{\mathcal{L}}_{\epsilon}\left(\mathbf{X},\mathbf{y}\right)\right)<\frac{1}{\delta}\mathbb{E}_{\mathbf{X}\sim\mathcal{N}}\mathbb{\mathcal{V}}\left(\mathcal{\mathcal{L}}_{\epsilon}\left(\mathbf{X},\mathbf{y}\right)\right)\right)\\
 & <\mathbb{P}_{\mathbf{X}\sim\mathcal{N}}\left(\mathbb{\mathcal{V}}\left(\mathcal{\mathcal{L}}_{\epsilon}\left(\mathbf{X},\mathbf{y}\right)\right)\dot{\leq}\frac{1}{\delta}\exp\left(-\gamma_{\epsilon}N^{3/4}\left[d_{1}d_{0}\right]^{1/4}\right)\right)
\end{align*}
where we note that replacing a regular inequality$<$ with inequality
in the leading order$\dot{\leq}$ only removes constraints, and therefore
increases the probability.
\end{proof}

\section{Construction of global minima: Proof of Theorem \ref{thm: overfitting solution}:
\label{subsec: Overfitting solution}}

Recall the LReLU non-linearity 
\[
f\left(x\right)\triangleq\begin{cases}
\rho x & ,\,\mathrm{if}\,x<0\\
x & ,\,\mathrm{if}\,x\geq0
\end{cases}
\]
 in eq. (\ref{eq: LReLU}), where $\rho\neq1$.
\begin{thm}
\label{thr: consturction of overfitting solution} \textbf{\emph{(Theorem
\ref{thm: overfitting solution} restated)}} For any \textbf{$\mathbf{y}\in\left\{ 0,1\right\} {}^{N}$}
and $\mathbf{X}\in\mathbb{R}^{d_{0}\times N}$ almost everywhere we
can find matrices $\mathbf{W}^{*}\in\mathbb{R}^{d_{1}^{*}\times d_{0}}$
and $\mathbf{z}^{*}\in\mathbb{R}^{d_{1}^{*}}$, such that $\mathbf{y}=f\left(\mathbf{W}^{*}\mathbf{X}\right)^{\top}\mathbf{z}^{*}$
, where $d_{1}^{*}\triangleq4\left\lceil N/\left(2d_{0}-2\right)\right\rceil $
and $\forall i,n:\,\mathbf{w}_{i}^{\top}\mathbf{x}^{\left(n\right)}\neq0.$.
Therefore, every MNN with $d_{1}\geq d_{1}^{*}$ has a DLM which achieves
zero error ($\mathrm{MSE}=\mathrm{MCE}=0$).
\end{thm}
We prove the existence of a solution ($\mathbf{W}^{*}$,$\mathbf{z}^{*}$),
by explicitly constructing it. This construction is a variant of \citep[Theorem 1]{Baum1988},
except we use LReLU without bias and MSE \textendash{} instead of
threshold units with bias and MCE. First, we note that for any $\epsilon_{1}>\epsilon_{2}>0$,
the following trapezoid function can be written as a scaled sum of
four LReLU:
\begin{align}
\tau\left(x\right) & \triangleq\begin{cases}
0 & ,\,\mathrm{if}\,\left|x\right|>\epsilon_{1}\\
1 & ,\,\mathrm{if}\,\left|x\right|\leq\epsilon_{2}\\
\frac{\epsilon_{1}-\left|x\right|}{\epsilon_{1}-\epsilon_{2}} & ,\,\mathrm{if}\,\epsilon_{2}<\left|x\right|\leq\epsilon_{1}
\end{cases}\label{eq: trapezoid function}\\
 & =\frac{1}{\epsilon_{1}-\epsilon_{2}}\frac{1}{1-\rho}\left[f\left(x+\epsilon_{1}\right)-f\left(x+\epsilon_{2}\right)-f\left(x-\epsilon_{2}\right)+f\left(x-\epsilon_{1}\right)\right]\,.\nonumber 
\end{align}
Next, we examine the set of data points which are classified to $1$:
$\mathcal{S}^{+}\triangleq\left\{ n\in\left[N\right]|y^{\left(n\right)}=1\right\} $
. Without loss of generality, assume $\left|\mathcal{S}^{+}\right|\leq\frac{N}{2}.$
We partition $\mathcal{S}^{+}$ to 
\[
K=\left\lceil \frac{\left|\mathcal{S}^{+}\right|}{d_{0}-1}\right\rceil \leq\left\lceil \frac{N}{2\left(d_{0}-1\right)}\right\rceil 
\]
 subsets $\left\{ \mathcal{S}_{i}^{+}\right\} _{i=1}^{K}$, each with
no more than $d_{0}-1$ samples. For almost any dataset we can find
$K$ hyperplanes passing through the origin, with normals $\left\{ \mathbf{\tilde{w}}_{i}\right\} _{i=1}^{K}$
such that each hyperplane contains all $d_{0}-1$ points in subset
$\mathcal{S}_{i}^{+}$, \emph{i.e.},
\begin{equation}
\tilde{\mathbf{w}}_{i}^{\top}\mathbf{X}_{\mathcal{S}_{i}^{+}}=0\,,\label{eq: w zero}
\end{equation}
but no other point, so $\forall n\notin\mathcal{S}_{i}^{+}:\,\tilde{\mathbf{w}}_{i}^{\top}\mathbf{x}^{\left(n\right)}\neq0\,,$

If $\epsilon_{1},\epsilon_{2}$ in eq. (\ref{eq: trapezoid function})
are sufficiently small ($\forall n\notin\mathcal{S}_{i}^{+}:\,\left|\tilde{\mathbf{w}}_{i}^{\top}\mathbf{x}^{\left(n\right)}\right|>\epsilon_{1}$)
then we have 
\[
\tau\left(\tilde{\mathbf{w}}_{i}^{\top}\mathbf{x}^{\left(n\right)}\right)=\begin{cases}
1 & ,\,\mathrm{if}\,n\in\mathcal{S}_{i}^{+}\\
0 & ,\,\mathrm{else}
\end{cases}\,.
\]
Then we have
\begin{equation}
\sum_{i=1}^{K}\tau\left(\tilde{\mathbf{w}}_{i}^{\top}\mathbf{x}^{\left(n\right)}\right)=\begin{cases}
1 & ,\,\mathrm{if}\,n\in\mathcal{S}^{+}\\
0 & ,\,\mathrm{else}
\end{cases}\label{eq: MNN output}
\end{equation}
which gives the correct classification on all the data points. Thus,
from eq. (\ref{eq: trapezoid function}), we can construct a MNN with
\[
d_{1}^{*}=4K
\]
 hidden neurons which achieves zero error. This is straightforward
to do if we have a bias in each neuron. To construct this MNN even
without bias, we first find a vector $\hat{\mathbf{w}}_{i}$ such
that 
\begin{equation}
\hat{\mathbf{w}}_{i}^{\top}\left[\mathbf{X}_{\mathcal{S}_{i}^{+}},\tilde{\mathbf{w}}_{i}\right]=\left[1,\dots,1,1,0\right].\label{eq: w_hat}
\end{equation}
Note that this is possible since $\left[\mathbf{X}_{\mathcal{S}_{i}^{+}},\tilde{\mathbf{w}}_{i}\right]$
has full rank $\mathbf{X}$-a.e. (the matrix $\mathbf{X}_{\mathcal{S}_{i}^{+}}\in\mathbb{R}^{d_{0}\times d_{0}-1}$
has, $\mathbf{X}$-a.e., one zero left eigenvector, which is $\tilde{\mathbf{w}}_{i}$,
according to eq. (\ref{eq: w zero})). Additionally, we can set 
\begin{equation}
\left\Vert \tilde{\mathbf{w}}_{i}\right\Vert =\left\Vert \hat{\mathbf{w}}_{i}\right\Vert \,,\label{eq: norm equality}
\end{equation}
since changing the scale of $\mathbf{w}_{i}$ would not affect the
validity of eq. (\ref{eq: w zero}). Then, we denote 
\begin{align*}
\mathbf{w}_{i}^{\left(1\right)} & \triangleq\tilde{\mathbf{w}}_{i}+\epsilon_{1}\hat{\mathbf{w}}_{i}\,\,;\,\,\mathbf{w}_{i}^{\left(2\right)}\triangleq\tilde{\mathbf{w}}_{i}+\epsilon_{2}\hat{\mathbf{w}}_{i}\\
\mathbf{w}_{i}^{\left(3\right)} & \triangleq\tilde{\mathbf{w}}_{i}-\epsilon_{2}\hat{\mathbf{w}}_{i}\,\,;\,\,\mathbf{w}_{i}^{\left(4\right)}\triangleq\tilde{\mathbf{w}}_{i}-\epsilon_{1}\hat{\mathbf{w}}_{i}\,.
\end{align*}
 Note, from eqs. (\ref{eq: w zero}) and (\ref{eq: w_hat}) that this
choice satisfies 
\begin{equation}
\forall n\in\mathcal{S}_{i}^{+}:\,\mathbf{w}_{i}^{\left(j\right)\top}\mathbf{x}^{\left(n\right)}=\begin{cases}
\epsilon_{1} & ,\,\mathrm{if}\,j=1\\
\epsilon_{2} & ,\,\mathrm{if}\,j=2\\
-\epsilon_{2} & ,\,\mathrm{if}\,j=3\\
-\epsilon_{1} & ,\,\mathrm{if}\,j=4
\end{cases}\,.\label{eq: cond 1}
\end{equation}
Also, to ensure that $\forall n\notin\mathcal{S}_{i}^{+}$ the sign
of $\mathbf{w}_{i}^{\left(j\right)^{\top}}\mathbf{x}^{\left(n\right)}$
does not change for different $j$, for some $\beta,\gamma<1$ we
define
\begin{equation}
\epsilon_{1}=\beta\frac{\min_{n\notin\mathcal{S}_{i}^{+}}\left|\tilde{\mathbf{w}}_{i}^{\top}\mathbf{x}^{\left(n\right)}\right|}{\max_{n\notin\mathcal{S}_{i}^{+}}\left|\hat{\mathbf{w}}_{i}^{\top}\mathbf{x}^{\left(n\right)}\right|}\,,\epsilon_{2}=\gamma\epsilon_{1}\,,\label{eq: epsilon 1,2}
\end{equation}
where with probability $1$, $\min_{n\notin\mathcal{S}_{i}^{+}}\left|\tilde{\mathbf{w}}_{i}^{\top}\mathbf{x}^{\left(n\right)}\right|>0$
and $\max_{n\notin\mathcal{S}_{i}^{+}}\left|\hat{\mathbf{w}}_{i}^{\top}\mathbf{x}^{\left(n\right)}\right|>0$.
Defining 
\begin{align}
\mathbf{W}_{i} & \triangleq\left[\mathbf{w}_{i}^{\left(1\right)},\mathbf{w}_{i}^{\left(2\right)},\mathbf{w}_{i}^{\left(3\right)},\mathbf{w}_{i}^{\left(4\right)}\right]^{\top}\in\mathbb{R}^{4K\times d_{0}}\label{eq: W_i}\\
\mathbf{z}_{i} & \triangleq\left[1,-1,-1,1\right]^{\top}\in\mathbb{R}^{4}\nonumber 
\end{align}
and combining all the above facts, we have
\begin{align*}
 & f\left(\mathbf{W}_{i}\mathbf{x}^{\left(n\right)}\right)^{\top}\mathbf{z}_{i}\\
 & =\frac{1}{\epsilon_{1}-\epsilon_{2}}\frac{1}{1-\rho}\left[f\left(\mathbf{w}_{i}^{\left(1\right)\top}\mathbf{x}^{\left(n\right)}\right)-f\left(\mathbf{w}_{i}^{\left(2\right)\top}\mathbf{x}^{\left(n\right)}\right)-f\left(\mathbf{w}_{i}^{\left(3\right)\top}\mathbf{x}^{\left(n\right)}\right)+f\left(\mathbf{w}_{i}^{\left(3\right)\top}\mathbf{x}^{\left(n\right)}\right)\right]\\
 & =\frac{1}{\epsilon_{1}-\epsilon_{2}}\frac{1}{1-\rho}\left[f\left(\mathbf{\tilde{w}}_{i}^{\top}\mathbf{x}^{\left(n\right)}+\epsilon_{1}\hat{\mathbf{w}}_{i}^{\top}\mathbf{x}^{\left(n\right)}\right)-f\left(\mathbf{\tilde{w}}_{i}^{\top}\mathbf{x}+\epsilon_{2}\hat{\mathbf{w}}_{i}^{\top}\mathbf{x}^{\left(n\right)}\right)\right.\\
 & -\left.f\left(\mathbf{\tilde{w}}_{i}^{\top}\mathbf{x}^{\left(n\right)}-\epsilon_{2}\hat{\mathbf{w}}_{i}^{\top}\mathbf{x}^{\left(n\right)}\right)+f\left(\mathbf{\tilde{w}}_{i}^{\top}\mathbf{x}^{\left(n\right)}-\epsilon_{1}\hat{\mathbf{w}}_{i}^{\top}\mathbf{x}^{\left(n\right)}\right)\right]\\
 & =\begin{cases}
1 & ,\,\mathrm{if}\,n\in\mathcal{S}_{i}^{+}\\
0 & ,\,\mathrm{else}
\end{cases}\,.
\end{align*}
Thus, for 
\begin{align*}
\mathbf{W}^{*} & =\left[\mathbf{W}_{1}^{\top},\dots,\mathbf{W}_{K}^{\top}\right]^{\top}\in\mathbb{R}^{4\times d_{0}}\\
\mathbf{z}^{*} & =\frac{1}{\epsilon_{1}-\epsilon_{2}}\frac{1}{1-\rho}\cdot\left[\mathbf{z}_{1},\dots,\mathbf{z}_{K}\right]\in\mathbb{R}^{4K}
\end{align*}
we obtain a MNN that implements 
\[
f\left(\mathbf{W}^{*}\mathbf{x}^{\left(n\right)}\right)^{\top}\mathbf{z}^{*}=\begin{cases}
1 & ,\,\mathrm{if}\,n\in\mathcal{S}^{+}\\
0 & ,\,\mathrm{else}
\end{cases}
\]
and thus achieves zero error. Clearly, from this construction, if
$\mathbf{w}_{i}$ is a row of $\mathbf{W}^{*}$, then $\forall n\in\mathcal{S}_{i}^{+}$,$\forall i:\,\left|\mathbf{w}_{i}^{\top}\mathbf{x}^{\left(n\right)}\right|\geq\epsilon_{2}$,
and with probability 1 $\forall n\notin\mathcal{S}_{i}^{+}$,$\forall i:\,\left|\mathbf{w}_{i}^{\top}\mathbf{x}^{\left(n\right)}\right|>0$,
so this construction does not touch any non-differentiable region
of the MSE.\QEDA

\section{\label{sec: global minima proof} Global minima: Proof of Theorem
\ref{thm: Main theorem 2}}
\begin{thm}
\textbf{\emph{(Theorem \ref{thm: Main theorem 2} restated).}} Given
assumptions \ref{asmp:Gaussian Input}-\ref{asmp: input dimension},
we set $\delta\dot{=}\sqrt{\frac{8}{\pi}}d_{0}^{-1/2}+2d_{0}^{1/2}\sqrt{\log d_{0}}/N$
and $d_{1}^{*}=2N/d_{0}$ , or if assumption \ref{asmp: Teacher}
holds, we set $d_{1}^{*}$ as in this assumption. Then, with probability
$1-\delta$, the angular volume of global minima is lower bounded
as,
\[
\mathcal{V}\left(\mathcal{G}\left(\mathbf{X},\mathbf{y}\right)\right)\dot{>}\exp\left(-d_{1}^{*}d_{0}\log N\right)\dot{\geq}\exp\left(-2N\log N\right)\,.
\]
\end{thm}
In this section we lower bound the angular volume of $\mathcal{G}$
(definition \ref{def: local minima regions}), \emph{i.e.}, differentiable
regions in which there exist DLMs with $\mathrm{MCE}=0$. We lower
bound $\mathcal{V}\left(\mathcal{G}\right)$ using the angular volume
corresponding to the differentiable region containing a single global
minimum. 

From assumption \ref{asmp: over-parameterization}, we have $d_{0}d_{1}\dot{>}N$,
so we can apply Theorem \ref{thm: overfitting solution} and say that
the labels are generated using a $\left(\mathbf{X},\mathbf{y}\right)$
-dependent MNN: $\mathbf{y}=f\left(\mathbf{W}^{*}\mathbf{X}\right)^{\top}\mathbf{z}^{*}$
with target weights $\mathbf{W}^{*}=\left[\mathbf{w}_{1}^{*\top},\dots,\mathbf{w}_{d_{1}^{*}}^{*\top}\right]^{\top}\in\mathbb{R}^{d_{1}^{*}\times d_{0}}$
and $\mathbf{z}^{*}\in\mathbb{R}^{d_{1}}$. If, in addition, assumption
\ref{asmp: Teacher} holds then we can assume $\mathbf{W}^{*}$ and
$\mathbf{z}^{*}$ are independent from $\left(\mathbf{X},\mathbf{y}\right)$.
In both cases, the following differentiable region
\begin{equation}
\tilde{\mathcal{G}}\left(\mathbf{X},\mathbf{W}^{*}\right)\triangleq\left\{ \mathbf{W}\in\mathbb{R}^{d_{1}\times d_{0}}|\forall i\leq d_{1}^{*}:\,\,\mathrm{sign}\left(\mathbf{w}_{i}^{\top}\mathbf{X}\right)=\mathrm{sign}\left(\mathbf{w}_{i}^{*\top}\mathbf{X}\right)\right\} \,,\label{eq: global minima-1}
\end{equation}
also contains a differentiable global minimum (just set\textbf{ $\mathbf{w}_{i}=\mathbf{w}_{i}^{*}$},
$z_{i}=z_{i}^{*}$ $\forall i\leq d_{1}^{*}$, and $z_{i}=0$ $\forall i>d_{1}^{*}$),
and therefore $\forall\mathbf{X},\mathbf{y}$ and their corresponding
$\mathbf{W}^{*}$, we have 
\begin{equation}
\mathcal{G}\left(\mathbf{X},\mathbf{y}\right)\supset\tilde{\mathcal{G}}\left(\mathbf{X},\mathbf{W}^{*}\right)\label{eq: G and  G tilde}
\end{equation}
Also, we will make use of the following definition.
\begin{defn}
Let $\mathbf{X}$ have an \emph{angular margin} $\alpha$ from $\mathbf{W}^{*}$
if all datapoints (columns in $\mathbf{X}$) are at an angle of at
least $\alpha$ from all the weight hyperplanes (rows of $\mathbf{W}^{*}$)
, \emph{i.e.}, $\mathbf{X}$ is in the set
\begin{equation}
\mathcal{M}^{\alpha}\left(\mathbf{W}^{*}\right)\triangleq\left\{ \mathbf{X}\in\mathbb{R}^{d_{0}\times N}|\forall i,n:\left|\frac{\mathbf{x}^{\left(n\right)\top}\mathbf{w}_{i}^{*}}{\left\Vert \mathbf{x}^{\left(n\right)}\right\Vert \left\Vert \mathbf{w}_{i}^{*}\right\Vert }\right|>\mathrm{sin}\alpha\right\} \,.\label{eq: angular margin}
\end{equation}
\end{defn}
Using the definitions in eqs. (\ref{eq: angular margin}) and (\ref{eq: global minima-1}),
we prove the Theorem using the following three Lemmas.

First, In appendix section \ref{subsec:Good:-parsimonious-global}
we prove
\begin{lem}
\label{lem: Good solution}For any $\alpha,$ if $\mathbf{W}^{*}$
is independent from $\mathbf{W}$ then, in the limit $N\rightarrow\infty$,
$\forall\mathbf{X}\in\mathcal{M}^{\alpha}\left(\mathbf{W}^{*}\right)$
with $\log\sin\alpha\dot{>}d_{0}^{-1}\log d_{0}$ 
\begin{align*}
\mathcal{V}\left(\tilde{\mathcal{G}}\right)=\mathbb{P}_{\mathbf{W}\sim\mathcal{N}}\left(\mathbf{W}\in\tilde{\mathcal{G}}\left(\mathbf{X},\mathbf{W}^{*}\right)\right) & \dot{\geq}\exp\left(d_{0}d_{1}^{*}\log\sin\alpha\right).
\end{align*}
\end{lem}
Second, in appendix section \ref{sec:Angular margin for a fixed target weights}
we prove
\begin{lem}
\label{lem: Fixed Margin} Let $\mathbf{W}^{*}\in\mathbb{R}^{d_{1}^{*}\times d_{0}}$
a fixed matrix independent of $\mathbf{X}.$ Then, in the limit $N\rightarrow\infty$
with $d_{1}^{*}\dot{\leq}d_{0}\dot{\leq}N$, the probability of not
having an angular margin $\sin\alpha=1/\left(d_{1}^{*}d_{0}N\right)$
(eq. (\ref{eq: angular margin})) is upper bounded by
\[
\mathbb{P}\left(\mathbf{X}\notin\mathcal{M}^{\alpha}\left(\mathbf{W}^{*}\right)\right)\dot{\leq}\sqrt{\frac{2}{\pi}}d_{0}^{-1/2}
\]
\end{lem}
Lastly, in appendix section \ref{sec:Angular Margin of over-fitting Global minima}
we prove
\begin{lem}
\label{lem: Overfitting Margin} Let $\mathbf{X}\in\mathbb{R}^{d_{0}\times N}$
be a standard random Gaussian matrix of datapoints. Then we can find,
with probability 1, $\left(\mathbf{X},\mathbf{y}\right)$-dependent
matrices $\mathbf{W}^{*}$ and $\mathbf{z^{*}}$ as in Theorem \ref{thm: overfitting solution}
(where $d_{1}^{*}\triangleq4\left\lceil N/\left(2d_{0}-2\right)\right\rceil )$.
Moreover, in the limit $N\rightarrow\infty$, where $N/d_{0}\dot{\leq}d_{0}\dot{\leq}N$,
for any $\mathbf{y}$, we can bound the probability of not having
an angular margin (eq. (\ref{eq: angular margin})) with $\sin\alpha=1/\left(d_{1}^{*}d_{0}N\right)$
by
\begin{align*}
\mathbb{P}\left(\mathbf{X}\notin\mathcal{M}^{\alpha}\left(\mathbf{W}^{*}\right)\right) & \dot{\leq}\sqrt{\frac{8}{\pi}}d_{0}^{-1/2}+\frac{2d_{0}^{1/2}\sqrt{\log d_{0}}}{N}
\end{align*}
\end{lem}
Recall that $\forall\mathbf{X},\mathbf{y}$ and their corresponding
$\mathbf{W}^{*}$, we have $\mathcal{G}\left(\mathbf{X},\mathbf{y}\right)\subset\tilde{\mathcal{G}}\left(\mathbf{X},\mathbf{W}^{*}\right)$
(eq. (\ref{eq: G and  G tilde})). Thus, combining Lemmas \ref{lem: Good solution}
with $\sin\alpha=1/\left(d_{1}^{*}d_{0}N\right)$ together with either
Lemma \ref{lem: Fixed Margin} or \ref{lem: Overfitting Margin},
we prove the first (left) inequality of Theorem \ref{thm: Main theorem 2}:

\[
\mathcal{V}\left(\mathcal{G}\left(\mathbf{X},\mathbf{y}\right)\right)\dot{\geq}\exp\left(-d_{1}^{*}d_{0}\log N\right)
\]
Next, if $d_{1}^{*}=2N/d_{0}$ or $d_{1}^{*}\dot{<}N/d_{0}$ (is assumption
\ref{asmp: Teacher} holds), we obtain the second (right) inequality
\[
\exp\left(-d_{1}^{*}d_{0}\log N\right)\dot{\geq}\exp\left(-2N\log N\right).
\]
\QEDA

\section{\label{sec:proof of Volume-ratio}Volume ratio of global and local
minima: Proof of Theorem \ref{thm: volume ratio}}
\begin{thm}
\textbf{\emph{(Theorem }}\textbf{\emph{\small{}\ref{thm: volume ratio}}}\textbf{\emph{
restated)}} Given assumptions \ref{asmp:Gaussian Input}-\ref{asmp: input dimension},
we set $\delta\doteq\sqrt{\frac{8}{\pi}}d_{0}^{-1/2}+2d_{0}^{1/2}\sqrt{\log d_{0}}/N$.
Then, with probability $1-\delta$, the angular volume of sub-optimal
DLMs, with $\mathrm{MCE}>\epsilon>0$, is exponentially vanishing
in N, in comparison to the angular volume of global minima with $\mathrm{MCE=0}$
\[
\frac{\mathcal{V}\left(\mathcal{L}_{\epsilon}\left(\mathbf{X},\mathbf{y}\right)\right)}{\mathcal{V}\left(\mathcal{G}\left(\mathbf{X},\mathbf{y}\right)\right)}\dot{\leq}\exp\left(-\gamma_{\epsilon}N^{3/4}\left[d_{1}d_{0}\right]^{1/4}\right)\dot{\leq}\exp\left(-\gamma_{\epsilon}N\log N\right)\,.
\]
\end{thm}
To prove this theorem we first calculate the expectation of the angular
volume ratio given the $\mathbf{X}$-event that the bound in Theorem
\ref{thm: Main theorem 2} holds (given assumptions \ref{asmp:Gaussian Input}-\ref{asmp: input dimension}),
\emph{i.e.}, $\mathcal{V}\left(\mathcal{G}\left(\mathbf{X},\mathbf{y}\right)\right)\dot{\geq}\exp\left(-2N\log N\right)$.
Denoting this event\footnote{This event was previously denoted as $\mathbf{X}\in\mathcal{M}^{\alpha}\left(\mathbf{W}^{*}\right)$
in the proof of Theorem \ref{thm: Main theorem 2}, but this is not
important for this proof, so we simplified the notation.} as $\mathcal{M}$, we find: 
\begin{align}
\E_{\mathbf{X}\sim\mathcal{N}}\left[\frac{\mathcal{V}\left(\mathcal{L}_{\epsilon}\left(\mathbf{X},\mathbf{y}\right)\right)}{\mathcal{V}\left(\mathcal{G}\left(\mathbf{X},\mathbf{y}\right)\right)}|\mathcal{M}\right] & \overset{\left(1\right)}{\dot{\leq}}\frac{\E_{\mathbf{X}\sim\mathcal{N}}\left[\mathcal{V}\left(\mathcal{L}_{\epsilon}\left(\mathbf{X},\mathbf{y}\right)\right)|\mathcal{M}\right]}{\exp\left(-2N\log N\right)}\overset{\left(2\right)}{\leq}\nonumber \\
\frac{\E_{\mathbf{X}\sim\mathcal{N}}\left[\mathcal{V}\left(\mathcal{L}_{\epsilon}\left(\mathbf{X},\mathbf{y}\right)\right)\right]}{\mathbb{P}_{\mathbf{X}\sim\mathcal{N}}\left(\mathcal{M}\right)\exp\left(-2N\log N\right)} & \overset{\left(3\right)}{\dot{\leq}}\frac{\exp\left(-\gamma_{\epsilon}N^{3/4}\left[d_{1}d_{0}\right]^{1/4}\right)}{\mathbb{P}_{\mathbf{X}\sim\mathcal{N}}\left(\mathcal{M}\right)\exp\left(-2N\log N\right)}\overset{\left(4\right)}{\dot{\leq}}\nonumber \\
\frac{\exp\left(-\gamma_{\epsilon}N^{3/4}\left[d_{1}d_{0}\right]^{1/4}\right)}{\exp\left(-2N\log N\right)} & \overset{\left(5\right)}{\dot{\leq}}\exp\left(-\gamma_{\epsilon}N^{3/4}\left[d_{1}d_{0}\right]^{1/4}\right)\label{eq: conditional expected ratio bound}
\end{align}
where
\begin{enumerate}
\item We apply Theorem \ref{thm: Main theorem 2}.
\item We use the following fact 
\begin{fact}
\label{fact: Expectation condigional bound} For any variable $X\geq0$
and event $\mathcal{A}$ (where $\bar{\mathcal{A}}$ is its complement)
\begin{align*}
\mathbb{E}\left[X\right] & =\mathbb{E}\left[X|\mathcal{A}\right]\mathbb{P}\left(\mathcal{A}\right)+\mathbb{E}\left[X|\bar{\mathcal{A}}\right]\left(1-\mathbb{P}\left(\mathcal{A}\right)\right)\geq\mathbb{E}\left[X|\mathcal{A}\right]\mathbb{P}\left(\mathcal{A}\right)
\end{align*}
\end{fact}
\item We apply Theorem \ref{thm: Main theorem}.
\item We apply Theorem \ref{thm: Main theorem 2}.
\item We use assumption \ref{asmp: over-parameterization}, which implies
$\gamma_{\epsilon}N^{3/4}\left[d_{1}d_{0}\right]^{1/4}\dot{>}2N\log N$.
\end{enumerate}
For simplicity, in the reminder of the proof we denote 
\[
R\left(\mathbf{X}\right)\triangleq\frac{\mathcal{V}\left(\mathcal{L}_{\epsilon}\left(\mathbf{X},\mathbf{y}\right)\right)}{\mathcal{V}\left(\mathcal{G}\left(\mathbf{X},\mathbf{y}\right)\right)}.
\]
From Markov inequality (Fact \ref{fact: Markov-Inequality}), since
$R\left(\mathbf{X}\right)\geq0$, we have $\forall\eta\left(N\right)>0$:
\begin{equation}
\mathbb{P}_{\mathbf{X}\sim\mathcal{N}}\left[R\left(\mathbf{X}\right)\geq\eta\left(N\right)|\mathcal{M}\right]\leq\frac{\E_{\mathbf{X}\sim\mathcal{N}}\left[R\left(\mathbf{X}\right)|\mathcal{M}\right]}{\eta\left(N\right)}\label{eq: bound final 1}
\end{equation}
On the other hand, from fact \ref{fact: Expectation condigional bound},
we have
\begin{equation}
1-\mathbb{P}_{\mathbf{X}\sim\mathcal{N}}\left[R\left(\mathbf{X}\right)<\eta\left(N\right)|\mathcal{M}\right]\geq1-\frac{\mathbb{P}_{\mathbf{X}\sim\mathcal{N}}\left[R\left(\mathbf{X}\right)<\eta\left(N\right)\right]}{\mathbb{P}_{\mathbf{X}\sim\mathcal{N}}\left(\mathcal{M}\right)}\,.\label{eq: bound final 2}
\end{equation}
Combining Eqs. (\ref{eq: bound final 1})-(\ref{eq: bound final 2})
we obtain 
\[
\frac{\E_{\mathbf{X}\sim\mathcal{N}}\left[R\left(\mathbf{X}\right)|\mathcal{M}\right]}{\eta\left(N\right)}\geq1-\frac{\mathbb{P}_{\mathbf{X}\sim\mathcal{N}}\left[R\left(\mathbf{X}\right)<\eta\left(N\right)\right]}{\mathbb{P}_{\mathbf{X}\sim\mathcal{N}}\left(\mathcal{M}\right)}\,,
\]
and so
\[
\mathbb{P}_{\mathbf{X}\sim\mathcal{N}}\left(\mathcal{M}\right)-\mathbb{P}_{\mathbf{X}\sim\mathcal{N}}\left(\mathcal{M}\right)\frac{\E_{\mathbf{X}\sim\mathcal{N}}\left[R\left(\mathbf{X}\right)|\mathcal{M}\right]}{\eta\left(N\right)}\leq\mathbb{P}_{\mathbf{X}\sim\mathcal{N}}\left[R\left(\mathbf{X}\right)<\eta\left(N\right)\right]\,.
\]
We choose 
\[
\eta\left(N\right)=N\mathbb{P}_{\mathbf{X}\sim\mathcal{N}}\left(\mathcal{M}\right)\E_{\mathbf{X}\sim\mathcal{N}}\left[R\left(\mathbf{X}\right)|\mathcal{M}\right]\dot{=}\exp\left(-\gamma_{\epsilon}N^{3/4}\left[d_{1}d_{0}\right]^{1/4}\right)
\]
so that
\[
\mathbb{P}_{\mathbf{X}\sim\mathcal{N}}\left(\mathcal{M}\right)-\frac{1}{N}\leq\mathbb{P}_{\mathbf{X}\sim\mathcal{N}}\left[R\left(\mathbf{X}\right)\dot{\leq}\exp\left(-\gamma_{\epsilon}N^{3/4}\left[d_{1}d_{0}\right]^{1/4}\right)\right]\,.
\]
Then, from Theorem \ref{thm: Main theorem 2} we have 
\begin{align}
1-\mathbb{P}_{\mathbf{X}\sim\mathcal{N}}\left(\mathcal{M}\right) & \dot{\leq}\sqrt{\frac{8}{\pi}}d_{0}^{-1/2}+\frac{2d_{0}^{1/2}\sqrt{\log d_{0}}}{N}\,.\label{eq: bound final 3}
\end{align}
so we obtain the first (left) inequality in the Theorem (\ref{thm: volume ratio})

\[
\sqrt{\frac{8}{\pi}}d_{0}^{-1/2}+\frac{2d_{0}^{1/2}\sqrt{\log d_{0}}}{N}\dot{\geq}1-\mathbb{P}_{\mathbf{X}\sim\mathcal{N}}\left[\frac{\mathcal{V}\left(\mathcal{L}_{\epsilon}\left(\mathbf{X},\mathbf{y}\right)\right)}{\mathcal{V}\left(\mathcal{G}\left(\mathbf{X},\mathbf{y}\right)\right)}\dot{\leq}\exp\left(-\gamma_{\epsilon}N^{3/4}\left[d_{1}d_{0}\right]^{1/4}\right)\right]\,.
\]
Lastly, we note that assumption \ref{asmp: over-parameterization}
implies $\gamma_{\epsilon}N^{3/4}\left[d_{1}d_{0}\right]^{1/4}\dot{>}N\log N$,
which proves the second (right) inequality of the theorem.

\QEDA

\newpage{}

\part{Proofs of technical results\label{part:Technical-Results-and-Proofs}}

In this part we prove the technical results used in part \ref{part:Proofs-of-Main-Results}.

\section{Upper bounding the angular volume of sub-optimal differentiable local
minima: Proofs of Lemmas used in Section \ref{sec: bad local minima proof}}

\subsection{Proof of Lemma \ref{lem: error lower bounds S-main paper} \label{sec:Proof-of-rank-amplification}}

In this section we will prove Lemma \ref{lem: error lower bounds S-main paper}
in subsection \ref{subsec:Final-proof:-Orthant-product-Gaussian-matrices}.
Recall the following definition
\begin{defn}
Let 
\begin{align*}
\mathbf{A} & =\left[\boldsymbol{a}_{1},\dots,\boldsymbol{a}_{N}\right]\,;\,\mathbf{X}=\left[\mathbf{x}_{1},\dots,\mathbf{x}_{N}\right],
\end{align*}
where $\mathbf{X}\in\mathbb{R}^{d_{0}\times N}$ and $\mathbf{A}\in\mathbb{R}^{d_{1}\times N}$.
The Khatari-Rao product between the two matrices is defined as
\begin{eqnarray}
\mathbf{A}\circ\mathbf{X} & \triangleq & [\boldsymbol{a}_{1}\otimes\mathbf{x}_{1},\boldsymbol{a}_{2}\otimes\mathbf{x}_{2},...\boldsymbol{a}_{N}\otimes\mathbf{x}_{N}]\label{eq: A}\\
 & = & \left(\begin{array}{ccc}
a_{11}\mathbf{x}_{1} & a_{12}\mathbf{x}_{2} & \dots\\
a_{21}\mathbf{x}_{1} & a_{22}\mathbf{x}_{2} & \ddots\\
\vdots & \ddots & \ddots
\end{array}\right)\,.\nonumber 
\end{eqnarray}
\end{defn}
\begin{lem}
\textbf{\emph{\small{}(Lemma \ref{lem: error lower bounds S-main paper}
restated) }}Let $\mathbf{X}\in\mathbb{R}^{d_{0}\times N}$, $\mathbf{\mathbf{A}}\in\left\{ \rho,1\right\} ^{d_{1}\times N}$,
$S\subset\left[N\right]$ and $d_{0}d_{1}\geq N$. Then, simultaneously
for every possible $\mathbf{A}$ and $S$ such that 
\[
\left|S\right|\leq\mathrm{rank}\left(\mathbf{A}_{S}\right)d_{0}\,,
\]
we have that, $\mathbf{X}$-a.e., $\nexists\mathbf{v}\in\mathbb{R}^{N}$
such that $v_{n}\neq0$ $\forall n\in S$ and $\left(\mathbf{A}\circ\mathbf{X}\right)\mathbf{v}=0$
.
\end{lem}
\begin{proof}
We examine specific $\mathbf{\mathbf{A}}\in\left\{ \rho,1\right\} ^{d_{1}\times N}$
and $S\subset\left[N\right]$, and such that $\left|S\right|\leq\dS d_{0}$,
where we defined $\dS\triangleq\mathrm{rank}\left(\mathbf{A}_{S}\right)$.
We assume that $d_{S}\geq1$, since otherwise the proof is trivial.
Also, we assume by contradiction that $\exists\mathbf{v}\in\mathbb{R}^{N}$
such that $v_{i}\neq0$ $\forall i\in S$ and $\left(\mathbf{A}\circ\mathbf{X}\right)\mathbf{v}=0$
. Without loss of generality, assume that $S=\left\{ 1,2,...,\left|S\right|\right\} $
and that $\boldsymbol{a}_{1},\boldsymbol{a}_{2},...,\boldsymbol{a}_{\dS}$
are linearly independent. Then
\begin{equation}
\left(\mathbf{A}\circ\mathbf{\mathbf{X}}\right)\mathbf{v}=\sum_{n=1}^{\left|S\right|}v_{n}a_{k,n}\mathbf{x}_{n}=0\label{eq:coldepeq}
\end{equation}
for every $1\leq k\leq d_{1}$. From the definition of $S$ we must
have $v_{n}\neq0$ for every $1\leq n\leq\left|S\right|$. Since $\boldsymbol{a}_{1},\boldsymbol{a}_{2},...,\boldsymbol{a}_{\dS}$
are linearly independent, the rows of $\mathbf{A}_{\dS}=\left[\boldsymbol{a}_{1},\boldsymbol{a}_{2},...,\boldsymbol{a}_{\dS}\right]$
span a $\dS$-dimensional space. Therefore, it is possible to find
a matrix $\mathbf{R}$ such that $\mathbf{R}\mathbf{A}_{\dS}=[\mathbf{I}_{\dS\times d_{S}},0_{\dS\times\left(d_{1}-d_{S}\right)}]^{\top}$,
where $0_{i\times j}$ is the all zeros matrix with $i$ columns and
$j$ rows. Consider now $\mathbf{A}_{S}\circ\mathbf{X}_{S}$, \emph{i.e.},\emph{
}the matrix composed of the columns of $\mathbf{A}\circ\mathbf{\mathbf{X}}$
in $S$. Applying $\mathbf{R}^{\prime}=\mathbf{R}\otimes\mathbf{I}_{d_{0}}$
to $\mathbf{A}_{S}\circ\mathbf{\mathbf{X}}_{S}$, turns (\ref{eq:coldepeq})
into $d_{0}\dS$ equations in the variables $v_{1},...,v_{\left|S\right|}$,
of the form
\begin{equation}
v_{k}\mathbf{x}_{k}+\sum_{n=\dS+1}^{\left|S\right|}v_{n}\tilde{a}_{k,n}\mathbf{x}_{n}=0\label{eq:coldepeq1}
\end{equation}
for every $1\leq k\leq\dS$. We prove by induction that for every
$1\leq d\leq\dS$, the first $d_{0}d$ equations are linearly independent,
except for a set of matrices $\mathbf{X}$ of measure 0. This will
immediately imply $\left|S\right|>\dS d_{0}$, or else eq. \ref{eq:coldepeq}
cannot be true for $\mathbf{v}\neq0$. which will contradict our assumption,
as required. The induction can be viewed as carrying out Gaussian
elimination of the system of equations described by (\ref{eq:coldepeq1}),
where in each elimination step we characterize the set of matrices
$\mathbf{X}$ that for which that step is impossible, and show it
has measure 0.

For $d=1$, the first $d_{0}$ equations read $v_{1}\mathbf{x}_{1}+\sum_{n=\dS+1}^{\left|S\right|}v_{n}\tilde{a}_{1,n}\mathbf{x}_{n}=0$,
and since $v_{1}\neq0$, we must have $\mathbf{x}_{1}\in\mathrm{Span}\left\{ \tilde{a}_{1,\dS+1}\mathbf{x}_{\dS+1},...,\tilde{a}_{1,\left|S\right|}\mathbf{x}_{\left|S\right|}\right\} $.
However, except for a set of measure 0 with respect to $\mathbf{x}_{1}$
(a linear subspace of $\mathbb{R}^{d_{0}}$ with dimension less than
$d_{0}$), this can only happen if $\dim\mathrm{Span}\left\{ \tilde{a}_{1,\dS+1}\mathbf{x}_{\dS+1},...,\tilde{a}_{1,\left|S\right|}\mathbf{x}_{\left|S\right|}\right\} =d_{0}$,
which implies $\left|S\right|\geq\dS-1+d_{0}>d_{0}$ and also that
the first $d_{0}$ rows are linearly independent (since there are
$d_{0}$ independent columns). 

For a general $d$, we begin by performing Gaussian elimination on
the first $\left(d-1\right)d_{0}$ equations, resulting in a new set
of $r_{d}$ equations, such that every new equation contains one variable
that appears in no other new equation. Let $C$ be the set of the
indices (equivalently, columns) of these variables $r_{d}$ variables.
From (\ref{eq:coldepeq1}) it is clear none of the variables $v_{d},v_{d+1},...,v_{\dS}$
appear in the first $\left(d-1\right)d_{0}$ equations, and therefore
$C\subseteq S'=S\setminus\left\{ d,d+1,...,\dS\right\} $. By our
induction assumptions, except for a set of measure 0, the first $\left(d-1\right)d_{0}$
are independent, which means that $\left|C\right|=r_{d}=\left(d-1\right)d_{0}$.
We now extend the Gaussian elimination to the next $d_{0}$ equations,
and eliminate all the variables in $C$ from them. The result of the
elimination can be written down as,
\begin{equation}
v_{d}\mathbf{x}_{d}+\sum_{n\in S'\setminus C}v_{n}\left(\tilde{a}_{d,n}\mathbf{I}_{d_{0}}-\mathbf{Y}\right)\mathbf{x}_{n}=0\,,\label{eq:roweleq}
\end{equation}
where $\mathbf{Y}$ is a square matrix of size $d_{0}$ whose coefficients
depend only on $\{\tilde{a}_{k,n}\}_{n\in C,d>k\geq1}$ and on $\left\{ \mathbf{x}_{n}\right\} _{n\in C}$
, and in particular do not depend on $\mathbf{x}_{d}$ and $\left\{ \mathbf{x}_{n}\right\} _{n\in S'\setminus C}$. 

Now set $\tilde{\mathbf{x}}_{n}=(\tilde{a}_{d,n}\mathbf{I}_{d_{0}}-\mathbf{Y})\mathbf{x}_{n}$
for $n\in S'\setminus C$. As in the case of $d=1$, since $v_{d}\neq0$,
$\mathbf{x}_{d}\in\mathrm{Span}\{\tilde{\mathbf{x}}_{n}\}_{n\in S'\setminus C}$.
Therefore, for all values of $\mathbf{x}_{d}\in\mathbb{R}^{d_{0}}$
but a set of measure zero (linear subspace of with dimension less
than $d_{0}$), we must have $\dim\mathrm{Span}\{\tilde{\mathbf{x}}_{n}\}_{n\in S'\setminus C}=d_{0}$.
From the independence of $\{\tilde{\mathbf{x}}_{n}\}_{n\in S'\setminus C}$
on $\mathbf{x}_{d}$ it follows that $\dim\mathrm{Span}\{\tilde{\mathbf{x}}_{n}\}_{n\in S'\setminus C}=d_{0}$
holds a.e. with respect to the Lebesgue measure over $\mathbf{x}$. 

Whenever $\dim\mathrm{Span}\{\tilde{\mathbf{x}}_{n}\}_{n\in S'\setminus C}=d_{0}$
we must have $\left|S'\setminus C\right|\geq d_{0}$ and therefore
\begin{equation}
\left|S\right|>\left|S'\right|=\left|C\right|+\left|S'\setminus C\right|\geq\left(d-1\right)d_{0}+d_{0}=d_{0}d\,.
\end{equation}
Moreover, $\dim\mathrm{Span}\{\tilde{\mathbf{x}}_{n}\}_{n\in S'\setminus C}=d_{0}$
implies that the $d_{0}$ equations $v_{d}\mathbf{x}_{d}+\sum_{n\in S'\setminus C}v_{n}\tilde{\mathbf{x}}_{n}=0$
are independent. Thus, we may perform another step of Gaussian elimination
on these $d_{0}$ equations, forming $d_{0}$ new equations each with
a variable unique to it. Denoting by $C'$ the set of these $d_{0}$
variables, it is seen from (\ref{eq:roweleq}) that $C'\subseteq\left(S'\cup\left\{ d\right\} \right)\setminus C$
and in particular $C'$ is disjoint from $C$. Thus, considering the
first $\left(d-1\right)d_{0}$ equations together with the new $d_{0}$
equations, we see that there is a set $C\cup C'$ of $d_{0}d$ variables,
such that each variable in $C\cup C'$ appears only in one of the
$d_{0}d$ equations, and each of the $d_{0}d$ contains only a single
variable in $C\cup C'$. This means that the first $d_{0}d$ must
be linearly independent for all values of $\mathbf{X}$ except for
a set of Lebesgue measure zero, completing the induction.

Thus, we have proven, that for some $\mathbf{\mathbf{A}}\in\left\{ \rho,1\right\} ^{d_{1}\times N}$
and $S\subset\left[N\right]$ such that $\left|S\right|\leq\mathrm{rank}\left(\mathbf{A}_{S}\right)d_{0}$
the event 
\[
\mathcal{E}\left(\mathbf{A},S\right)=\left\{ \mathbf{X\in\mathbb{R}}^{d_{0}\times N}|\exists\mathbf{v}\in\mathbb{R}^{N}:\left(\mathbf{A}\circ\mathbf{X}\right)\mathbf{v}=0\,\mathrm{and}\,v_{n}\neq0,\,\forall n\in S\right\} 
\]
has zero measure. The event discussed in the theorem is a union of
these events:
\[
\mathcal{E}_{0}\triangleq\bigcup_{\mathbf{A}\in\left\{ \rho,1\right\} ^{d_{1}\times N}}\left[\bigcup_{S\subset\left[N\right]:\left|S\right|\leq\mathrm{rank}\left(\mathbf{A}_{S}\right)d_{0}}\mathcal{E}\left(\mathbf{A},S\right)\right]\,,
\]
and it also has zero measure, since it is a finite union of zero measure
events. 
\end{proof}
For completeness we note the following corollary, which is not necessary
for a our main results.
\begin{cor}
\label{lem: rank amplification} If $N\leq d_{1}d_{0}$, then $\mathrm{rank}\left(\mathbf{\mathbf{A}\circ\mathbf{X}}\right)=N$,
$\mathbf{X}$-a.e., if and only if, 
\[
\forall S\subseteq\left[N\right]:\left|S\right|\leq\mathrm{rank}\left(\mathbf{A}_{S}\right)d_{0}\,.
\]
\end{cor}
\begin{proof}
We define $\dS\triangleq\mathrm{rank}\left(\mathbf{A}_{S}\right)$
and $\mathbf{A}\circ\mathbf{X}$. The necessity of the condition $\left|S\right|\leq d_{0}\dS$
holds for every $\mathbf{X}$, as can be seen from the following counting
argument. Since the matrix $\mathbf{A}_{S}$ has rank $\dS$, there
exists an invertible row transformation matrix $\mathbf{R}$, such
that $\mathbf{R}\mathbf{A}_{S}$ has only $\dS$ non-zero rows. Consider
now $\mathbf{G}_{S}=\mathbf{A}_{S}\circ\mathbf{X}_{S}$, \emph{i.e.},\emph{
}the matrix composed of the columns of $\mathbf{G}$ in $S$. We have
\begin{equation}
\mathbf{G}_{S}^{\prime}=\left(\mathbf{R}\mathbf{A}_{S}\right)\circ\mathbf{X}_{S}=\mathbf{R}^{\prime}\left(\mathbf{A}_{S}\circ\mathbf{X}_{S}\right)=\mathbf{R}^{\prime}\mathbf{G}_{S}\,,
\end{equation}
where $\mathbf{R}^{\prime}=\mathbf{R}\otimes\mathbf{I}_{d_{0}}$ is
also an invertible row transformation matrix, which applies $\mathbf{R}$
separately on the $d_{0}$ sub-matrices of $\mathbf{G}_{S}$ that
are constructed by taking one every $d_{0}$ rows. Since $\mathbf{G}_{S}^{\prime}$
has at most $d_{0}\dS$ non-zero rows, the rank of $\mathbf{G}_{S}$
cannot exceed $d_{0}\dS$. Therefore, if $\left|S\right|>d_{0}\dS$,
$\mathbf{G}_{S}$ will not have full column rank, and hence neither
will $\mathbf{G}$. To demonstrate sufficiency a.e., suppose $\mathbf{G}$
does not have full column rank. Let $S$ be the minimum set of columns
of $\mathbf{G}$ which are linearly dependent. Since the columns of
$\mathbf{G}_{S}$ are assumed linearly dependent there exists $\mathbf{v}\in\mathbb{R}^{\left|S\right|}$
such $\left\Vert \mathbf{v}\right\Vert _{0}=\left|S\right|$ and $\mathbf{G}_{S}\mathbf{v}=0$.
Using Lemma \ref{lem: rank amplification} we complete the proof.
\end{proof}

\subsection{\label{sec:Proof-of-Lemma-rank-configurations}Proof of Lemma \ref{lem: Rank configurations bound-main paper}}

In this section we will prove Lemma \ref{lem: Rank configurations bound-main paper}
in subsection \ref{subsec:Final-proof:-Orthant-product-Gaussian-matrices}.
This proof relies on two rather basic results, which we first prove
in subsections \ref{subsec:Number-of-dichotomies} and \ref{subsec:A-basic-probabilistic}.

\subsubsection{Number of dichotomies induced by a hyperplane\label{subsec:Number-of-dichotomies}}
\begin{fact}
\label{fact: Schlafli}A hyperplane $\mathbf{w}\in d_{0}$ can separate
a given set of points $\mathbf{X}=\left[\mathbf{x}^{\left(1\right)},\dots,\mathbf{x}^{\left(N\right)}\right]\in\mathbb{R}^{d_{0}\times N}$
into several different dichotomies, i.e., different results for $\mathrm{sign}\left(\mathbf{w}^{\top}\mathbf{X}\right)$.
The number of dichotomies is upper bounded as follows:
\begin{align}
\sum_{\mathbf{h}\in\left\{ -1,1\right\} ^{N}}\mathbf{\mathcal{I}}\left(\exists\mathbf{w}:\mathrm{sign}\left(\mathbf{w}^{\top}\mathbf{X}\right)=\mathbf{h}^{\top}\right) & \leq2\sum_{k=0}^{d_{0}-1}\left(\begin{array}{c}
N-1\\
k
\end{array}\right)\leq2N^{d_{0}}\,.\label{eq: C count bound}
\end{align}
\end{fact}
\begin{proof}
See \citep[Theorem 1]{Cover1965} for a proof of the left inequality
as equality (the Schl{\"a}fli Theorem) in the case that the columns
of $\mathbf{X}$ are in ``general position'' (which holds $\mathbf{X}$-a.e,
see definition in \citep{Cover1965}) . If $\mathbf{X}$ is not in
general position then this result becomes an upper bound, since some
dichotomies might not be possible.

Next, we prove the right inequality. For $N=1$ and $N=2$ the inequality
trivially holds. For $N\geq3$, we have 
\[
2\sum_{k=0}^{d_{0}-1}\left(\begin{array}{c}
N-1\\
k
\end{array}\right)\overset{\left(1\right)}{\leq}2\sum_{k=0}^{d_{0}-1}\left(N-1\right)^{k}\overset{\left(2\right)}{\leq}2\frac{\left(N-1\right)^{d_{0}}-1}{N-2}\leq2N^{d_{0}}\,.
\]
where in $\left(1\right)$ we used the bound $\left(\begin{array}{c}
N\\
k
\end{array}\right)\leq N^{k}$ , in $\left(2\right)$ we used the sum of a geometric series.
\end{proof}

\subsubsection{A basic probabilistic bound\label{subsec:A-basic-probabilistic}}
\begin{lem}
\label{lem: Sum decomposition-1}Let $\mathbf{H}=\left[\mathbf{h}_{1}^{\top},\dots,\mathbf{h}_{d_{1}}^{\top}\right]^{\top}\in\left\{ -1,1\right\} ^{d_{1}\times k}$
be a deterministic binary matrix, $\mathbf{W}=\left[\mathbf{w}_{1}^{\top},\dots,\mathbf{w}_{d_{1}}^{\top}\right]^{\top}\in\mathbb{R}^{d_{1}\times d_{0}}$
be an independent standard random Gaussian matrix, and $\mathbf{X}\in\mathbb{R}^{d_{0}\times k}$
be a random matrix with independent and identically distributed columns.
\[
\mathbb{P}\left(\mathrm{sign}\left(\mathbf{W}\mathbf{X}\right)=\mathbf{H}\right)\leq\left(\begin{array}{c}
k\\
\left\lfloor k/2\right\rfloor 
\end{array}\right)\mathbb{P}\left(\mathbf{W}\mathbf{X}_{\left[\left\lfloor k/2\right\rfloor \right]}>0\right)\,.
\]
\end{lem}
\begin{proof}
By direct calculation
\begin{align*}
\mathbb{P}\left(\mathrm{sign}\left(\mathbf{W}\mathbf{X}\right)=\mathbf{H}\right) & =\mathbb{E}\left[\mathbb{P}\left(\mathrm{sign}\left(\mathbf{W}\mathbf{X}\right)=\mathbf{H}|\mathbf{X}\right)\right]\overset{\left(1\right)}{=}\mathbb{E}\left[\prod_{i=1}^{d_{1}}\mathbb{P}\left(\mathrm{sign}\left(\mathbf{w}_{i}^{\top}\mathbf{X}\right)=\mathbf{h}_{i}^{\top}|\mathbf{X}\right)\right]\\
 & \overset{\left(2\right)}{\leq}\mathbb{E}\left[\prod_{i=1}^{d_{1}}\mathbb{P}\left(\mathbf{w}_{i}^{\top}\mathbf{X}_{\hat{S}\left(\mathbf{h}_{i}\right)}>0|\mathbf{X}\right)\right]\overset{\left(3\right)}{\leq}\mathbb{E}\left[\prod_{i=1}^{d_{1}}\mathbb{P}\left(\mathbf{w}_{i}^{\top}\mathbf{X}_{S_{*}}>0|\mathbf{X}\right)\right]\\
 & \overset{\left(4\right)}{=}\mathbb{E}\left[\mathbb{P}\left(\mathbf{W}\mathbf{X}_{S_{*}}>0|\mathbf{X}\right)\right]\overset{\left(5\right)}{\leq}\mathbb{E}\left[\sum_{S\subset\left[k\right]:\left|S\right|=\left\lfloor k/2\right\rfloor }\mathbb{P}\left(\mathbf{W}\mathbf{X}_{S}>0|\mathbf{X}\right)\right]\\
 & =\sum_{S\subset\left[k\right]:\left|S\right|=\left\lfloor k/2\right\rfloor }\mathbb{E}\left[\mathbb{P}\left(\mathbf{W}\mathbf{X}_{S}>0|\mathbf{X}\right)\right]\overset{\left(6\right)}{=}\left(\begin{array}{c}
k\\
\left\lfloor k/2\right\rfloor 
\end{array}\right)\mathbb{P}\left(\mathbf{W}\mathbf{X}_{\left[\left\lfloor k/2\right\rfloor \right]}>0\right)\,.
\end{align*}
where 
\begin{enumerate}
\item We used the independence of the $\mathbf{w}_{i}$.
\item We define $\hat{S}_{\pm}\left(\mathbf{h}\right)\triangleq\left\{ S\subset\left[k\right]:\,\pm\mathbf{h}_{S}^{\top}>0\right\} $
as the sets in which $\mathbf{h}$ is always positive/negative, and
$\hat{S}\left(\mathbf{h}\right)$ as the maximal set between these
two. Note that $\mathbf{w}_{i}$ has a standard normal distribution
which is symmetric to sign flips, so $\forall S:$ $\mathbb{P}\left(\mathbf{w}_{i}^{\top}\mathbf{X}_{S}>0|\mathbf{X}\right)=\mathbb{P}\left(\mathbf{w}_{i}^{\top}\mathbf{X}_{S}<0|\mathbf{X}\right)$.
\item Note that $\left|\hat{S}\left(\mathbf{h}\right)\right|\geq\left\lfloor k/2\right\rfloor $.
Therefore, we define $S_{*}=\underset{S\subset\left[k\right]:\left|S\right|=\left\lfloor k/2\right\rfloor }{\mathrm{argmax}}\mathbb{P}\left(\mathbf{w}_{i}^{\top}\mathbf{X}_{S}>0|\mathbf{X}\right)$.
\item We used the independence of the $\mathbf{w}_{i}$.
\item The maximum is a single term in the following sum of non-negative
terms. 
\item Taking the expectation over $\mathbf{X}$, since the columns of $\mathbf{X}$
are independent and identically distributed, the location of $S$
does not affect the probability. Therefore, we can set without loss
of generality $S=\left[\left\lfloor k/2\right\rfloor \right]$.
\end{enumerate}
\end{proof}

\subsubsection{Main proof: Bound on the number of configurations for a binary matrix
with certain rank\label{subsec:Final-proof:rank-config-Bound}}

Recall the function $a\left(\cdot\right)$ from eq. (\ref{eq: LReLU}):
\[
a\left(u\right)\triangleq\begin{cases}
1 & ,\,\mathrm{if}\,,u>0\\
\rho & ,\,\mathrm{if}\,u<0
\end{cases}\,.
\]
where $\rho\neq1$.
\begin{lem}
\label{lem: Rank configurations bound} \textbf{\emph{(Lemma \ref{lem: Rank configurations bound-main paper}
restated).}} Let $\mathbf{X}\in\mathbb{R}^{d_{0}\times k}$ be a random
matrix with independent and identically distributed columns, and $\mathbf{W}\in\mathbb{R}^{d_{1}\times d_{0}}$
an independent standard random Gaussian matrix. Then, in the limit
$\min\left[k,d_{0},d_{1}\right]\dot{>}r$,
\[
\mathbb{P}\left(\mathrm{rank}\left(a\left(\mathbf{W}\mathbf{X}\right)\right)=r\right)\dot{\leq}2^{k+rd_{0}\left(\log d_{1}+\log k\right)+r^{2}}\mathbb{P}\left(\mathbf{W}\mathbf{X}_{\left[\left\lfloor k/2\right\rfloor \right]}>0\right)\,.
\]
\end{lem}
\begin{proof}
We denote $\mathbf{A}=a\left(\mathbf{W}\mathbf{X}\right)\in\left\{ \rho,1\right\} ^{d_{1}\times k}$.
For any such $\mathbf{A}$ for which $\mathrm{rank}\left(\mathbf{A}\right)=r$,
we have a collection of $r$ rows that span the remaining rows. There
are $\left(\begin{array}{c}
d_{1}\\
r
\end{array}\right)$ possible locations for these $r$ spanning rows. In these rows there
exist a collection of $r$ columns that span the remaining columns.
There are $\left(\begin{array}{c}
k\\
r
\end{array}\right)$ possible locations for these $r$ spanning columns. At the intersection
of the spanning rows and columns, there exist a full rank sub-matrix
$\mathbf{D}$. We denote $\tilde{\mathbf{A}}$ as the matrix \textbf{$\mathbf{A}$
}which rows and columns are permuted so that $\mathbf{D}$ is the
lower right block
\begin{equation}
\tilde{\mathbf{A}}\triangleq\left(\begin{array}{cc}
\mathbf{Z} & \mathbf{B}\\
\mathbf{C} & \mathbf{D}
\end{array}\right)=a\left(\begin{array}{cc}
\mathbf{W}_{1}\mathbf{X}_{1} & \mathbf{W}_{1}\text{\textbf{X}}_{2}\\
\mathbf{W}_{2}\text{\textbf{X}}_{1} & \mathbf{W}_{2}\text{\textbf{X}}_{2}
\end{array}\right)\,,\label{eq: M tilde}
\end{equation}
where \textbf{$\mathbf{D}$} is an invertible $r\times r$ matrix,
and we divided $\mathbf{X}$ and $\mathbf{W}$ to the corresponding
block matrices
\[
\mathbf{W}\triangleq\left[\mathbf{W}_{1}^{\top},\mathbf{W}_{2}^{\top}\right]^{\top},\mathbf{X}\triangleq\left[\mathbf{X}_{1},\mathbf{X}_{2}\right]\,,
\]
with $\mathbf{W}_{2}\in\mathbb{R}^{r\times d_{0}}$ rows and $\mathbf{X}_{2}\in\mathbb{R}^{d_{0}\times r}$. 

Since $\mathrm{rank}\left(\tilde{\mathbf{A}}\right)=r$, the first
$d_{1}-r$ rows are contained in the span of the last $r$ rows. Therefore,
there exists a matrix $\mathbf{Q}$ such that $\mathbf{Q}\mathbf{C}=\mathbf{Z}$
and $\mathbf{Q}\mathbf{D}=\mathbf{B}$. Since $\mathbf{D}$ is invertible,
this implies that $\mathbf{Q}=\text{\textbf{B}}\mathbf{D}^{-1}$ and
therefore 
\begin{equation}
\mathbf{Z}=\mathbf{B}\mathbf{D}^{-1}\mathbf{C}\,,\label{eq: A span}
\end{equation}
\emph{i.e.}, $\mathbf{B},\mathbf{C}$ and $\mathbf{D}$ uniquely determine
$\mathbf{Z}$. 

Using the union bound over all possible permutations from $\mathbf{A}$
to $\tilde{\mathbf{A}}$, and eq. (\ref{eq: A span}), we have

\begin{align}
 & \mathbb{P}\left(\mathrm{rank}\left(\mathbf{A}\right)=r\right)\label{eq: rank(M)}\\
 & \leq\left(\begin{array}{c}
d_{1}\\
r
\end{array}\right)\!\!\left(\begin{array}{c}
k\\
r
\end{array}\right)\mathbb{P}\left(\mathrm{rank}\left(\tilde{\mathbf{A}}\right)=r\right)\nonumber \\
 & \leq\left(\begin{array}{c}
d_{1}\\
r
\end{array}\right)\!\!\left(\begin{array}{c}
k\\
r
\end{array}\right)\mathbb{P}\left(\mathbf{Z}=\mathbf{B}\mathbf{D}^{-1}\mathbf{C}\right)\nonumber \\
 & =\left(\begin{array}{c}
d_{1}\\
r
\end{array}\right)\!\!\left(\begin{array}{c}
k\\
r
\end{array}\right)\mathbf{\mathbb{P}}\left(a\left(\mathbf{W}_{1}\mathbf{X}_{2}\right)\left[a\left(\mathbf{W}_{2}\mathbf{X}_{2}\right)\right]^{-1}a\left(\mathbf{W}_{2}\mathbf{X}_{1}\right)=a\left(\mathbf{W}_{1}\mathbf{X}_{1}\right)\right)\nonumber \\
 & =\left(\begin{array}{c}
d_{1}\\
r
\end{array}\right)\!\!\left(\begin{array}{c}
k\\
r
\end{array}\right)\!\!\!\!\!\!\!\!\!\!\!\!\!\!\!\!\!\!\!\!\!\!\!\!\!\!\!\!\!\!\!\!\sum_{\quad\quad\quad\mathbf{\quad\quad H}\in\left\{ -1,1\right\} ^{\left(d_{1}-r\right)\times\left(k-r\right)}}\!\!\!\!\!\!\!\!\!\!\!\!\!\!\!\!\!\!\!\!\!\!\!\!\!\!\!\!\!\!\mathbf{\mathbb{P}}\!\left(a\left(\mathbf{W}_{1}\mathbf{X}_{2}\right)\!\left[a\left(\mathbf{W}_{2}\mathbf{X}_{2}\right)\right]^{-1}\!a\left(\mathbf{W}_{2}\mathbf{X}_{1}\right)\!=\!a\left(\mathbf{H}\right)|\mathrm{sign}\left(\mathbf{W}_{1}\mathbf{X}_{1}\right)\!=\!\mathbf{H}\right)\!\mathbf{\mathbb{P}}\!\left(\mathrm{sign}\left(\mathbf{W}_{1}\mathbf{X}_{1}\right)=\mathbf{H}\right)\nonumber 
\end{align}
Using Lemma \ref{lem: Sum decomposition-1}, we have 
\begin{align}
\mathbf{\mathbb{P}}\left(\mathrm{sign}\left(\mathbf{W}_{1}\mathbf{X}_{1}\right)=\mathbf{H}\right)\leq & \left(\begin{array}{c}
k-r\\
\left\lfloor \left(k-r\right)/2\right\rfloor 
\end{array}\right)\mathbb{P}\left(\mathbf{W}_{1}\mathbf{X}_{\left[\left\lfloor \left(k-r\right)/2\right\rfloor \right]}>0\right)\,,\label{eq: matirx product bound}
\end{align}
an upper bound which does not depend on $\mathbf{H}.$ So all that
remains is to compute the sum:
\begin{align}
 & \!\!\!\!\!\!\sum_{\mathbf{\quad\quad H}\in\left\{ -1,1\right\} ^{\left(d_{1}-r\right)\times\left(k-r\right)}}\!\!\!\!\!\!\!\!\!\!\!\!\!\!\!\!\!\!\!\!\!\!\!\!\mathbf{\mathbb{P}}\left(a\left(\mathbf{W}_{1}\mathbf{X}_{2}\right)\left[a\left(\mathbf{W}_{2}\mathbf{X}_{2}\right)\right]^{-1}a\left(\mathbf{W}_{2}\mathbf{X}_{1}\right)=a\left(\mathbf{H}\right)|\mathrm{sign}\left(\mathbf{W}_{1}\mathbf{X}_{1}\right)=\mathbf{H}\right)\nonumber \\
= & \!\!\!\!\!\!\sum_{\mathbf{\quad\quad H}\in\left\{ -1,1\right\} ^{\left(d_{1}-r\right)\times\left(k-r\right)}}\!\!\!\!\!\!\!\!\!\!\!\!\!\!\!\!\!\!\!\!\!\!\!\!\mathbb{E}\left[\mathbf{\mathbb{P}}\left(a\left(\mathbf{W}_{1}\mathbf{X}_{2}\right)\left[a\left(\mathbf{W}_{2}\mathbf{X}_{2}\right)\right]^{-1}a\left(\mathbf{W}_{2}\mathbf{X}_{1}\right)=a\left(\mathbf{H}\right)|\mathbf{W}_{1},\mathbf{X}_{1}\right)|\mathrm{sign}\left(\mathbf{W}_{1}\mathbf{X}_{1}\right)=\mathbf{H}\right]\nonumber \\
\overset{\left(1\right)}{\leq} & \mathbb{E}\left[\left.\!\!\!\!\!\!\!\!\!\!\!\!\!\!\!\!\!\!\!\!\!\!\!\!\sum_{\mathbf{\quad\quad\quad\quad H}\in\left\{ -1,1\right\} ^{\left(d_{1}-r\right)\times\left(k-r\right)}}\!\!\!\!\!\!\!\!\!\!\!\!\!\!\!\!\!\!\!\!\!\!\!\mathbf{\mathcal{I}}\left(\exists\left(\mathbf{W}_{2},\mathbf{X}_{2}\right):\,a\left(\mathbf{W}_{1}\mathbf{X}_{2}\right)\left[a\left(\mathbf{W}_{2}\mathbf{X}_{2}\right)\right]^{-1}a\left(\mathbf{W}_{2}\mathbf{X}_{1}\right)=a\left(\mathbf{H}\right)\right)\right|\mathrm{sign}\left(\mathbf{W}_{1}\mathbf{X}_{1}\right)=\mathbf{H}\right]\label{eq: sum indicator}\\
\overset{\left(2\right)}{\leq} & \mathbb{E}\left[\left.2^{r^{2}}\left[\!\!\!\!\!\!\!\!\!\!\!\!\!\!\!\!\!\!\!\!\!\!\!\!\sum_{\quad\quad\mathbf{\quad\quad H}\in\left\{ -1,1\right\} ^{\left(d_{1}-r\right)\times r}}\!\!\!\!\!\!\!\!\!\!\!\!\!\!\!\!\!\!\!\!\!\!\!\!\mathbf{\mathcal{I}}\left(\exists\mathbf{X}_{2}:\,\mathrm{sign}\left(\mathbf{W}_{1}\mathbf{X}_{2}\right)=\mathbf{H}\right)\right]\left[\!\!\!\!\!\!\!\!\!\!\!\!\!\!\!\!\!\!\!\!\!\!\!\!\sum_{\quad\quad\quad\quad\mathbf{H}\in\left\{ -1,1\right\} ^{r\times\left(k-r\right)}}\!\!\!\!\!\!\!\!\!\!\!\!\!\!\!\!\!\mathbf{\mathcal{I}}\left(\exists\mathbf{W}_{2}:\mathrm{sign}\left(\mathbf{W}_{2}\mathbf{X}_{1}\right)=\mathbf{H}\right)\right]\right|\mathrm{sign}\left(\mathbf{W}_{1}\mathbf{X}_{1}\right)=\mathbf{H}\right]\nonumber \\
\leq & \mathbb{E}\left[\left.2^{r^{2}}\left[\!\!\!\!\!\!\!\!\!\!\!\!\!\!\!\!\!\!\!\!\!\!\!\!\sum_{\mathbf{\quad\quad\quad\quad h}\in\left\{ -1,1\right\} ^{\left(d_{1}-r\right)}}\!\!\!\!\!\!\!\!\!\!\!\!\!\!\!\!\!\!\!\!\!\!\!\!\mathbf{\mathcal{I}}\left(\exists\mathbf{x}:\mathrm{sign}\left(\mathbf{W}_{1}\mathbf{x}\right)=\mathbf{h}\right)\right]^{r}\left[\!\!\!\!\!\!\!\!\!\!\!\!\!\!\!\!\!\!\!\!\!\!\!\!\sum_{\quad\quad\mathbf{\quad\quad h}\in\left\{ -1,1\right\} ^{\left(k-r\right)}}\!\!\!\!\!\!\!\!\!\!\!\!\!\!\!\!\!\mathbf{\mathcal{I}}\left(\exists\mathbf{w}:\mathrm{sign}\left(\mathbf{w}^{\top}\mathbf{X}_{1}\right)=\mathbf{h}^{\top}\right)\right]^{r}\right|\mathrm{sign}\left(\mathbf{W}_{1}\mathbf{X}_{1}\right)=\mathbf{H}\right]\nonumber \\
\overset{\left(3\right)}{\leq} & \mathbb{E}\left[\left.2^{r^{2}}2^{rd_{0}\log\left(d_{1}-r\right)+r}2^{rd_{0}\log\left(k-r\right)+r}\right|\mathrm{sign}\left(\mathbf{W}_{1}\mathbf{X}_{1}\right)=\mathbf{H}\right]\nonumber \\
= & 2^{rd_{0}\left[\log\left(d_{1}-r\right)+\log\left(k-r\right)\right]+r^{2}+2r}\,,\label{eq: combi bound}
\end{align}
where
\begin{enumerate}
\item Given $\left(\mathbf{W}_{1},\mathbf{X}_{1}\right)$, and eq. (\ref{eq: M tilde}),
the indicator function in eq. (\ref{eq: sum indicator}) is equal
to zero only if $\mathbf{\mathbb{P}}\left(a\left(\mathbf{W}_{1}\mathbf{X}_{2}\right)\left[a\left(\mathbf{W}_{2}\mathbf{X}_{2}\right)\right]^{-1}a\left(\mathbf{W}_{2}\mathbf{X}_{1}\right)=\mathbf{A}|\mathbf{W}_{1},\mathbf{X}_{1}\right)=0$,
and one otherwise. 
\item This sum counts the number of values of $\mathbf{H}$ consistent with
$\mathbf{W}_{1}$ and $\mathbf{X}_{1}$. Conditioned on $\left(\mathbf{W}_{1},\mathbf{X}_{1}\right)$,
$\mathbf{D}=\left[a\left(\mathbf{W}_{2}\mathbf{X}_{2}\right)\right]^{-1}$,$\mathbf{B}=a\left(\mathbf{W}_{1}\mathbf{X}_{2}\right)$
and $\mathbf{C}=a\left(\mathbf{W}_{2}\mathbf{X}_{1}\right)$ can have
multiple values, depending on $\mathbf{W}_{2}$ and $\mathbf{X}_{2}$.
Also, any single value for $\left(\mathbf{D},\mathbf{B},\mathbf{C}\right)$
results in a single value of $\mathbf{H}$. Therefore, the number
of possible values of $\mathbf{H}$ in eq. (\ref{eq: sum indicator})
is upper bounded by the product of the number of possible values of
$\mathbf{D}$, $\mathbf{B}$ and $\mathbf{C}$, which is product in
the following equation.
\item The function $\sum_{\mathbf{h}\in\left\{ -1,1\right\} ^{\left(k-r\right)}}\mathbf{\mathcal{I}}\left(\exists\mathbf{w}:\mathrm{sign}\left(\mathbf{w}^{\top}\mathbf{X}_{1}\right)=\mathbf{h}^{\top}\right)$
counts the number of dichotomies that can be induced by the linear
classifier \textbf{w} on $\mathbf{X}_{1}$. Using eq. (\ref{eq: C count bound})
we can bound this number by $2\left(k-r\right)^{d_{0}}$. Similarly,
the other sum can be bounded by $2\left(d_{1}-r\right)^{r}$.
\end{enumerate}
Combining eqs. (\ref{eq: rank(M)}), (\ref{eq: matirx product bound})
and (\ref{eq: combi bound}) we obtain 
\begin{align*}
 & \mathbb{P}\left(\mathrm{rank}\left(\mathbf{A}\right)=r\right)\leq\\
 & \left(\begin{array}{c}
d_{1}\\
r
\end{array}\right)\left(\begin{array}{c}
k\\
r
\end{array}\right)\left(\begin{array}{c}
k-r\\
\left\lfloor \left(k-r\right)/2\right\rfloor 
\end{array}\right)2^{rd_{0}\left[\log\left(d_{1}-r\right)+\log\left(k-r\right)\right]+r^{2}+2r}\mathbb{P}\left(\mathbf{W}_{1}\mathbf{X}_{\left[\left\lfloor \left(k-r\right)/2\right\rfloor \right]}>0\right)\,.
\end{align*}
Next, we take the log. To upper bound $\left(\begin{array}{c}
N\\
k
\end{array}\right)$, for small $k$ we use $\left(\begin{array}{c}
N\\
k
\end{array}\right)\leq N^{k}$, while for $k=N/2$, we use $\left(\begin{array}{c}
N\\
N/2
\end{array}\right)\leq2^{N}$ . Thus, we obtain
\begin{align}
\log\mathbb{P}\left(\mathrm{rank}\left(\mathbf{A}\right)=r\right) & \leq\left(rd_{0}\left(\log\left(d_{1}-r\right)+\log\left(k-r\right)\right)+r^{2}+2r\right)\log2\label{eq: log(P(rank(A)<r)}\\
 & +r\log d_{1}+r\log k+\left(k-r\right)\log2+\log\mathbb{P}\left(\mathbf{W}_{1}\mathbf{X}_{\left[\left\lfloor \left(k-r\right)/2\right\rfloor \right]}>0\right)\,.\nonumber 
\end{align}
Recalling that $\mathbf{W}_{1}\in\mathbb{R}^{\left(d_{1}-r\right)\times d_{0}}$
while $\mathbf{W}\in\mathbb{R}^{d_{1}\times d_{0}}$, we obtain from
Jensen's inequality
\begin{equation}
\log\mathbb{P}\left(\mathbf{W}_{1}\mathbf{X}_{\left[\left\lfloor \left(k-r\right)/2\right\rfloor \right]}>0\right)\leq\frac{\left\lfloor \left(k-r\right)/2\right\rfloor \left\lfloor d_{1}-r\right\rfloor }{\left\lfloor k/2\right\rfloor \left\lfloor d_{1}\right\rfloor }\log\mathbb{P}\left(\mathbf{W}\mathbf{X}_{\left[\left\lfloor k/2\right\rfloor \right]}>0\right)\,.\label{eq: Jensen inequality}
\end{equation}
Taking the limit $\min\left[k,d_{0},d_{1}\right]\dot{>}r$ on eqs.
(\ref{eq: log(P(rank(A)<r)}) and (\ref{eq: Jensen inequality}) we
obtain
\begin{align*}
\mathbb{P}\left(\mathrm{rank}\left(\mathbf{A}\right)=r\right) & \dot{\leq}2^{k+rd_{0}\left(\log d_{1}+\log k\right)+t^{2}}\mathbb{P}\left(\mathbf{W}\mathbf{X}_{\left[\left\lfloor k/2\right\rfloor \right]}>0\right)\,.
\end{align*}
\end{proof}

\subsection{Proof of Lemma \ref{lem: Product of Gaussian Matrices-main paper}\label{sec:Proof-of-product of Gaussian}}

In this section we will prove Lemma \ref{lem: Product of Gaussian Matrices-main paper}
in subsection \ref{subsec:Final-proof:-Orthant-product-Gaussian-matrices}.
This proof relies on more elementary results, which we first prove
in subsections \ref{subsec:Orthant-probability-of-Gaussian-Vector}
and \ref{subsec:Mutual-Coherence}.

\subsubsection{Orthant probability of a random Gaussian vector\label{subsec:Orthant-probability-of-Gaussian-Vector}}

Recall that $\phi\left(x\right)$ and $\Phi\left(x\right)$ are, respectively,
the probability density function and cumulative distribution function
for a scalar standard normal random variable.
\begin{defn}
We define the following functions $\forall x\geq0$ 
\begin{align}
g\left(x\right)\triangleq & \frac{x\Phi\left(x\right)}{\phi\left(x\right)}\,,\label{eq: g function}\\
\psi\left(x\right)\triangleq & \frac{\left(g^{-1}\left(x\right)\right)^{2}}{2x}-\log\left(\Phi\left(g^{-1}\left(x\right)\right)\right),\label{eq: Psi function}
\end{align}
where the inverse function $g^{-1}\left(x\right):\left[0,\infty\right)\rightarrow\left[0,\infty\right)$
is well defined since $g\left(x\right)$ monotonically increase from
$0$ to $\infty$, for $x\geq0$.
\end{defn}
\begin{lem}
\label{lem: Equivariant Gaussian Vector}Let $\mathbf{z}\sim\mathcal{N}\left(0,\boldsymbol{\Sigma}\right)$
be a random Gaussian vector in $\mathbb{R}^{K}$, with a covariance
matrix $\Sigma_{ij}=\left(1-\theta K^{-1}\right)\delta_{mn}+\theta K^{-1}$
where $K\gg\theta>0$. Then, recalling $\psi\left(\theta\right)$
in eq. (\ref{eq: Psi function}), we have
\begin{align*}
\,\log\mathbb{P}\left(\forall i:\,z_{i}>0\right) & \leq-K\psi\left(\theta\right)+O\left(\log K\right)\,.
\end{align*}
\end{lem}
\begin{proof}
Note that we can write $\mathbf{z}=\mathbf{u}+\eta$, where $\mathbf{u}\sim\mathcal{N}\left(0,\left(1-\theta K^{-1}\right)\mathbf{I}_{K}\right)$,
and $\eta\sim\mathcal{N}\left(0,\theta K^{-1}\right)$. Using this
notation, we have
\begin{align}
 & \mathbb{P}\left(\forall i:\,z_{i}>0\right)\nonumber \\
= & \int_{-\infty}^{\infty}d\eta\left[\prod_{i=1}^{K}\int_{-\infty}^{\infty}du_{i}\mathcal{I}\left(\sqrt{1-\theta K^{-1}}u_{i}+\sqrt{\theta K^{-1}}\eta>0\right)\phi\left(u_{i}\right)\right]\phi\left(\eta\right)\nonumber \\
= & \int_{-\infty}^{\infty}d\eta\left[\Phi\left(\sqrt{\frac{\theta K^{-1}}{1-\theta K^{-1}}}\eta\right)\right]^{K}\phi\left(\eta\right)\nonumber \\
\overset{\left(1\right)}{=} & \sqrt{\frac{\theta}{2\pi\left(K-\theta\right)}}\int_{-\infty}^{\infty}d\xi\left[\Phi\left(\xi\right)\right]^{K}\exp\left(-\frac{\left(K-\theta\right)\xi^{2}}{2\theta}\right)\nonumber \\
= & \sqrt{\frac{\theta}{2\pi\left(K-\theta\right)}}\int_{-\infty}^{\infty}d\xi\exp\left(\frac{\xi^{2}}{2}\right)\exp\left[K\left(\log\Phi\left(\xi\right)-\frac{\xi^{2}}{2\theta}\right)\right],\label{eq: integral for Laplace}
\end{align}
where in $\left(1\right)$ we changed the variable of integration
to $\xi=\sqrt{\theta/\left(K-\theta\right)}\eta$. We denote, for
a fixed $\theta$,
\begin{align}
q\left(\xi\right) & \triangleq\log\Phi\left(\xi\right)-\frac{\xi^{2}}{2\theta}\label{eq: f function}\\
h\left(\xi\right) & \triangleq\sqrt{\frac{\theta}{2\pi\left(K-\theta\right)}}\exp\left(\frac{\xi^{2}}{2}\right)
\end{align}
and $\xi_{0}$ as its global maximum. Since $q$ is twice differentiable,
we can use Laplace's method (\emph{e.g.}, \citep{Butler2007}) to
simplify eq. (\ref{eq: integral for Laplace})

\begin{equation}
\log\int_{-\infty}^{\infty}h\left(\xi\right)\exp\left(Kq\left(\xi\right)\right)d\xi=Kq\left(\xi_{0}\right)+O\left(\log K\right)\,.\label{eq: Laplace Method}
\end{equation}
To find $\xi_{0}$, we differentiate $q\left(\xi\right)$ and equate
to zero to obtain
\begin{equation}
q^{\prime}\left(\xi\right)=\frac{\phi\left(\xi\right)}{\Phi\left(\xi\right)}-\frac{1}{\theta}\xi=0.\label{eq: maximum equation}
\end{equation}
which implies (recall eq. (\ref{eq: g function}))
\begin{equation}
g\left(\xi\right)\triangleq\frac{\xi\Phi\left(\xi\right)}{\phi\left(\xi\right)}=\theta\,.\label{eq: v function}
\end{equation}
This is a monotonically increasing function from $0$ to $\infty$
in the range $\xi\geq0$. Its inverse function can also be defined
in that range $g^{-1}\left(\theta\right):\left[0,\infty\right]\rightarrow\left[0,\infty\right]$.
This implies that this equation has only one solution, $\xi_{0}=g^{-1}\left(\theta\right)$.
Since $\lim_{\xi\rightarrow\infty}q\left(\xi\right)=-\infty$, this
$\xi_{0}$ is indeed the global maximum of $q\left(\xi\right)$. Substituting
this solution into $q\left(\xi\right)$, we get (recall eq. (\ref{eq: Psi function}))
\begin{equation}
\forall\theta>0:\,q\left(\xi_{0}\right)=-\psi\left(\theta\right)=q\left(g^{-1}\left(\theta\right)\right)=\log\left(\Phi\left(g^{-1}\left(\theta\right)\right)\right)-\frac{\left(g^{-1}\left(\theta\right)\right)^{2}}{2\theta}.\label{eq: Psi function 2}
\end{equation}
Using eq. (\ref{eq: integral for Laplace}), (\ref{eq: Laplace Method})
and (\ref{eq: Psi function 2}) we obtain:
\begin{align*}
 & \log\mathbb{P}\left(\forall i:\,z_{i}>0\right)\\
 & =\log\left[\int_{-\infty}^{\infty}d\xi\exp\left(\frac{\xi^{2}}{2}\right)\exp\left[K\left(\log\Phi\left(\xi\right)-\frac{\xi^{2}}{2\theta}\right)\right]\right]+O\left(\log K\right)\\
 & =-K\psi\left(\theta\right)+O\left(\log K\right)\,.
\end{align*}
 
\end{proof}
Next, we generalize the previous Lemma to a general covariance matrix.
\begin{cor}
\label{cor: General Gaussian Vector}Let $\mathbf{u}\sim\mathcal{N}\left(0,\boldsymbol{\Sigma}\right)$
be a random Gaussian vector in $\mathbb{R}^{K}$ for which $\forall n:\,\Sigma_{nn}=1$,
and $\theta\geq K\max_{n,m:\,n\neq m}\Sigma_{nm}>0\,.$ Then, again,
for large $K$
\begin{align*}
\log\mathbb{P}\left(\forall i:\,u_{i}>0\right) & \leq-K\psi\left(\theta\right)+O\left(\log K\right)\,.
\end{align*}
\end{cor}
\begin{proof}
We define $\tilde{\mathbf{u}}\sim\mathcal{N}\left(0,\tilde{\boldsymbol{\Sigma}}\right)$,
with $\tilde{\Sigma}_{mn}=\left(1-\theta K^{-1}\right)\delta_{mn}+\theta K^{-1}$.
Note that $\forall n:\,\Sigma_{nn}=\tilde{\Sigma}_{nn}=1$ and $\forall m\neq n$:
$\Sigma_{mn}\leq\tilde{\Sigma}_{mn}$. Therefore, from Slepian's Lemma
\citep[Lemma 1]{Slepian1962}, 
\[
\mathbb{P}\left(\forall n:\,\tilde{u}_{n}>0\right)\geq\mathbb{P}\left(\forall n:\,u_{n}>0\right).
\]
Using Lemma \ref{lem: Equivariant Gaussian Vector} on $\tilde{\mathbf{u}}$
completes the proof.
\end{proof}

\subsubsection{Mutual coherence bounds\label{subsec:Mutual-Coherence}}
\begin{defn}
\label{def: mutual coherence}We define the mutual coherence of the
columns of a matrix $\mathbf{A}=\left[\boldsymbol{a}_{1},\cdots,\boldsymbol{a}_{N}\right]\in\mathbb{R}^{M\times N}$
as the maximal angle between different columns
\[
\gamma\left(\mathbf{A}\right)\triangleq\max_{i,j:i\neq j}\frac{\left|\boldsymbol{a}_{i}^{\top}\boldsymbol{a}_{j}\right|}{\left\Vert \boldsymbol{a}_{i}\right\Vert \left\Vert \boldsymbol{a}_{j}\right\Vert }\,.
\]
Note that $\gamma\left(\mathbf{A}\right)\leq1$ and from \citep{Welch1974},
for $N\geq M$, $\gamma\left(\mathbf{A}\right)\geq\sqrt{\frac{N-M}{M\left(N-1\right)}}\,.$ 
\end{defn}
\begin{lem}
\label{lem: mutual coherence} Let $\mathbf{A}=\left[\boldsymbol{a}_{1},\cdots,\boldsymbol{a}_{N}\right]\in\mathbb{R}^{M\times N}$
be a standard random Gaussian matrix, and $\gamma\left(\mathbf{A}\right)$
is the mutual coherence of it columns (see definition \ref{def: mutual coherence}).
Then
\[
\mathbb{P}\left(\gamma\left(\mathbf{A}\right)>\epsilon\right)\leq2N^{2}\exp\left(-\frac{M\epsilon^{2}}{24}\right)\,.
\]
\end{lem}
\begin{proof}
In this case, we have from \citep[Appendix 1]{Chen2016}:
\[
\mathbb{P}\left(\gamma\left(\mathbf{A}\right)>\epsilon\right)\leq N\left(N-1\right)\left[\exp\left(-\frac{Ma^{2}\epsilon^{2}}{4\left(1+\epsilon/2\right)}\right)+\exp\left(-\frac{M}{4}\left(1-a\right)^{2}\right)\right]\,,
\]
for any $a\in\left(0,1\right)$. Setting $a=1-\epsilon/2$ 
\begin{align*}
\mathbb{P}\left(\gamma\left(\mathbf{A}\right)>\epsilon\right) & \leq N\left(N-1\right)\left[\exp\left(-\frac{M\left(1-\epsilon/2\right)^{2}\epsilon^{2}}{4\left(1+\epsilon/2\right)}\right)+\exp\left(-\frac{M}{16}\epsilon^{2}\right)\right]\\
 & \overset{\left(1\right)}{\leq}N\left(N-1\right)\left[\exp\left(-\frac{M\epsilon^{2}}{24}\right)+\exp\left(-\frac{M}{16}\epsilon^{2}\right)\right]\\
 & \leq2N^{2}\exp\left(-\frac{M\epsilon^{2}}{24}\right)\,,
\end{align*}
where in $\left(1\right)$ we can assume that $\epsilon\le1$, since
for $\epsilon\geq1$, we have $\mathbb{P}\left(\gamma\left(\mathbf{A}\right)>\epsilon\right)=0$
(recall $\gamma\left(\mathbf{A}\right)\leq1$). 
\end{proof}
\begin{lem}
\label{lem: subset mutual coherence}Let $\mathbf{B}=\left[\mathbf{b}_{1},\cdots,\mathbf{b}_{L}\right]\in\mathbb{R}^{M\times L}$
be a standard random Gaussian matrix and mutual coherence $\gamma$
as in definition \ref{def: mutual coherence}. Then, $\forall\epsilon>0$
and $\forall K\in\left[L\right]$: 
\[
\mathbb{P}\left(\min_{S\subset\left[N\right]:\left|S\right|=K}\gamma\left(\mathbf{B}_{S}\right)>\epsilon\right)\leq\exp\left[\left(2\log\left(2K\right)-\frac{M\epsilon^{2}}{24}\right)\left(\frac{L}{K}-1\right)\right]\,.
\]
\end{lem}
\begin{proof}
We upper bound this probability by partitioning the set of column
vectors into $\left\lfloor L/K\right\rfloor $ subsets $S_{i}$ of
size $\left|S_{i}\right|=K$ and require that in each subset the mutual
coherence is lower bounded by $\epsilon$. Since the columns are independent,
we have
\begin{align*}
 & \mathbb{P}\left(\min_{S\subset\left[N\right]:\left|S\right|=K}\gamma\left(\mathbf{B}_{S}\right)>\epsilon\right)\\
\leq & \prod_{i=1}^{\left\lfloor L/K\right\rfloor }\mathbb{P}\left(\forall S=\left\{ 1+\left(i-1\right)K,2+\left(1-i\right)K,\dots,iK\right\} :\gamma\left(\mathbf{B}_{S}\right)>\epsilon\right)\\
\overset{\left(1\right)}{\leq} & \prod_{i=1}^{L/K-1}2K^{2}\exp\left(-\frac{M\epsilon^{2}}{24}\right)\\
\leq & \exp\left[\left(2\log\left(2K\right)-\frac{M\epsilon^{2}}{24}\right)\left(\frac{L}{K}-1\right)\right]\,,
\end{align*}
where in $\left(1\right)$ we used the bound from Lemma \ref{lem: mutual coherence}.
\end{proof}

\subsubsection{Main proof: Orthant probability of a product Gaussian matrices\label{subsec:Final-proof:-Orthant-product-Gaussian-matrices}}
\begin{lem}
\textbf{\emph{(Lemma \ref{lem: Product of Gaussian Matrices-main paper}
restated).}}\label{lem: Product of Gaussian Matrices} Let $\mathbf{C}=\left[\mathbf{c}_{1},\cdots,\mathbf{c}_{N}\right]^{\top}\in\mathbb{R}^{N\times M}$
and $\mathbf{B}\in\mathbb{R}^{M\times L}$ be two independent random
Gaussian matrices. Without loss of generality, assume $N\geq L$,
and denote $\alpha\triangleq ML/N$. Then, in the regime $M\leq N$
and in the limit $\mathrm{min}\left[N,M,L\right]\dot{>}\alpha\dot{>}1$,
we have 
\[
\mathbb{P}\left(\mathbf{C}\mathbf{B}>0\right)\dot{\leq}\exp\left(-0.4N\alpha^{1/4}\right)\,.
\]
\end{lem}
\begin{proof}
For some $\theta>0$, and subset $S$ such that $\left|S\right|=K<L$,
we have
\begin{align*}
 & \mathbb{P}\left(\mathbf{C}\mathbf{B}>0\right)\\
\leq & \mathbb{P}\left(\mathbf{C}\mathbf{B}_{S}>0|\gamma\left(\mathbf{B}_{S}\right)\leq\epsilon\right)\mathbb{P}\left(\gamma\left(\mathbf{B}_{S}\right)\leq\epsilon\right)+\mathbb{P}\left(\mathbf{C}\mathbf{B}_{S}>0|\gamma\left(\mathbf{B}_{S}\right)>\epsilon\right)\mathbb{P}\left(\gamma\left(\mathbf{B}_{S}\right)>\epsilon\right)\\
\leq & \mathbb{P}\left(\mathbf{C}\mathbf{B}_{S}>0|\gamma\left(\mathbf{B}_{S}\right)\leq\epsilon\right)+\mathbb{P}\left(\gamma\left(\mathbf{B}_{S}\right)>\epsilon\right)\\
= & \mathbb{E}\left[\left[\mathbb{P}\left(\mathbf{c}_{1}^{\top}\mathbf{B}_{S}>0|\mathbf{B}_{S},\,\gamma\left(\mathbf{B}_{S}\right)\leq\epsilon\right)\right]^{N}|\gamma\left(\mathbf{B}_{S}\right)\leq\epsilon\right]+\mathbb{P}\left(\gamma\left(\mathbf{B}_{S}\right)>\epsilon\right)\,,
\end{align*}
where in the last equality we used the fact that the rows of $\mathbf{C}$
are independent and identically distributed.

We choose a specific subset
\[
S^{*}=\mathrm{argmin}_{S\subset\left[L\right]:\left|S\right|=K}\gamma\left(\mathbf{B}_{S}\right)
\]
 to minimize the second term and then upper bound it using Lemma \ref{lem: subset mutual coherence}
with $\theta=K\epsilon$; additionally, we apply Corollary \ref{cor: General Gaussian Vector}
on the first term with the components of the vector $\mathbf{u}$
being 
\[
u_{i}=\left(\mathbf{B}_{S}^{\top}\mathbf{c}_{1}\right)_{i}/\sqrt{\left(\mathbf{B}_{S}^{\top}\mathbf{B}_{S}\right)_{ii}}\in\mathbb{R}^{K}\,,
\]
which is a Gaussian random vector with mean zero and covariance $\boldsymbol{\Sigma}$
for which $\forall i:\,\Sigma_{ii}=1$ and $\forall i\neq j:$ $\Sigma_{ij}\leq\epsilon=\theta K^{-1}$.
Thus, we obtain
\begin{align}
\mathbb{P}\left(\mathbf{C}\mathbf{B}>0\right) & \leq\exp\left(-NK\psi\left(\theta\right)+O\left(N\log K\right)\right)+\exp\left[\left(\log\left(2K\right)^{2}-\frac{M\theta^{2}}{24K^{2}}\right)\left(\frac{L}{K}-1\right)\right]\:,\label{eq: AB  bound}
\end{align}
where we recall $\psi\left(\theta\right)$ is defined in eq. (\ref{eq: Psi function}). 

Next, we wish to select good values for $\theta$ and $K$, which
minimize this bound for large $\left(M,N,L,K\right)$. Thus, keeping
only the first order terms in each exponent (assuming $L\gg K\gg1)$,
we aim to minimize the function as much as possible
\begin{equation}
f\left(K,\theta\right)\triangleq\exp\left(-NK\psi\left(\theta\right)\right)+\exp\left(-\frac{M\theta^{2}L}{24K^{3}}\right).\label{eq: f(theta, K)}
\end{equation}
Note that the first term is decreasing in $K$, while the second term
increases. Therefore, for any $\theta$ the minimum of this function
in $K$ would be approximately achieved when both terms are equal,
\emph{i.e.},
\[
NK\psi\left(\theta\right)=\frac{M\theta^{2}L}{24K^{3}}\,,
\]
so we choose 
\begin{equation}
K\left(\theta\right)=\left(\frac{\theta^{2}ML}{24\psi\left(\theta\right)N}\right)^{1/4}\,.\label{eq: K}
\end{equation}
Substituting $K\left(\theta\right)$ into $f\left(K,\theta\right)$
yields 
\[
f\left(K\left(\theta\right),\theta\right)=2\exp\left(-N\left[\frac{\psi^{3}\left(\theta\right)\theta^{2}ML}{24N}\right]^{1/4}\right).
\]
To minimize this function in $\theta$, we need to maximize the function
$\psi^{3}\left(\theta\right)\theta^{2}$ (which has a single maximum).
Doing this numerically  gives us 
\begin{equation}
\theta_{*}\approx23.25\,;\,\psi\left(\theta_{*}\right)\approx0.1062;\,\psi^{3}\left(\theta_{*}\right)\theta_{*}^{2}\approx0.6478\,.\label{eq: optimal theta}
\end{equation}
 Substituting eqs. (\ref{eq: K}) and (\ref{eq: optimal theta}) into
eq. (\ref{eq: AB  bound}), we obtain
\begin{align*}
 & \mathbb{P}\left(\mathbf{C}\mathbf{B}>0\right)\\
 & \leq\exp\left(-N\left[\frac{ML}{37.05N}\right]^{1/4}+O\left(N\log K\right)\right)\\
 & +\exp\left[-N\left[\frac{ML}{37.05N}\right]^{1/4}+2L\frac{\log K}{K}+\frac{M\theta^{2}}{24K^{2}}-\log\left(2K^{2}\right)\right]\\
 & \leq\exp\left(-N\left[\frac{ML}{37.05N}\right]^{1/4}+O\left(N\log\left(\frac{ML}{N}\right)\right)\right)\,,
\end{align*}
where in the last line we used $N\geq L$,$N\geq M$ and $\mathrm{min}\left[N,M,L\right]\dot{>}\alpha\dot{>}1$.
Taking the log, and denoting $\alpha\triangleq ML/N$, we thus obtain
\[
\log\mathbb{P}\left(\mathbf{C}\mathbf{B}>0\right)\leq-0.4N\alpha^{1/4}+O\left(N\log\alpha\right)\,,
\]
Therefore, in the limit that $N\rightarrow\infty$ and $\alpha\left(N\right)\rightarrow\infty$,
with $\alpha\left(N\right)\dot{<}N$, we have
\[
\mathbb{P}\left(\mathbf{C}\mathbf{B}>0\right)\dot{\leq}\exp\left(-0.4N\alpha^{1/4}\right)\,.
\]
\end{proof}

\section{Lower bounding the angular volume of global minima: Proof of Lemmas
used in section \ref{sec: global minima proof}}

\subsection{Angles between random Gaussian vectors}

To prove the results in the next appendix sections, we will rely on
the following basic Lemma.
\begin{lem}
\label{lem: Beta bound}For any vector $\mathbf{y}$ and $\mathbf{x}\sim\mathcal{N}\left(0,\mathbf{I}_{d_{0}}\right)$,
we have 
\begin{equation}
\mathbb{P}\left(\left|\frac{\mathbf{x}^{\top}\mathbf{y}}{\left\Vert \mathbf{x}\right\Vert \left\Vert \mathbf{y}\right\Vert }\right|>\cos\left(\epsilon\right)\right)\geq\frac{2\sin\left(\epsilon\right)^{d_{0}-1}}{\left(d_{0}-1\right)B\left(\frac{1}{2},\frac{d_{0}-1}{2}\right)}\label{eq: Beta Bound 1}
\end{equation}
\begin{align}
\mathbb{P}\left(\left|\frac{\mathbf{x}^{\top}\mathbf{y}}{\left\Vert \mathbf{x}\right\Vert \left\Vert \mathbf{y}\right\Vert }\right|<u\right) & \leq\frac{2u}{B\left(\frac{1}{2},\frac{d_{0}-1}{2}\right)},\label{eq: Beta Bound 2}
\end{align}
where we recall that $B\left(x,y\right)$ is the beta function.
\end{lem}
\begin{proof}
Since $\mathcal{N}\left(0,\mathbf{I}_{d_{0}}\right)$ is spherically
symmetric, we can set $\mathbf{y}=\left[1,0\dots,0\right]^{\top}$,
without loss of generality. Therefore, 
\[
\left|\frac{\mathbf{x}^{\top}\mathbf{y}}{\left\Vert \mathbf{x}\right\Vert \left\Vert \mathbf{y}\right\Vert }\right|^{2}=\frac{x_{1}^{2}}{x_{1}^{2}+\sum_{i=2}^{d_{0}}x_{i}^{2}}\sim\mathcal{B}\left(\frac{1}{2},\frac{d_{0}-1}{2}\right),
\]
the Beta distribution, since $x_{1}^{2}\sim\chi^{2}\left(1\right)$
and $\sum_{i=2}^{d_{0}}x_{i}^{2}\sim\chi^{2}\left(d_{0}-1\right)$
are independent chi-square random variables. 

Suppose $Z\sim\mathcal{B}\left(\alpha,\beta\right)$, $\alpha\in\left(0,1\right)$,
and $\beta>1$ .
\begin{align*}
\mathbb{P}\left(Z>u\right)= & \frac{\int_{u}^{1}x^{\alpha-1}\left(1-x\right)^{\beta-1}dx}{B\left(\alpha,\beta\right)}\geq\frac{\int_{u}^{1}1^{\alpha-1}\left(1-x\right)^{\beta-1}dx}{B\left(\alpha,\beta\right)}=\frac{\int_{0}^{1-u}x^{\beta-1}dx}{B\left(\alpha,\beta\right)}=\frac{\left(1-u\right)^{\beta}}{\beta B\left(\alpha,\beta\right)}\,.
\end{align*}
Therefore, for $\epsilon>0$,
\begin{align*}
\mathbb{P}\left(\left|\frac{\mathbf{x}^{\top}\mathbf{y}}{\left\Vert \mathbf{x}\right\Vert \left\Vert \mathbf{y}\right\Vert }\right|^{2}>\cos^{2}\left(\epsilon\right)\right) & \geq\frac{2\left(1-\cos^{2}\left(\epsilon\right)\right)^{\frac{d_{0}-1}{2}}}{\left(d_{0}-1\right)B\left(\frac{1}{2},\frac{d_{0}-1}{2}\right)}=\frac{2\sin\left(\epsilon\right)^{d_{0}-1}}{\left(d_{0}-1\right)B\left(\frac{1}{2},\frac{d_{0}-1}{2}\right)}\,,
\end{align*}
which proves eq. (\ref{eq: Beta Bound 1}). 

Similarly, for $\alpha\in\left(0,1\right)$ and $\beta>1$ 
\begin{align*}
\mathbb{P}\left(Z<u\right) & =\frac{\int_{0}^{u}x^{\alpha-1}\left(1-x\right)^{\beta-1}dx}{B\left(\alpha,\beta\right)}\leq\frac{\int_{0}^{u}x^{\alpha-1}1^{\beta-1}dx}{B\left(\alpha,\beta\right)}=\frac{u^{\alpha}}{\alpha B\left(\alpha,\beta\right)}\,.
\end{align*}
Therefore, for $\epsilon>0$,
\[
\mathbb{P}\left(\left|\frac{\mathbf{x}^{\top}\mathbf{y}}{\left\Vert \mathbf{x}\right\Vert \left\Vert \mathbf{y}\right\Vert }\right|^{2}<u^{2}\right)\leq\frac{2u}{B\left(\frac{1}{2},\frac{d_{0}-1}{2}\right)}\,,
\]
which proves eq. (\ref{eq: Beta Bound 2}). 
\end{proof}

\subsection{Proof of Lemma \ref{lem: Good solution}: \label{subsec:Good:-parsimonious-global}}

Given three matrices: datapoints, $\mathbf{X}=\left[\mathbf{x}^{\left(1\right)},\dots,\mathbf{x}^{\left(N\right)}\right]\in\mathbb{R}^{d_{0}\times N}$,
weights \textbf{$\mathbf{W}=\left[\mathbf{w}_{1}^{\top},\dots,\mathbf{w}_{d_{1}}^{\top}\right]^{\top}\in\mathbb{R}^{d_{1}\times d_{0}}$},
and target weights $\mathbf{W}^{*}=\left[\mathbf{w}_{1}^{*\top},\dots,\mathbf{w}_{d_{1}^{*}}^{*\top}\right]^{\top}\in\mathbb{R}^{d_{1}^{*}\times d_{0}}$,
with $d_{1}^{*}\leq d_{1}$,we recall the following definitions:

\begin{equation}
\mathcal{M}^{\alpha}\left(\mathbf{W}^{*}\right)\triangleq\left\{ \mathbf{X}\in\mathbb{R}^{d_{0}\times N}|\forall i,n:\left|\frac{\mathbf{x}^{\left(n\right)\top}\mathbf{w}_{i}^{*}}{\left\Vert \mathbf{x}^{\left(n\right)}\right\Vert \left\Vert \mathbf{w}_{i}^{*}\right\Vert }\right|>\mathrm{sin}\alpha\right\} \label{eq: angular margin-appendix}
\end{equation}
and

\begin{equation}
\tilde{\mathcal{G}}\left(\mathbf{X},\mathbf{W}^{*}\right)\triangleq\left\{ \mathbf{W}\in\mathbb{R}^{d_{1}\times d_{0}}|\forall i\leq d_{1}^{*}:\,\,\mathrm{sign}\left(\mathbf{w}_{i}^{\top}\mathbf{X}\right)=\mathrm{sign}\left(\mathbf{w}_{i}^{*\top}\mathbf{X}\right)\right\} \,.\label{eq: global minima}
\end{equation}
Using these definitions, in this section we prove the following Lemma. 
\begin{lem}
\textbf{\emph{(Lemma \ref{lem: Good solution} restated).}} For any
$\alpha,$ if $\mathbf{W}^{*}$ is independent from $\mathbf{W}$
then, in the limit $N\rightarrow\infty$, $\forall\mathbf{X}\in\mathcal{M}^{\alpha}\left(\mathbf{W}^{*}\right)$
with $\log\sin\alpha\dot{>}d_{0}^{-1}\log d_{0}$ 
\begin{align*}
\mathbb{P}_{\mathbf{W}\sim\mathcal{N}}\left(\mathbf{W}\in\tilde{\mathcal{G}}\left(\mathbf{X},\mathbf{W}^{*}\right)\right) & \dot{\geq}\exp\left(d_{0}d_{1}^{*}\log\sin\alpha\right).
\end{align*}
\end{lem}
\begin{proof}
To lower bound $\mathbb{P}_{\mathbf{W}\sim\mathcal{N}}\left(\mathbf{W}\in\tilde{\mathcal{G}}\left(\mathbf{X},\mathbf{W}^{*}\right)\right)$
$\forall\mathbf{X}\in\mathcal{M}^{\alpha}\left(\mathbf{W}^{*}\right)$,
we define the event that all weight hyperplanes (with normals $\mathbf{w}_{i}$)
have an angle of at least $\alpha$ from the corresponding target
hyperplanes (with normals $\mathbf{w}_{i}^{*}$).
\[
\tilde{\mathcal{G}}_{i}^{\alpha}\left(\mathbf{W}^{*}\right)=\left\{ \mathbf{W}\in\mathbb{R}^{d_{1}\times d_{0}}|\left|\frac{\mathbf{w}_{i}^{\top}\mathbf{w}_{i}^{*}}{\left\Vert \mathbf{w}_{i}\right\Vert \left\Vert \mathbf{w}_{i}^{*}\right\Vert }\right|<\cos\left(\alpha\right)\right\} \,.
\]
In order that $\mathrm{sign}\left(\mathbf{w}_{i}^{\top}\mathbf{x}^{\left(n\right)}\right)\neq\mathrm{sign}\left(\mathbf{w}_{1}^{*\top}\mathbf{x}^{\left(n\right)}\right)$,
$\mathbf{w}_{i}$ must be rotated in respect to $\mathbf{w}_{i}^{*}$
by an angle greater then the angular margin $\alpha$, which is the
minimal the angle between $\mathbf{x}^{\left(n\right)}$ and the solution
hyperplanes (with normals $\mathbf{w}_{i}^{*}$). Therefore, we have
that, given $\mathbf{X}\in\mathcal{M}^{\alpha}\left(\mathbf{W}^{*}\right)$,
\begin{equation}
\forall\alpha:\,\bigcap_{i=1}^{d_{1}^{*}}\tilde{\mathcal{G}}_{i}^{\alpha}\left(\mathbf{W}^{*}\right)\subset\tilde{\mathcal{G}}\left(\mathbf{X},\mathbf{W}^{*}\right)\,.\label{eq: G subsets}
\end{equation}
And so, $\forall\mathbf{X}\in\mathcal{M}^{\alpha}\left(\mathbf{W}^{*}\right):$
\begin{align}
\mathbb{P}_{\mathbf{W}\sim\mathcal{N}}\left(\mathbf{W}\in\tilde{\mathcal{G}}\left(\mathbf{X},\mathbf{W}^{*}\right)\right) & \overset{\left(1\right)}{\geq}\mathbb{P}_{\mathbf{W}\sim\mathcal{N}}\left(\mathbf{W}\in\bigcap_{i=1}^{d_{1}^{*}}\tilde{\mathcal{G}}_{i}^{\alpha}\left(\mathbf{W}^{*}\right)\right)\,\label{eq: PG_W}\\
\quad\quad\overset{\left(2\right)}{=}\prod_{i=1}^{d_{1}^{*}}\mathbb{P}_{\mathbf{W}\sim\mathcal{N}}\left(\mathbf{W}\in\tilde{\mathcal{G}}_{i}^{\alpha}\left(\mathbf{W}^{*}\right)\right) & \overset{\left(3\right)}{\geq}\left[\frac{2\sin\left(\alpha\right)^{d_{0}-1}}{\left(d_{0}-1\right)B\left(\frac{1}{2},\frac{d_{0}-1}{2}\right)}\right]^{d_{1}^{*}}\,,\nonumber 
\end{align}
where in (1) we used eq. (\ref{eq: G subsets}), in (2) we used the
independence of $\left\{ \mathbf{w}_{i}\right\} _{i=1}^{d_{1}^{*}}$
and in (3) we used eq. (\ref{eq: Beta Bound 1}) from Lemma \ref{lem: Beta bound}.
Lastly, to simplify this equation we use the asymptotic expansion
of the beta function $B\left(\frac{1}{2},x\right)=\sqrt{\pi/x}+O\left(x^{-3/2}\right)$
for large $x$:
\begin{align*}
\log\mathbb{P}_{\mathbf{W}\sim\mathcal{N}}\left(\mathbf{W}\in\tilde{\mathcal{G}}\left(\mathbf{X},\mathbf{W}^{*}\right)\right) & \geq d_{0}d_{1}^{*}\log\sin\alpha+O\left(d_{1}^{*}\log d_{0}\right)\,.
\end{align*}
We obtain the Lemma in the limit $N\rightarrow\infty$ when $\log\sin\alpha\dot{>}d_{0}^{-1}\log d_{0}.$
\end{proof}

\subsection{Proof of Lemma \ref{lem: Fixed Margin}: \label{sec:Angular margin for a fixed target weights}}
\begin{lem}
\textbf{\emph{(Lemma \ref{lem: Fixed Margin} restated).}} \label{lem:fixed solution Margin}
Let $\mathbf{W}^{*}=\left[\mathbf{w}_{1}^{\top},\dots,\mathbf{w}_{d_{1}^{*}}^{\top}\right]^{\top}\in\mathbb{R}^{d_{1}^{*}\times d_{0}}$
a fixed matrix independent of $\mathbf{X}.$ Then, in the limit $N\rightarrow\infty$
with $d_{1}^{*}\dot{\leq}d_{0}\dot{\leq}N$, the probability of not
having an angular margin $\sin\alpha=1/\left(d_{1}^{*}d_{0}N\right)$
(eq. (\ref{eq: angular margin-appendix})) is upper bounded by
\[
\mathbb{P}\left(\mathbf{X}\notin\mathcal{M}^{\alpha}\left(\mathbf{W}^{*}\right)\right)\dot{\leq}\sqrt{\frac{2}{\pi}}d_{0}^{-1/2}
\]
\end{lem}
\begin{proof}
We define 
\[
\mathcal{M}_{n,i}^{\alpha}\left(\mathbf{W}^{*}\right)\triangleq\left\{ \mathbf{X}\in\mathbb{R}^{d_{0}\times N}|\left|\frac{\mathbf{x}^{\left(n\right)\top}\mathbf{w}_{i}^{*}}{\left\Vert \mathbf{x}^{\left(n\right)}\right\Vert \left\Vert \mathbf{w}_{i}^{*}\right\Vert }\right|>\mathrm{sin}\left(\alpha\right)\right\} \,,
\]
and $\mathcal{M}_{n}^{\alpha}\left(\mathbf{W}^{*}\right)\triangleq\bigcap_{i=1}^{d_{1}^{*}}\mathcal{M}_{n,i}^{\alpha}\left(\mathbf{W}^{*}\right)$.
Since $\mathcal{M}\left(\mathbf{W}^{*}\right)=\bigcap_{n=1}^{N}\mathcal{M}_{n}^{\alpha}\left(\mathbf{W}^{*}\right)$,
we have
\begin{align*}
\mathbb{P}\left(\mathbf{X}\in\mathcal{M}^{\alpha}\left(\mathbf{W}^{*}\right)\right) & \overset{\left(1\right)}{=}\prod_{n=1}^{N}\mathbb{P}\left(\mathbf{X}\in\mathcal{M}_{n}^{\alpha}\left(\mathbf{W}^{*}\right)\right)=\prod_{n=1}^{N}\left[1-\mathbb{P}\left(\mathbf{X}\notin\mathcal{M}_{n}^{\alpha}\left(\mathbf{W}^{*}\right)\right)\right]\\
 & \overset{\left(2\right)}{\geq}\prod_{n=1}^{N}\left[1-\sum_{i=1}^{d_{1}^{*}}\mathbb{P}\left(\mathbf{X}\notin\mathcal{M}_{n,i}^{\alpha}\left(\mathbf{W}^{*}\right)\right)\right]\overset{\left(3\right)}{\geq}\left[1-d_{1}^{*}\frac{2\sin\left(\alpha\right)}{B\left(\frac{1}{2},\frac{d_{0}-1}{2}\right)}\right]^{N}\,,
\end{align*}
where in (1) we used the independence of $\left\{ \mathbf{x}^{\left(n\right)}\right\} _{n=1}^{N}$,
in (2) we use the union bound, and in (3) we use eq. (\ref{eq: Beta Bound 2})
from Lemma \ref{lem: Beta bound}. Taking the log and we using the
asymptotic expansion of the beta function $B\left(\frac{1}{2},x\right)=\sqrt{\pi/x}+O\left(x^{-3/2}\right)$
for large $x$, we get 
\begin{align*}
\log\mathbb{P}\left(\mathbf{X}\in\mathcal{M}^{\alpha}\left(\mathbf{W}^{*}\right)\right) & \geq N\log\left[1-\sqrt{\frac{2}{\pi}d_{0}}d_{1}^{*}\sin\alpha+O\left(d_{1}^{*}d_{0}^{-1/2}\sin\alpha\right)\right]\\
 & =-\sqrt{\frac{2}{\pi}}d_{0}^{-1/2}+O\left(d_{0}^{-3/2}/N+d_{0}^{-1}N^{-2}\right)\,,
\end{align*}
where in the last line we recalled $\sin\alpha=1/N$. Recalling that
$d_{1}^{*}\dot{\leq}d_{0}\dot{\leq}N$ , we find 
\[
\mathbb{P}\left(\mathbf{X}\notin\mathcal{M}^{\alpha}\left(\mathbf{W}^{*}\right)\right)\dot{\geq}1-\exp\left(-\sqrt{\frac{2}{\pi}}d_{0}^{-1/2}\right)\geq\sqrt{\frac{2}{\pi}}d_{0}^{-1/2}
\]
\end{proof}

\subsection{Proof of Lemma \ref{lem: Overfitting Margin}: \label{sec:Angular Margin of over-fitting Global minima}}
\begin{lem}
\textbf{\emph{(Lemma \ref{lem: Overfitting Margin} restated).}} Let
$\mathbf{X}\in\mathbb{R}^{d_{0}\times N}$ be a standard random Gaussian
matrix of datapoints. Then we can find, with probability 1, $\left(\mathbf{X},\mathbf{y}\right)$-dependent
matrices $\mathbf{W}^{*}$ and $\mathbf{z^{*}}$ as in Theorem \ref{thm: overfitting solution}
(where $d_{1}^{*}\triangleq4\left\lceil N/\left(2d_{0}-2\right)\right\rceil )$.
Moreover, in the limit $N\rightarrow\infty$, where $N/d_{0}\dot{\leq}d_{0}\dot{\leq}N$,
for any $\mathbf{y}$, we can bound the probability of not having
an angular margin (eq. (\ref{eq: angular margin-appendix})) with
$\sin\alpha=1/\left(d_{1}^{*}d_{0}N\right)$ by
\begin{align*}
\mathbb{P}\left(\mathbf{X}\notin\mathcal{M}^{\alpha}\left(\mathbf{W}^{*}\right)\right) & \dot{\leq}\sqrt{\frac{8}{\pi}}d_{0}^{-1/2}+\frac{2d_{0}^{1/2}\sqrt{\log d_{0}}}{N}
\end{align*}
\end{lem}
\begin{proof}
In this proof we heavily rely on the notation and results from the
proof of in appendix section \ref{subsec: Overfitting solution}.
Without loss of generality we assume $\mathcal{\mathcal{S}}_{1}^{+}=\left[d_{0}-1\right]$.
Unfortunately, we can't use Lemma \ref{lem:fixed solution Margin}
\textendash{} this proof is significantly more complicated since the
constructed solution $\mathbf{W}^{*}$ depends on $\mathbf{X}$ (we
keep this dependence implicit, for brevity). Similarly to the proof
of Lemma \ref{lem:fixed solution Margin}, we define, 
\[
\mathcal{M}_{i,n}^{\alpha}\left(\mathbf{W}^{*}\right)\triangleq\left\{ \mathbf{X\in\mathbb{R}}^{d_{0}\times N}|\left|\frac{\mathbf{x}^{\left(n\right)\top}\mathbf{w}_{i}^{*}}{\left\Vert \mathbf{x}^{\left(n\right)}\right\Vert \left\Vert \mathbf{w}_{i}^{*}\right\Vert }\right|>\mathrm{sin}\left(\alpha\right)\right\} 
\]
and $\mathcal{M}_{i}^{\alpha}\left(\mathbf{W}^{*}\right)\triangleq\bigcap_{n=1}^{N}\mathcal{M}_{i,n}^{\alpha}\left(\mathbf{W}^{*}\right)$,
so $\mathcal{M}\left(\mathbf{W}^{*}\right)=\bigcap_{i=1}^{d_{1}^{*}}\mathcal{M}_{i}^{\alpha}\left(\mathbf{W}^{*}\right)$.
We have
\begin{align}
\mathbb{P}\left(\mathbf{X}\in\mathcal{M}^{\alpha}\left(\mathbf{W}^{*}\right)\right) & =1-\mathbb{P}\left(\mathbf{X}\notin\mathcal{M}^{\alpha}\left(\mathbf{W}^{*}\right)\right)\overset{\left(1\right)}{\geq}1-\sum_{i=1}^{d_{1}}\mathbb{P}\left(\mathbf{X}\notin\mathcal{M}_{i}^{\alpha}\left(\mathbf{W}^{*}\right)\right)\nonumber \\
 & \overset{\left(2\right)}{=}1-d_{1}^{*}\mathbb{P}\left(\mathbf{X}\notin\mathcal{M}_{1}^{\alpha}\left(\mathbf{W}^{*}\right)\right)=1-d_{1}^{*}\left(1-\mathbb{P}\left(\mathbf{X}\in\mathcal{M}_{1}^{\alpha}\left(\mathbf{W}^{*}\right)\right)\right)\,,\label{eq: P(M_alpha)}
\end{align}
where in $\left(1\right)$ we used the union bound, and in $\left(2\right)$
we used the fact that, from symmetry, $\forall i:\,\mathbb{P}\left(\mathbf{X}\notin\mathcal{M}_{i}^{\alpha}\left(\mathbf{W}^{*}\right)\right)=\mathbb{P}\left(\mathbf{X}\notin\mathcal{M}_{1}^{\alpha}\left(\mathbf{W}^{*}\right)\right)$.
Next, we examine the minimal angular margin in $\mathcal{M}_{1,n}^{\alpha}$:
separately for $\forall n<d_{0}$ and $\forall n\geq d_{0}$. Recalling
the construction of $\mathbf{W}$ in appendix section \ref{subsec: Overfitting solution},
we have, for $\forall n<d_{0}$:
\begin{align}
 & \min_{i,n<d_{0}}\left|\frac{\mathbf{x}^{\left(n\right)\top}\mathbf{w}_{i}^{*}}{\left\Vert \mathbf{x}^{\left(n\right)}\right\Vert \left\Vert \mathbf{w}_{i}^{*}\right\Vert }\right|=\min_{n<d_{0},\pm}\frac{\left|\left(\tilde{\mathbf{w}}_{1}\pm\epsilon_{2}\hat{\mathbf{w}}_{1}\right)^{\top}\mathbf{x}^{\left(n\right)}\right|}{\left\Vert \mathbf{\tilde{w}}_{1}\pm\epsilon_{2}\hat{\mathbf{w}}_{1}\right\Vert \left\Vert \mathbf{x}^{\left(n\right)}\right\Vert }\nonumber \\
 & \overset{\left(1\right)}{=}\min_{n<d_{0},\pm}\frac{\epsilon_{2}}{\left\Vert \mathbf{\tilde{w}}_{1}\pm\epsilon_{2}\hat{\mathbf{w}}_{1}\right\Vert \left\Vert \mathbf{x}^{\left(n\right)}\right\Vert }\overset{\left(2\right)}{=}\frac{\gamma\epsilon_{1}/\sqrt{1+\gamma^{2}\epsilon_{1}^{2}}}{\left\Vert \hat{\mathbf{w}}_{1}\right\Vert \max_{n<d_{0}}\left\Vert \mathbf{x}^{\left(n\right)}\right\Vert }\,,\label{eq: n<d0}
\end{align}
where in $\left(1\right)$ we used $\forall n<d_{0}$: $\mathbf{x}^{\left(n\right)\top}\hat{\mathbf{w}}_{1}=1$
and $\mathbf{x}^{\left(n\right)\top}\mathbf{\tilde{w}}_{1}=0$ , from
the construction of $\mathbf{\tilde{w}}_{1}$ and $\hat{\mathbf{w}}_{1}$
(eqs. (\ref{eq: w zero}), (\ref{eq: norm equality}), and (\ref{eq: w_hat})),
and in $\left(2\right)$ we used the fact that $\hat{\mathbf{w}}_{1}^{\top}\mathbf{\tilde{w}}_{1}=0$
from eq. (\ref{eq: w_hat}) together with $\left\Vert \mathbf{\tilde{w}}_{1}\right\Vert =\left\Vert \hat{\mathbf{w}}_{1}\right\Vert $
from eq. (\ref{eq: norm equality}), and $\epsilon_{2}=\gamma\epsilon_{1}$
from eq. (\ref{eq: epsilon 1,2}).

For $\forall n\geq d_{0}:$
\begin{equation}
\min_{i,n\geq d_{0}}\left|\frac{\mathbf{x}^{\left(n\right)\top}\mathbf{w}_{i}^{*}}{\left\Vert \mathbf{x}^{\left(n\right)}\right\Vert \left\Vert \mathbf{w}_{i}^{*}\right\Vert }\right|=\min_{n\geq d_{0},\pm}\frac{\left|\left(\mathbf{\tilde{w}}_{1}\pm\epsilon_{1}\hat{\mathbf{w}}_{1}\right)^{\top}\mathbf{x}^{\left(n\right)}\right|}{\left\Vert \mathbf{\tilde{w}}_{1}\pm\epsilon_{1}\hat{\mathbf{w}}_{1}\right\Vert \left\Vert \mathbf{x}^{\left(n\right)}\right\Vert }\geq\frac{\left(1-\gamma\beta\right)\epsilon_{1}}{\gamma\beta\sqrt{1+\epsilon_{1}^{2}}}\min_{n\geq d_{0}}\frac{\left|\mathbf{\hat{w}}_{1}^{\top}\mathbf{x}^{\left(n\right)}\right|}{\left\Vert \mathbf{\hat{w}}_{1}\right\Vert \left\Vert \mathbf{x}^{\left(n\right)}\right\Vert }\,,\label{eq: n>=00003Dd0}
\end{equation}
where we used the fact that $\forall n\geq d_{0}:\,\epsilon_{2}\left|\hat{\mathbf{w}}_{1}^{\top}\mathbf{x}^{\left(n\right)}\right|\leq\gamma\beta\left|\mathbf{\tilde{w}}_{1}^{\top}\mathbf{x}^{\left(n\right)}\right|$,
from eq. (\ref{eq: epsilon 1,2}), and also that $\hat{\mathbf{w}}_{1}^{\top}\mathbf{\tilde{w}}_{1}=0$
from eq. (\ref{eq: w_hat}).

We substitute eqs. (\ref{eq: n<d0}) and (\ref{eq: n>=00003Dd0})
into $\mathbb{P}\left(\mathbf{X}\in\mathcal{M}_{1}^{\alpha}\left(\mathbf{W}^{*}\right)\right)$:
\begin{align}
 & \mathbb{P}\left(\mathbf{X}\in\mathcal{M}_{1}^{\alpha}\left(\mathbf{W}^{*}\right)\right)\nonumber \\
 & \geq\mathbb{P}\left(\frac{\gamma\epsilon_{1}/\sqrt{1+\gamma^{2}\epsilon_{1}^{2}}}{\left\Vert \mathbf{\hat{w}}_{1}\right\Vert \max_{n<d_{0}}\left\Vert \mathbf{x}^{\left(n\right)}\right\Vert }>\sin\alpha,\frac{\left(1-\gamma\beta\right)\epsilon_{1}}{\gamma\beta\sqrt{1+\epsilon_{1}^{2}}}\min_{n\geq d_{0}}\frac{\left|\hat{\mathbf{w}}_{1}^{\top}\mathbf{x}^{\left(n\right)}\right|}{\left\Vert \mathbf{\hat{w}}_{1}\right\Vert \left\Vert \mathbf{x}^{\left(n\right)}\right\Vert }>\sin\alpha\right)\nonumber \\
 & \overset{\left(1\right)}{\geq}\mathbb{P}\left(\frac{\gamma\kappa}{\left\Vert \mathbf{\hat{w}}_{1}\right\Vert \max_{n<d_{0}}\left\Vert \mathbf{x}^{\left(n\right)}\right\Vert }>\sin\alpha,\frac{\left(1-\gamma\beta\right)}{\gamma\beta}\kappa\min_{n\geq d_{0}}\frac{x_{1}^{\left(n\right)}}{\left\Vert \mathbf{x}^{\left(n\right)}\right\Vert }>\sin\alpha,\frac{\epsilon_{1}}{\sqrt{1+\epsilon_{1}^{2}}}>\kappa\right)\label{eq:rotate}\\
 & \overset{\left(2\right)}{\geq}\mathbb{P}\left(\frac{\gamma\kappa}{\eta\sin\alpha}>\left\Vert \mathbf{\hat{w}}_{1}\right\Vert ,\eta>\max_{n<d_{0}}\left\Vert \mathbf{x}^{\left(n\right)}\right\Vert \right)\mathbb{P}\left(\frac{\left(1-\gamma\beta\right)}{\gamma\beta}\kappa\min_{n\geq d_{0}}\frac{x_{1}^{\left(n\right)}}{\left\Vert \mathbf{x}^{\left(n\right)}\right\Vert }>\sin\alpha,\frac{\epsilon_{1}}{\sqrt{1+\epsilon_{1}^{2}}}>\kappa\right)\,,\nonumber 
\end{align}
where in $\left(1\right)$ we rotate the axes so that $\hat{\mathbf{w}}_{1}\propto\left[1,0,0\dots,0\right]$
axes $\tilde{\mathbf{w}}_{1}\propto\left[0,1,0,0\dots,0\right]$ \textendash{}
this is possible due to the spherical symmetry of $\mathbf{x}^{\left(n\right)}$,
and the fact that $\hat{\mathbf{w}}_{1}$ and $\tilde{\mathbf{w}}_{1}$
are functions of $\mathbf{x}^{\left(n\right)}$ for $n<d_{0}$ (from
eqs. (\ref{eq: w_hat}) and (\ref{eq: w zero})), and as such, they
are independent from $\mathbf{x}^{\left(n\right)}$ for $n\geq d_{0}$,
in $(2)$ we use that fact that $\left\Vert \mathbf{\hat{w}}_{1}\right\Vert $
and $\max_{n<d_{0}}\left\Vert \mathbf{x}^{\left(n\right)}\right\Vert $
are functions of $\mathbf{x}^{\left(n\right)}$ for $n<d_{0}$ , and
as such, they are independent from $\mathbf{x}^{\left(n\right)}$
for $n\geq d_{0}$. Thus, 
\begin{align}
 & \mathbb{P}\left(\mathbf{X}\in\mathcal{M}_{1}^{\alpha}\left(\mathbf{W}^{*}\right)\right)\nonumber \\
 & \geq\left(1-\mathbb{P}\left(\frac{\gamma\kappa}{\eta\sin\alpha}\leq\left\Vert \mathbf{\hat{w}}_{1}\right\Vert \,\,\mathrm{or}\,\,\eta\leq\max_{n<d_{0}}\left\Vert \mathbf{x}^{\left(n\right)}\right\Vert \right)\right)\nonumber \\
 & \cdot\left(1-\mathbb{P}\left(\frac{\left(1-\gamma\beta\right)}{\gamma\beta}\kappa\min_{n\geq d_{0}}\frac{x_{1}^{\left(n\right)}}{\left\Vert \mathbf{x}^{\left(n\right)}\right\Vert }\leq\sin\alpha\,\,\mathrm{or\,}\,\frac{\epsilon_{1}}{\sqrt{1+\epsilon_{1}^{2}}}\leq\kappa\right)\right)\nonumber \\
 & \overset{\left(1\right)}{\geq}\left(1-\mathbb{P}\left(\frac{\gamma\kappa}{\eta\sin\alpha}\leq\left\Vert \mathbf{\hat{w}}_{1}\right\Vert \right)-\mathbb{P}\left(\eta\leq\max_{n<d_{0}}\left\Vert \mathbf{x}^{\left(n\right)}\right\Vert \right)\right)\nonumber \\
 & \cdot\left(1-\mathbb{P}\left(\frac{\left(1-\gamma\beta\right)}{\gamma\beta}\kappa\min_{n\geq d_{0}}\frac{x_{1}^{\left(n\right)}}{\left\Vert \mathbf{x}^{\left(n\right)}\right\Vert }\leq\sin\alpha\right)-\mathbb{P}\left(\frac{\epsilon_{1}}{\sqrt{1+\epsilon_{1}^{2}}}\leq\kappa\right)\right)\nonumber \\
 & =\left(\mathbb{P}\left(\eta>\max_{n<d_{0}}\left\Vert \mathbf{x}^{\left(n\right)}\right\Vert \right)-\mathbb{P}\left(\frac{\gamma\kappa}{\eta\sin\alpha}\leq\left\Vert \mathbf{\hat{w}}_{1}\right\Vert \right)\right)\label{eq: P(margin)}\\
 & \cdot\left(\mathbb{P}\left(\frac{\left(1-\gamma\beta\right)}{\gamma\beta}\kappa\min_{n\geq d_{0}}\frac{x_{1}^{\left(n\right)}}{\left\Vert \mathbf{x}^{\left(n\right)}\right\Vert }>\sin\alpha\right)-\mathbb{P}\left(\frac{\epsilon_{1}}{\sqrt{1+\epsilon_{1}^{2}}}\leq\kappa\right)\right)\,,\nonumber 
\end{align}
where in $\left(1\right)$ we use the union bound on both probability
terms. 

All that remains is to calculate each remaining probability term in
eq. (\ref{eq: P(margin)}). First, we have
\begin{align}
 & \mathbb{P}\left(\frac{\epsilon_{1}}{\sqrt{1+\epsilon_{1}^{2}}}\leq\kappa\right)=1-\mathbb{P}\left(\frac{\kappa}{\sqrt{1-\kappa^{2}}}<\epsilon_{1}\right)\nonumber \\
 & \overset{\left(1\right)}{=}1-\mathbb{P}\left(\min_{n\geq d_{0}}\frac{\left|\tilde{\mathbf{w}}_{i}^{\top}\mathbf{x}^{\left(n\right)}\right|}{\left|\hat{\mathbf{w}}_{i}^{\top}\mathbf{x}^{\left(n\right)}\right|}>\frac{\kappa}{\sqrt{1-\kappa^{2}}}\frac{1}{\beta}\right)\overset{\left(2\right)}{=}1-\mathbb{P}\left(\min_{n\geq d_{0}}\left|\frac{x_{2}^{\left(n\right)}}{x_{1}^{\left(n\right)}}\right|>\frac{\kappa}{\sqrt{1-\kappa^{2}}}\frac{1}{\beta}\right)\nonumber \\
 & \overset{\left(3\right)}{=}1-\left[\mathbb{P}\left(\left|\frac{x_{2}^{\left(1\right)}}{x_{1}^{\left(1\right)}}\right|>\frac{\kappa}{\sqrt{1-\kappa^{2}}}\frac{1}{\beta}\right)\right]^{N-d_{0}-1}\overset{\left(4\right)}{\leq}1-\left[1-\frac{2}{\pi}\arctan\left(\frac{\kappa}{\sqrt{1-\kappa^{2}}}\frac{1}{\beta}\right)\right]^{N}\,,\label{eq: epsilon bound}
\end{align}
where in $\left(1\right)$ we used eq. (\ref{eq: epsilon 1,2}), in
$\left(2\right)$ we recall that in eq. (\ref{eq:rotate}) we rotated
the axes so that $\hat{\mathbf{w}}_{1}\propto\left[1,0,0\dots,0\right]$
axes $\tilde{\mathbf{w}}_{1}\propto\left[0,1,0,0\dots,0\right]$,
in $\left(3\right)$ we used the independence of different $\mathbf{x}^{\left(n\right)}$,
and in $\left(4\right)$ we used the fact that the ratio of two independent
Gaussian variables is distributed according to the symmetric Cauchy
distribution, which has the cumulative distribution function $\mathbb{P}\left(X>x\right)=\frac{1}{2}-\frac{1}{\pi}\arctan\left(x\right)$,
and therefore $\mathbb{P}\left(\left|X\right|>x\right)=1-\frac{2}{\pi}\arctan\left(x\right)$.

Second, we use eq. (\ref{eq: Beta Bound 2})

\begin{equation}
\mathbb{P}\left(\min_{n\geq d_{0}}\frac{x_{1}^{\left(n\right)}}{\left\Vert \mathbf{x}^{\left(n\right)}\right\Vert }>\frac{\gamma\beta\sin\alpha}{\left(1-\gamma\beta\right)\kappa}\right)>\left[1-\frac{2\gamma\beta\sin\alpha}{\left(1-\gamma\beta\right)\kappa B\left(\frac{1}{2},\frac{d_{0}-1}{2}\right)}\right]^{N}.\label{eq: using the beta bound}
\end{equation}
Third,$\left\Vert \mathbf{x}^{\left(n\right)}\right\Vert ^{2}$ is
distributed according to the chi-square distribution of order $d_{0}$,
so for $\eta^{2}>d_{0}$,
\begin{align*}
\mathbb{P}\left(\left\Vert \mathbf{x}^{\left(n\right)}\right\Vert ^{2}\geq\eta^{2}\right)\leq & \left(\eta^{2}\exp\left(1-\eta^{2}/d_{0}\right)/d_{0}\right)^{d_{0}/2}\,.
\end{align*}
Therefore, 
\begin{equation}
\mathbb{P}\left(\max_{n<d_{0}}\left\Vert \mathbf{x}^{\left(n\right)}\right\Vert ^{2}<\eta^{2}\right)>\left[1-\left(\eta^{2}\exp\left(1-\eta^{2}/d_{0}\right)/d_{0}\right)^{d_{0}/2}\right]^{d_{0}-1}\,.\label{eq: eta}
\end{equation}
Lastly, we bound $\left\Vert \tilde{\mathbf{w}}_{1}\right\Vert =\left\Vert \hat{\mathbf{w}}_{1}\right\Vert $
(from eq. (\ref{eq: norm equality})). From eq. (\ref{eq: w_hat}),
we have 
\begin{equation}
\hat{\mathbf{w}}_{1}^{\top}\mathbf{X}_{\left[d_{0}-1\right]}=\left[1,\dots,1,1\right]\,,\label{eq: w1 X}
\end{equation}
where $\mathbf{X}_{\left[d_{0}-1\right]}$ has a singular value decomposition
\[
\mathbf{X}_{\left[d_{0}-1\right]}=\sum_{i=1}^{d_{0}}\sigma_{i}\mathbf{u}_{i}\mathbf{v}_{i}^{\top}\,,
\]
with $\sigma_{i}$ being the singular values, and $\mathbf{u}_{i}$
and $\mathbf{v}_{i}$ being the singular vectors. The singular values
are ordered from smallest to largest, and $\sigma_{1}=0$ with $\mathbf{u}_{1}=\tilde{\mathbf{w}}_{1}$,
from eq. (\ref{eq: w zero}). With probability 1, the other $d_{0}-1$
singular value are non-zero: they are the square roots of the eigenvalues
of the random matrix $\mathbf{X}_{\left[d_{0}-1\right]}^{\top}\mathbf{X}_{\left[d_{0}-1\right]}\in\mathbb{R}^{d_{0}-1\times d_{0}-1}$.
Taking the squared norm of eq. (\ref{eq: w1 X}), we have
\begin{equation}
d_{0}-1=\hat{\mathbf{w}}_{1}^{\top}\mathbf{X}_{\left[d_{0}-1\right]}\mathbf{X}_{\left[d_{0}-1\right]}^{\top}\hat{\mathbf{w}}_{1}=\sum_{i=1}^{d_{0}}\sigma_{i}^{2}\left(\mathbf{u}_{i}^{\top}\hat{\mathbf{w}}_{1}\right)^{2}\geq\sigma_{2}^{2}\left\Vert \hat{\mathbf{w}}_{1}\right\Vert ^{2}\,,\label{eq: sigma w1}
\end{equation}
where the last inequality stems from the fact that $\mathbf{u}_{1}^{\top}\hat{\mathbf{w}}_{1}=\mathbf{\tilde{w}}_{1}^{\top}\hat{\mathbf{w}}_{1}=0$
(from eq. (\ref{eq: w_hat})), so the minimal possible value is attained
when $\mathbf{u}_{2}^{\top}\hat{\mathbf{w}}_{1}=\left\Vert \hat{\mathbf{w}}_{1}\right\Vert $.
The minimal nonzero singular value, $\sigma_{2}$, can be bounded
using the following result from \citep[eq. (3.2)]{Rudelson2010} 
\[
\mathbb{P}\left(\min_{\mathbf{r}\in\mathbb{R}^{d_{0}}}\left\Vert \mathbf{X}_{\left[d_{0}\right]}\mathbf{r}\right\Vert \leq\eta d_{0}^{-1/2}\right)\leq\eta.
\]
Since

\[
\sigma_{2}=\min_{\mathbf{r}\in\mathbb{R}^{d_{0}-1}}\left\Vert \mathbf{X}_{\left[d_{0}-1\right]}\mathbf{r}\right\Vert \geq\min_{\mathbf{r}\in\mathbb{R}^{d_{0}}}\left\Vert \mathbf{X}_{\left[d_{0}\right]}\mathbf{r}\right\Vert 
\]
we have, 
\[
\mathbb{P}\left(\sigma_{2}<\eta d_{0}^{-1/2}\right)\leq\eta.
\]
Combining this with eq. (\ref{eq: sigma w1}) we get
\begin{equation}
\mathbb{P}\left(\frac{\beta\kappa}{\eta\sin\alpha}<\left\Vert \mathbf{w}_{1}\right\Vert \right)\leq\frac{\eta d_{0}}{\beta\kappa}\sin\alpha.\label{eq: omega}
\end{equation}
Lastly, combining eqs. (\ref{eq: epsilon bound}), (\ref{eq: using the beta bound}),
(\ref{eq: eta}) and (\ref{eq: omega}) into eqs. (\ref{eq: P(M_alpha)})
and (\ref{eq: P(margin)}), we get, for $\eta^{2}>d_{0}$,
\begin{align*}
 & \mathbb{P}\left(\mathbf{X}\in\mathcal{M}^{\alpha}\left(\mathbf{W}^{*}\right)\right)\\
 & \geq1-d_{1}^{*}\left(1-\left(\left[1-\left(\eta^{2}\exp\left(1-\eta^{2}/d_{0}\right)/d_{0}\right)^{d_{0}/2}\right]^{d_{0}-1}-\frac{\eta d_{0}}{\gamma\kappa}\sin\alpha\right)\right.\\
 & \cdot\left.\left(\left[1-\frac{2\gamma\beta\sin\alpha}{\left(1-\gamma\beta\right)\kappa B\left(\frac{1}{2},\frac{d_{0}-1}{2}\right)}\right]^{N}-\left[1-\frac{2}{\pi}\arctan\left(\frac{\kappa}{\sqrt{1-\kappa^{2}}}\frac{1}{\beta}\right)\right]^{N}\right)\right)\\
 & \geq1-d_{1}^{*}\left(1-\left(\left[1-\left(\log d_{0}\exp\left(1-\log d_{0}\right)\right)^{d_{0}/2}\right]^{d_{0}-1}-\frac{2d_{0}^{3/2}\sqrt{\log d_{0}}}{d_{1}^{*}N}\right)\right.\\
 & \left.\left(\left[1-\sqrt{\frac{8}{\pi}}\frac{1}{d_{1}^{*}d_{0}^{1/2}N}+O\left(\frac{1}{Nd_{1}^{*}d_{0}^{3/2}}\right)\right]^{N}-0.45^{N}\right)\right)\,,
\end{align*}
where in the last line we take $\beta=\gamma=\kappa=1/\sqrt{2}$,
$\eta=d_{0}^{1/2}\sqrt{\log d_{0}}$, $\sin\alpha=1/\left(d_{1}^{*}d_{0}N\right)$.
Using the asymptotic expansion of the beta function $B\left(\frac{1}{2},x\right)=\sqrt{\pi/x}+O\left(x^{-3/2}\right)$
for large $x$, we obtain, for $\sin\alpha=1/\left(d_{1}^{*}d_{0}N\right)$
\begin{align*}
 & 1-\mathbb{P}\left(\mathbf{X}\in\mathcal{M}^{\alpha}\left(\mathbf{W}^{*}\right)\right)\\
 & \leq d_{1}^{*}\left(1-\left(\left[1-\exp\left(-\frac{d_{0}}{2}\log\left(\frac{d_{0}}{e\log d_{0}}\right)\right)\right]^{d_{0}-1}-\frac{2d_{0}^{1/2}\sqrt{\log d_{0}}}{d_{1}^{*}N}\right)\right.\\
 & \left.\cdot\left(\left[1-\sqrt{\frac{8}{\pi}}\frac{1}{Nd_{1}^{*}d_{0}^{1/2}}+O\left(\frac{1}{Nd_{1}^{*}d_{0}^{3/2}}\right)\right]^{N}-2^{-N}\right)\right)\\
 & =d_{1}^{*}\left(1-\left(1-\frac{2d_{0}^{1/2}\sqrt{\log d_{0}}}{d_{1}^{*}N}+O\left(d_{0}\exp\left(-\frac{d_{0}}{2}\log\left(\frac{d_{0}}{\log d_{0}}\right)\right)\right)\right)\right.\\
 & \cdot\left.\left(1-\sqrt{\frac{8}{\pi}}\frac{1}{d_{1}^{*}d_{0}^{1/2}}+O\left(\frac{1}{d_{1}^{*}d_{0}^{3/2}}+\frac{1}{d_{1}^{*2}d_{0}N}+d_{1}^{*}2^{-N}+d_{1}^{*}d_{0}\exp\left(-\frac{d_{0}}{2}\log\left(\frac{d_{0}}{\log d_{0}}\right)\right)\right)\right)\right)\\
 & =\sqrt{\frac{8}{\pi}}\frac{1}{d_{0}^{1/2}}+\frac{2d_{0}^{1/2}\sqrt{\log d_{0}}}{N}+O\left(\frac{1}{d_{0}^{3/2}}+\frac{d_{0}^{1/4}}{d_{1}^{*}N}+d_{1}^{*}2^{-N}+d_{1}^{*}d_{0}\exp\left(-\frac{d_{0}}{2}\log\left(\frac{d_{0}}{\log d_{0}}\right)\right)\right).
\end{align*}
Thus, taking the log, and using $\log\left(1-x\right)=-x+O\left(x^{2}\right)$,
we obtain, for$\sin\alpha=1/\left(d_{1}^{*}d_{0}N\right)$
\begin{align*}
 & \log\mathbb{P}\left(\mathbf{X}\in\mathcal{M}^{\alpha}\left(\mathbf{W}^{*}\right)\right)\\
 & \geq\log\left(1-\sqrt{\frac{8}{\pi}}\frac{1}{d_{0}^{1/2}}-\frac{2d_{0}^{1/2}\sqrt{\log d_{0}}}{N}+O\left(\frac{1}{d_{0}^{3/2}}+\frac{d_{0}^{1/4}}{d_{1}^{*}N}+d_{1}^{*}2^{-N}+d_{0}\exp\left(-\frac{d_{0}}{2}\log\left(\frac{d_{0}}{\log d_{0}}\right)\right)\right)\right)\\
 & =-\sqrt{\frac{8}{\pi}}\frac{1}{d_{0}^{1/2}}-\frac{2d_{0}^{1/2}\sqrt{\log d_{0}}}{N}+O\left(\frac{1}{d_{0}^{3/2}}+\frac{d_{0}^{1/4}}{d_{1}^{*}N}+d_{1}^{*}2^{-N}+d_{0}\exp\left(-\frac{d_{0}}{2}\log\left(\frac{d_{0}}{\log d_{0}}\right)\right)\right)\,.
\end{align*}
Recall that $d_{1}^{*}\triangleq4\left\lceil N/\left(2d_{0}-2\right)\right\rceil \dot{=}N/d_{0}$.
Taking the limit $N\rightarrow\infty$, $d_{0}\rightarrow\infty$
with $d_{1}^{*}\dot{\leq}d_{0}\dot{\leq}N$, we have
\[
\mathbb{P}\left(\mathbf{X}\notin\mathcal{M}^{\alpha}\left(\mathbf{W}^{*}\right)\right)\dot{\leq}1-\exp\left(-\sqrt{\frac{8}{\pi}}d_{0}^{-1/2}-\frac{2d_{0}^{1/2}\sqrt{\log d_{0}}}{N}\right)\leq\sqrt{\frac{8}{\pi}}d_{0}^{-1/2}+\frac{2d_{0}^{1/2}\sqrt{\log d_{0}}}{N}
\]
\end{proof}

\part{Numerical Experiments - implementation details\label{sec:Implementation-details}}

Code and trained models for CIFAR and ImageNet results is available
here \url{https://github.com/MNNsMinima/Paper}. In MNIST, CIFAR and
ImageNet we performed binary classification on between the original
odd and even class numbers. In we performed this binary classification
between digits $0-4$ and $5-9$. Weights were initialized to be uniform
with mean zero and variance $2/d$, where $d$ is fan-in (here the
width of the previous neuron layer), as suggested in \citep{He2015a}.
In each epoch we randomly permuted the dataset and used the Adam \citep{Kingma2015}
optimization method (a variant of SGD) with $\beta_{1}=0.9,\beta_{2}=0.99,\eps=10^{-8}$.
Different learning rates and mini-batch sizes were selected for each
dataset and architecture. In CIFAR10 and ImageNet we used a learning-rate
of $\alpha=10^{-3}$ and a mini-batch size of $1024$; also, ZCA whitening
of the training samples was done to remove correlations between the
input dimensions, allowing faster convergence. We define $L$ as the
number of weight layers. For the random dataset we use a mini-batch
size of $\left\lfloor \min\left(N/2,d/2\right)\right\rfloor $ with
learning rate $\alpha=0.1$ and $0.05$, for $L=2$ and $3$, respectively.
In the random data parameter scans the training was done for no more
than $4000$ epochs  \textendash{} we stopped if $\text{MCE}=0$
was reached.
\end{document}